%% file: Formatting-Instructions-LaTeX-2024.tex
\documentclass[letterpaper]{article} 
\usepackage{aaai24}  
\usepackage{times}  
\usepackage{helvet}  
\usepackage{courier}  
\usepackage[hyphens]{url}  
\usepackage{graphicx} 
\usepackage{natbib}  
\usepackage{caption} 
\usepackage{subfiles}
\usepackage{pifont}

\usepackage{amsmath}
\usepackage{amsthm}
\usepackage{amssymb}
\usepackage{multirow}
\usepackage{subfig}
\usepackage[ruled, vlined]{algorithm2e}
\newtheorem{theorem}{Theorem}

\title{Exploiting Label Skews in Federated Learning with Model Concatenation}
\author{
    Yiqun Diao\textsuperscript{\rm 1},
    Qinbin Li\textsuperscript{\rm 2},
    Bingsheng He\textsuperscript{\rm 1}
}
\affiliations{
    \textsuperscript{\rm 1}National University of Singapore\\
    \textsuperscript{\rm 2}University of California, Berkeley\\

    \{yiqun, hebs\}@comp.nus.edu.sg,    qinbin@berkeley.edu
}

\usepackage{bibentry}

\begin{document}

\maketitle

\begin{abstract}
Federated Learning (FL) has emerged as a promising solution to perform deep learning on different data owners without exchanging raw data. However, non-IID data has been a key challenge in FL, which could significantly degrade the accuracy of the final model. Among different non-IID types, label skews have been challenging and common in image classification and other tasks. Instead of averaging the local models in most previous studies, we propose FedConcat, a simple and effective approach that concatenates these local models as the base of the global model to effectively aggregate the local knowledge. To reduce the size of the global model, we adopt the clustering technique to group the clients by their label distributions and collaboratively train a model inside each cluster. We theoretically analyze the advantage of concatenation over averaging by analyzing the information bottleneck of deep neural networks. Experimental results demonstrate that FedConcat achieves significantly higher accuracy than previous state-of-the-art FL methods in various heterogeneous label skew distribution settings and meanwhile has lower communication costs. Our code is publicly available at https://github.com/sjtudyq/FedConcat.
\end{abstract}

\section{Introduction}
\label{sec:intro}
A good machine learning model usually needs a large high-quality dataset to train. However, due to privacy concern and regulations such as GDPR \cite{voigt2017eu}, sometimes it is not allowed to collect original data for centralized training. Federated learning (FL) \cite{kairouz2019advances,li2019flsurvey,litian2019survey,yang2019federated} is proposed to let data owners collaboratively train a better machine learning model without exposing raw data. It has become a hot research topic \cite{li2020practical,dai2020federated,he2020group,lifedprox,karimireddy2019scaffold,liu2020secure,wu2020privacy}. FL has many potential practical applications~\cite{bonawitz2019towards,hard2018federated,kaissis2020secure}. For example, different hospitals can collectively train a FL model for diagnosing diseases through medical imaging, while protecting the privacy of individual patients. 

A typical framework of FL is FedAvg \cite{mcmahan2016communication}, where the clients train and send their local models to the server, and the server averages the local models to update the global model in each round. It has been shown that data heterogeneity is a challenging problem in FL, since non-IID data distributions among FL clients can degrade the FL model performance and slow down model convergence \cite{karimireddy2019scaffold,Li2020On,hsu2019measuring, li2021federated}. According to \citet{li2021federated}, non-IID data includes label skews, feature skews and quantity skews. In this paper, we focus on label skews (i.e., the label distributions of different clients are different), which is popular in reality (e.g., disease distributions vary across different areas). 

Researchers have put some promising effort to address the above label skew challenge. For example, FedProx \cite{lifedprox} uses the $L_2$ distance between the local model and the global model to regularize the local training. MOON \cite{li2021model} regularizes the local training using the similarity between representations of the local model and the global model. FedRS \cite{li2021fedrs} restricts the updates of unseen classes during local training. FedLC \cite{zhang2022federated} further calibrates logits to reduce the updates of minority classes. The key idea of existing studies is usually to reduce the drift produced in local training \cite{lifedprox,li2021model,karimireddy2019scaffold, li2021fedrs, zhang2022federated} or design a better federated averaging scheme in the server \cite{wang2020tackling, Wang2020Federated}. Those algorithms are based on the averaging framework. They attempt to address the label skew problem by mitigating its side effect in federated averaging. However, existing methods cannot achieve satisfactory performance. In label skews, the averaging methods may not make much sense, as each party may have very different models to predict different classes. Especially under extreme label skews where each client has quite different classes (e.g., face recognition), since the local optima are far from each other, averaging these local models leads to significant accuracy degradation. Even worse, it is challenging to quantify how label skews influence the model due to the diversity of label skews in practice. 

In this paper, we think out of the model-averaging scheme, and propose to use model concatenation as the aggregation method. Since each local model is good at classifying samples of several classes due to label skews, we propose to concatenate the features learned by the local models to combine the knowledge from the local models. For example, in the label skew setting, one client has sufficient data on cats with little data on dogs, while another client has sufficient data on dogs with little data on cats. Then, each client can train a local model which is good at predicting one class. Intuitively, concatenating those models can gather all key information, which can help train a good classifier for all classes among clients. This seemingly simple idea fundamentally changes the way of existing methods regarding label skews as an issue to avoid or mitigate. 

With this idea, we propose a novel FL algorithm to address label skews named FedConcat. First, the server divides clients into a few different clusters according to their label distributions. To address the privacy concern of uploading label distribution information, we develop an effective method to infer label distribution directly from the model. Second, FedAvg is conducted among each cluster to learn a good model for each kind of label distribution. Third, the server concatenates encoders of models of all clusters (i.e. neural networks except the last layer). Finally, with the parameters of the concatenated encoders fixed, the server and the clients jointly train a classifier on top of it using FedAvg. We theoretically justify that concatenation keeps richer mutual information than averaging in the feature space by applying the information bottleneck theory. 

Among each cluster, clients have similar label distributions. The label skew problem is alleviated inside the cluster, so FedAvg is competent to train a good model for each cluster with slight label skews. Since the concatenated encoders have already extracted good features, the task of training a linear classifier in the final stage becomes simpler. Therefore, FedAvg can achieve good accuracy for the simplified task. Moreover, through clustering, we can control the size of global model by adjusting the number of clusters.

We conduct extensive experiments with various label skew settings. Our experimental results show that FedConcat can significantly improve the accuracy compared with the other state-of-the-art FL algorithms including FedAvg \cite{mcmahan2016communication}, FedProx \cite{lifedprox}, MOON \cite{li2021model}, FedRS \cite{li2021fedrs} and FedLC \cite{zhang2022federated}. The improvement is more significant under extreme label skews. Besides, FedConcat can achieve better accuracy with much smaller communication and computation costs compared with baselines.

{Our contributions can be summarized as follow:}
\begin{itemize}
    \item Instead of averaging, we propose a new aggregation method in FL by concatenating the local models. Moreover, we apply clustering technique to alleviate label skew and control the size of global model.
    \item We theoretically show that concatenation preserves more information than averaging from the information bottleneck perspective, which guarantees the effectiveness of our approach.
    \item We conduct extensive experiments to show the effectiveness and communication efficiency of FedConcat. Under various label skew settings of a popular FL benchmark \cite{li2021federated}, FedConcat can outperform baselines averagely by 4\% on CIFAR-10, by 8\% on CIFAR-100, by 2\% on Tiny-ImageNet, and by 1\% on FMNIST and SVHN datasets.
\end{itemize}

\section{Background and Related Work}

Denote $D^i=(X^i,Y^i)$ the local dataset of client $i$. Label skews mean that $P(Y^i)$ differs among clients. According to \citet{li2021federated}, label skews can lead to significant accuracy degradation of the global model. It is also prevalent in real-world scenarios. For example, the disease distributions differ in different regions, which leads to label skews when training a global automatic disease diagnosis system. 

Previous studies like FedAvg \cite{mcmahan2016communication} average all models submitted by clients. However, under the non-IID data distribution cases, each client trains a good local model towards its local optimum. While the local optima may be far from each other, simply averaging the local models may produce a global model that is also far from the global optimum. There are many existing studies aiming to solve the non-IID data distribution problem based on FedAvg \cite{mcmahan2016communication}. 

A popular way is to improve local training so that the local model is not too far from the global optimum. For example, FedProx \cite{lifedprox} adds a regularization term which measures the distance between the local model and the global model. MOON \cite{li2021model} shares a similar motivation, regularizing by a contrastive loss to measure the distance between representations of the local model and the global model. Both methods add one more term to the loss function and require extra computations than FedAvg. SCAFFOLD \cite{karimireddy2019scaffold} adjusts the local gradient by keeping a correction term for each client, therefore its communication cost doubles. \citet{wang2021addressing} propose to monitor the class imbalance of each client based on uploaded gradient together with a small public dataset. Then they mitigate the imbalance by their Ratio Loss. FedRS \cite{li2021fedrs} proposes to restrict the updates of missing classes by down-scaling their logits, however it only deals with missing classes. To further deal with minority classes, FedLC \cite{zhang2022federated} proposes to calibrate logits based on the label statistics of local training data. FedAlign \cite{mendieta2022local} proposes to add regularization terms during local training to learn well-generalized representations. FedOV \cite{diao2023towards} introduces the ``unknown'' class and trains open-set classifiers in local training for a better ensemble. More FL works and clustering techniques are discussed in Appendix A.1 and A.2.

\section{Our Method: FedConcat}

\subsection{Problem Statement}
Federated learning aims to train a global model on multiple clients without exposing their raw data. Denote $D^i$ the local dataset of client $i$. Suppose there are $K$ clients, and the local loss function for each client is $\mathcal{L}(\cdot,\cdot)$. Formally, our goal is to train a global model $f$ that minimizes the following objective.

\begin{equation}
    L = \frac{\sum_{i=1}^{K} |D^i|\cdot \mathbb{E}_{(X^i,Y^i)\sim D^i}[\mathcal{L}(f(X^i),Y^i)]}{\sum_{i=1}^{K} |D^i|}
    \label{eq:goal}
\end{equation}

{Like existing studies \citep{mcmahan2016communication, lifedprox, li2021model, karimireddy2019scaffold, wang2021addressing, li2021fedrs, zhang2022federated}, in this study, we assume that transmitting only models achieves a basic level of privacy protection in FL, compared with transferring the raw data. One can consider more advanced protection such as differential privacy \citep{dwork2011differential} to avoid inference attacks against the client models. Our method can be easily integrated with those privacy protection, which is not the focus of this paper and will be interesting future work. }

\subsection{Motivation}
\paragraph{Pitfalls of existing methods in label skews}
Under label skews, the local models can be much different as they are trained on different classes. Therefore it hardly makes sense to average each parameter of these models with quite different tasks. As an example, we train FedAvg on two clients of CIFAR-10 under label skews. The first client only has samples of class 0 and 2, while the second client only has samples of class 1 and 9. For both clients, we train 10 local epochs per round. We show the accuracy of local models and averaged global model of two rounds in Figure \ref{fig:track}. As we can see, the accuracy of local classes increases during local training, while the averaging operation leads to significant accuracy degradation. This example illustrates the problem of averaging local models under extreme label skews. 

\begin{figure}[ht]
    \centering
    \includegraphics[width=\columnwidth]{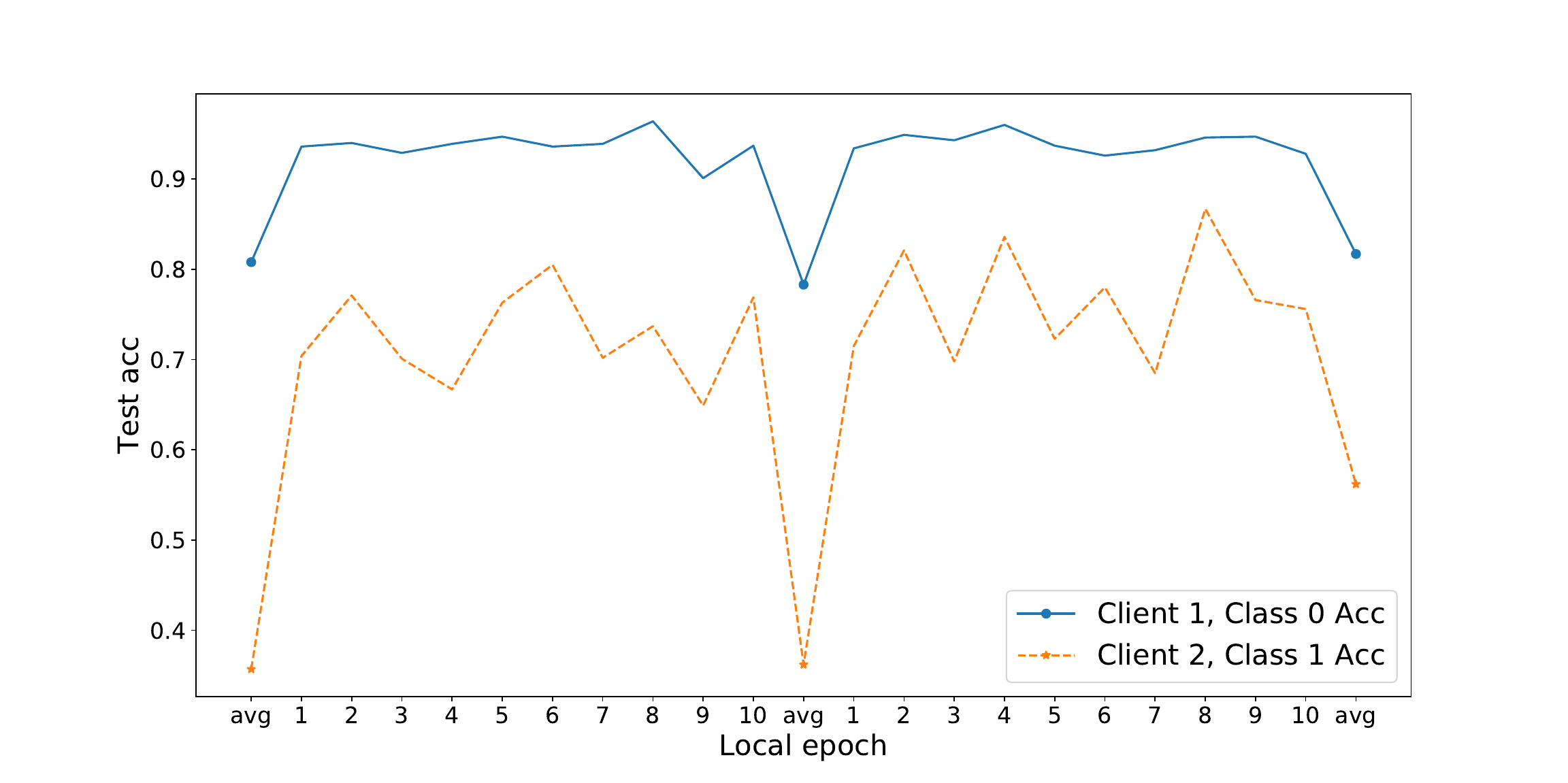}
    \caption{Accuracy of local models and averaged model on two clients under label skews. }
    \label{fig:track}
\end{figure}

\paragraph{An alternative view of label skews}
Let us view the neural network as a feature extractor (all the layers in the network except the last layer) and a classifier (the last layer). Since each client's model is well-fitted in its own dataset, we already have quite a few locally well-trained feature extractors. {Intuitively, concatenating the features from different local extractors can provide a better feature representation for label skews.} Thus we propose the idea of concatenating feature extractors and training a global classifier. 

If we concatenate the models of all clients, our final model size can grow much large if there are many clients, and the overhead of training the global classifier is much more expensive. {In practice, although label skews are prevalent, some parties may have similar label distributions. For example, hospitals in the same region may encounter similar types of diseases. } Therefore, we adopt the clustering method before training. By clustering all clients into a few groups via their label distributions, we can control the size of global model. Inside each group, since grouped clients have similar label distributions, the trained model can capture this kind of data well.

In brief, we tackle the label skew problem by generating solutions for each group individually. Next, we combine those solutions together to get a better global model with smaller communication cost. 

\begin{figure}[ht]
    \centering
    \includegraphics[width=1\columnwidth]{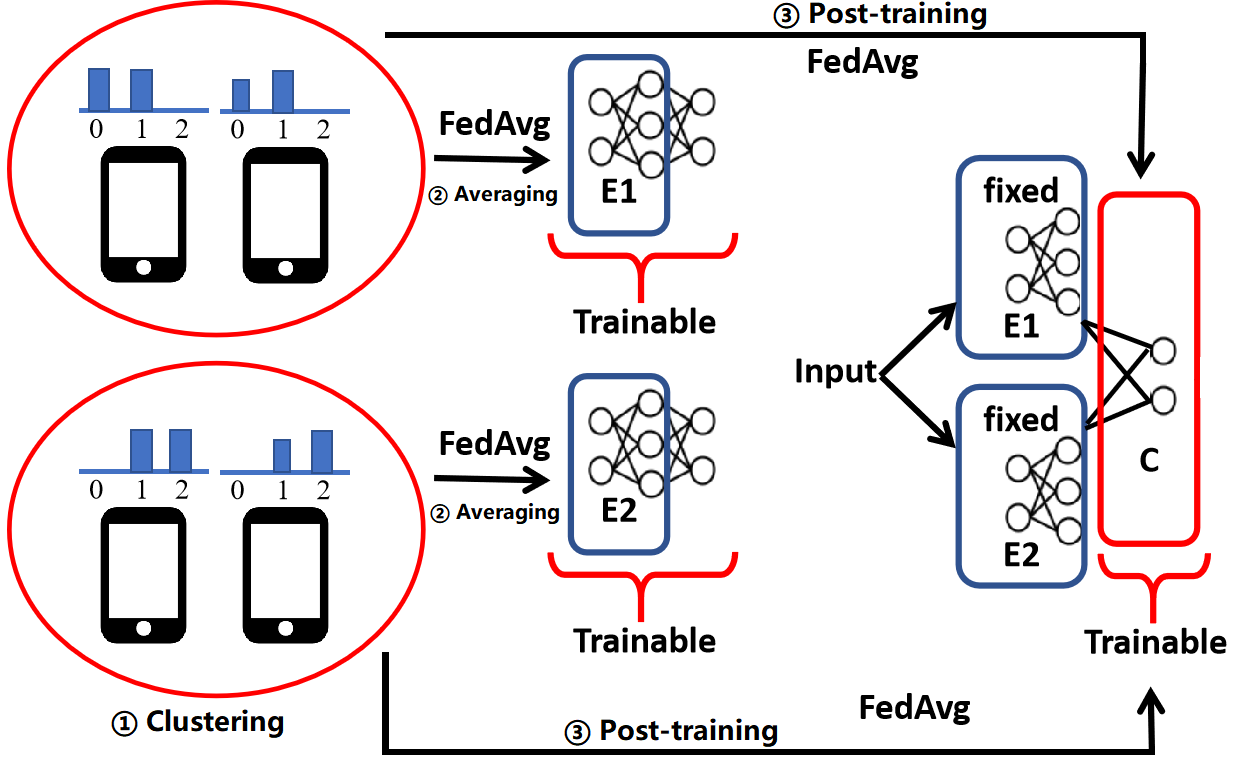}
    \caption{The workflow of FedConcat. (1) Clustering stage: clients are clustered based on label distributions; (2) Averaging stage: each cluster trains a model using FedAvg; (3) Post-training stage: all well-trained feature extractors (E1, E2) are concatenated. All clients train a global classifier (C) collectively with feature extractors fixed. For FedConcat-ID, label distributions are inferred in the clustering stage.}
    \label{fig:alg}
\end{figure}

\subsection{Proposed Algorithm}
Our framework is illustrated in Figure \ref{fig:alg}. It has three stages: clustering, averaging and post-training. First, clients with similar label distributions are grouped into same cluster. Then, each cluster performs FL to train a model that fits well inside the cluster. Finally, the server collects feature extractors of all clusters with their parameters fixed, and train a global classifier among all clients. The overall algorithm is shown in Algorithm \ref{alg:new}. In the following, we elaborate those stages in detail. 

\paragraph{Stage 1-A: Clustering with label distributions}
In order to alleviate the label imbalance problem, we perform clustering based on label distributions, so that each cluster hosts clients with similar label distributions. Formally, for client $i$, suppose there are $N_{i,j}$ samples of class $j$, and there are a total of $N_i = \sum_j N_{i,j}$ samples. Its label distribution is defined as vector
\begin{equation}
    P_i(y)=(\frac{N_{i,1}}{N_i},\frac{N_{i,2}}{N_i},...,\frac{N_{i,m}}{N_i} ) ,
    \label{eq:lb}
\end{equation}
where there are $m$ classes globally. In this paper, we use K-means algorithm \cite{Lloyd1982LeastSQ} to perform clustering. For the hyper-parameter $K$, one can utilize elbow method to select the best value. We use K-means as it is simple, popular and sufficiently good for our study. We have also tried other clustering methods, which are presented in Section \ref{sec:main_cfl}. {With clustering, we can control the number of different models generated by the clients, which helps to reduce the model size in our later concatenation.}

\paragraph{Stage 1-B: Clustering without label distributions}
If clients are unable to upload label distributions due to privacy concerns, we propose to utilize the uploaded local models of the first round to infer the approximate label distribution of each client. {In this way, we only upload trained models like FedAvg, which does not cause any extra privacy leakage.}

{During local training, if a class appears more frequently, the model is prone to output a higher probability for that class. Many works \cite{johnson2019survey, bahng2020learning} have observed that predictions of deep learning model are biased towards the majority classes of the training set.} Intuitively, if we put a large batch of random inputs into the client model, the average prediction can indicate the label distribution of training data. Thus, we generate random data (i.e., images that each pixel is randomly generated from range zero to one) and input these random data into each client model. Then, we calculate the average prediction probability for each class as the inferred label distribution of each client. Formally, denote the model of client $i$ as $f_i$. We randomly generate $r$ inputs $X_1,...,X_r$, the inferred distribution of client $i$
\begin{equation}
P_{i}^{ID}(y)=\frac{1}{r} \sum_{j=1}^r \sigma(f_i(X_j)),
    \label{eq:id}
\end{equation}
where $\sigma$ is the softmax function.

{We refer to this variant as FedConcat with Inferred Distribution (FedConcat-ID). A neural network classifier can be viewed as a function $p(Y|X)$ learned on its training data. In an ideal scenario, if the inputs $X$ are independent of $Y$, the equation $p(Y)=p(Y|X)$ holds true. The underlying intuition of Eq. \eqref{eq:id} is to employ uninformative inputs to approximate $p(Y)$. }

\paragraph{Stage 2: Averaging}
Within each cluster, we use FedAvg \cite{mcmahan2016communication} to train a model that fits well for such cluster. Inside a cluster, since the label distributions of the clients are similar, we expect the global model to have a good performance on the dominant classes of the cluster.

\paragraph{Stage 3: Post-training}
Now that we have $K$ models, we stack their encoders (all layers but the last layer) as the global feature extractor. Then we broadcast the global feature extractor to all clients for one time, and ask clients to jointly train a classifier using FedAvg, with the global feature extractor fixed. Since the encoder training is stopped, we can calculate the features of raw data in a forward pass for only one time. For other training rounds, we can directly feed features into the linear classifier to train it. Therefore in this stage, our major computation and communication happens only for the linear classifier.

\begin{algorithm}[ht]
\SetNoFillComment
\LinesNumbered
\SetArgSty{textnormal}
\KwIn{number of clients $N$, number of clusters $K$, number of training rounds of the encoder $T_e$, number of training rounds of the classifier $T_c$}
\KwOut{the final model $w$}

\If {FedConcat}{
$S_1, S_2, ..., S_K \leftarrow Kmeans({P_i(y)}_{i=1}^N)$ {// Perform K-means based on label distributions}
}
\If {FedConcat-ID}{
Initialize global model $f_g$

\For {$i=1,2,...,N$ in parallel} {
$f_i \leftarrow TrainLocal(f_{g})$ // Send model to each client for local training}

$S_1, S_2, ..., S_K \leftarrow Kmeans({P_{i}^{ID}(y)}_{i=1}^N)$ {// Infer label distributions by Eq. \eqref{eq:id} and perform K-means}

}

Initialize encoder $E_i$ and classifier $C_i$ for each cluster

\For {$t=1,2,..., T_e$}{
\For {$i=1,2,...,K$}{
$E_i, C_i \leftarrow FedAvg(\{E_i, C_i\}, S_i)$ // {Run FedAvg to train encoder and classifier for each cluster}
}
}

$E=\{E_1,E_2,...,E_{K}\}$ 

Initialize global classifier $C$

\For {$t=1,2, ..., T_c$}{
$C \leftarrow FedAvg(C, \bigcup_{i=1}^{K}S_i)$ // {Fix $E$ and run FedAvg on all clients to train $C$}
}

\Return final model $w=\{E, C\}$

\caption{FedConcat and FedConcat-ID}
\label{alg:new}
\end{algorithm}

\section{Theoretical Analysis and Discussion}

\subsection{Analyzing FedConcat by Information Bottleneck}
To answer the question why concatenating encoders works for extreme label skews, we refer to the information bottleneck theory \cite{shwartz2017opening}. Suppose a deep neural network $f$ is trained on dataset $D$. Denote a random train sample $(X,Y) \sim D$ where $X$ are input variables and $Y$ are desired outputs. Suppose the extracted features (representation before the last fully-connected layer) is $Z$, the neural network learns an encoder which minimizes 
\begin{equation}
   \mathbf{E}_{(X,Y)\sim D} [I(X;Z)-\beta I(Z;Y)],
\end{equation}
where $I(\cdot;\cdot)$ denotes the mutual information between two variables and $\beta$ is a positive trade-off parameter related to the task. 

In brief, the encoder of deep neural network aims to remember the features related to the target outputs (maximizing $I(Z;Y)$), while forgetting the information of inputs unrelated to the target outputs (minimizing $I(X;Z)$). 

Consider the label skews in federated learning. Suppose there are two clients with local datasets $D^1$ and $D^2$ respectively. Their locally trained encoders are $f_{e1}$ and $f_{e2}$. Denote the FedConcat encoder as $f_e(\cdot)=\{f_{e1}(\cdot), f_{e2}(\cdot)\}$ and the FedAvg encoder as $f_{avg}$. We have the following theorem.

\begin{theorem}
    $I(f_{avg}(X);Y) < I(f_{e}(X);Y)$, $\forall (X,Y)\sim D^1 \cup D^2$.
\end{theorem}

\begin{proof}
According to the information bottleneck theory, the local model of the first client minimizes
\begin{equation}
   \mathbf{E}_{(X^1,Y^1)\sim D^1} [I(X^1;f_{e1}(X^1))-\beta_1 I(f_{e1}(X^1);Y^1)].
\end{equation}

Similarly, the second client's local model minimizes
\begin{equation}
   \mathbf{E}_{(X^2,Y^2)\sim D^2} [I(X^2;f_{e2}(X^2))-\beta_2 I(f_{e2}(X^2);Y^2)].
\end{equation}

For a good global encoder $f_e$, it should minimize
\begin{equation}
   \mathbf{E}_{(X,Y)\sim D^1 \cup D^2} [I(X;f_{e}(X))-\beta I(f_{e}(X);Y)].
\end{equation}

For the mutual information between representation and target, no matter whether $(X,Y) \sim D^1$ or $(X,Y) \sim D^2$, we have 
\begin{equation}
    I(f_{e}(X);Y) \geq \max\{I(f_{e1}(X);Y), I(f_{e2}(X);Y)\},
\end{equation}
which means the representation of concatenated encoders are more related to the global targets than single locally optimized encoder. 

For the part of forgetting task-unrelated information, we have
\begin{equation}
    I(f_{e}(X);X) \geq \max\{I(f_{e1}(X);X), I(f_{e2}(X);X)\},
\end{equation}
which is a disadvantage according to information bottleneck theory. However, under extreme label skews, the encoder of a client can extract little information related to data of another client. According to experimental results in \citet{shwartz2017opening}, when deep neural network reaches convergence, the mutual information between last layer representation and raw input (i.e. $I(f_e(X);X)$) becomes very small, as compared to $I(f_e(X);Y)$. Therefore, we regard $I(f_e(X);Y)$ as the main part. We also justify such hypothesis by experiments, which are included in Appendix C.

For the averaging solution, under label skews, the averaged global model becomes very different from local optima. Thus, for $(X,Y) \sim D^1$, we have 
\begin{equation}
    I(f_{avg}(X);Y) < I(f_{e1}(X);Y) \leq I(f_{e}(X);Y).
    \label{eq:10}
\end{equation}

Similarly, for $(X,Y) \sim D^2$, we have
\begin{equation}
    I(f_{avg}(X);Y) < I(f_{e2}(X);Y) \leq I(f_{e}(X);Y).
    \label{eq:11}
\end{equation}

Combining Eq. \eqref{eq:10} and \eqref{eq:11}, for $(X,Y)\sim D^1 \cup D^2$, we have 
$I(f_{avg}(X);Y) < I(f_{e}(X);Y)$.

\end{proof}

\paragraph{Implication} Under label skews, the averaged encoder contains less mutual information about the labels compared with the concatenation of well-trained local encoders. This explains why the averaged model suffers from accuracy decay compared with local models, as shown in Figure \ref{fig:track}.

\subsection{Privacy}
FedConcat needs client label distribution information, therefore it is applicable when local label distribution is not sensitive. For scenarios that also consider label distribution privacy, users can adopt FedConcat-ID, which only transfers the models and provides the same privacy level as FedAvg. The inference attack towards client model is a complex topic and defense mechanisms against them fall outside this paper's scope. It is an interesting topic to explore more robust measures to prevent such breaches in future works.

\subsection{Communication}

Suppose the model size is $w$, and its last classifier layer size is $cw$ ($c<1$). For $T_e$ encoder rounds, each client's communication cost is $2T_ew$. Downloading the concatenated model costs $Kw$. Next, for classifier rounds each client costs $2T_cKcw$. The total cost of FedConcat is $2wN(T_e+K/2+cKT_c)$. For FedAvg with $T$ rounds, the cost is $2wNT$. Suppose $T=T_e+T_c$ (i.e., we train the same communication rounds for FedAvg and FedConcat). Given the same model size $w$ and the number of clients $N$, communication overhead can be saved by choosing small $c,K$, i.e. limiting the classifier size and number of cluster. Experimental results of accuracy with respect to communication cost are shown in Appendix D.2. {Our approach achieves higher accuracy and more stable convergence given the same communication cost with FedAvg and other baselines.}

\section{Experiments}
\label{sec:exp}

We have conducted extensive experiments to evaluate our method. Through comprehensive experiments, we find that our techniques consistently outperform baseline methods, delivering superior accuracy and more stable convergence under various label skews. Importantly, our methods remain effective in scenarios characterized by partial client participation, large models, and an increased number of clients. The introduced label inference and clustering components are both straightforward and effective. Due to the space limit, we put the following experiments in Appendix.

\begin{itemize}
    \item D.2: training curves with communication costs.
    \item D.3: training curves with computation costs.
    \item D.4: compared to baselines on the concatenated model.
    \item D.5: varying the number of clusters.
    \item D.6: more results on varying the clustering strategies.
    \item D.7: client partial participation settings. 
    \item D.8: more results and analysis on large models.
    \item D.9: more results on varying the number of clients.
    \item D.10: analyzing FedConcat-ID label inference module.
\end{itemize}

\begin{table*}[ht]
\centering
\caption{Experimental results of our methods compared with baselines with same communication cost. The model of baseline algorithms is the model of one cluster in FedConcat. {We repeat experiments with three different random seeds.}}
\label{tbl:main}
\resizebox{2\columnwidth}{!}{
\begin{tabular}{|c|c|c|c|c|c|c||c|c||c|}
\hline
Dataset & Partition & FedAvg & FedProx & MOON & FedRS & FedLC & FedConcat & FedConcat-ID & Centralized \\ \hline
\multirow{4}{*}{CIFAR-10} & $\#C=2$ & 53.6\%$\pm$0.8\% & 53.1\%$\pm$0.6\% & 53.4\%$\pm$1.3\% & 53.8\%$\pm$0.8\% & 49.8\%$\pm$0.6\% & \textbf{56.9\%$\pm$0.2\%}& \textbf{56.5\%$\pm$2.6\%} & \multirow{4}{*}{70.2\%$\pm$0.9\%} \\ \cline{2-9}
& $\#C=3$ & 57.6\%$\pm$0.6\% & 57.4\%$\pm$0.8\% & 58.6\%$\pm$1.8\% & 59.1\%$\pm$1.4\% & 58.1\%$\pm$1.4\% & \textbf{62.0\%$\pm$0.6\%}& \textbf{61.8\%$\pm$0.8\%}& \\ \cline{2-9}
& $p_k \sim Dir(0.1)$ & 53.0\%$\pm$1.0\% & 52.8\%$\pm$1.0\% & 53.1\%$\pm$3.5\% & 54.8\%$\pm$0.3\% & 53.7\%$\pm$0.8\% & \textbf{57.7\%$\pm$0.4\%}& \textbf{56.9\%$\pm$1.4\%}& \\ \cline{2-9}
& $p_k \sim Dir(0.5)$ & 59.9\%$\pm$0.5\% & 59.9\%$\pm$0.6\% & 61.2\%$\pm$1.8\% & 61.5\%$\pm$0.8\% & 61.5\%$\pm$1.1\% & \textbf{64.2\%$\pm$0.7\%}& \textbf{63.7\%$\pm$0.8\%}& \\ \hline
\multirow{4}{*}{SVHN} & $\#C=2$ & 82.8\%$\pm$0.5\% & 82.6\%$\pm$0.6\% & 83.0\%$\pm$0.1\% & 79.5\%$\pm$1.2\% & 75.7\%$\pm$2.3\% & \textbf{83.4\%$\pm$1.4\%}& \textbf{83.2\%$\pm$1.9\%}& \multirow{4}{*}{86.2\%$\pm$0.9\%}  \\ \cline{2-9}
& $\#C=3$ & 85.2\%$\pm$0.4\% & 85.2\%$\pm$0.6\% & 84.7\%$\pm$0.4\% & 85.7\%$\pm$0.4\% & 84.8\%$\pm$0.4\% & \textbf{86.0\%$\pm$0.9\%}& \textbf{86.1\%$\pm$0.5\%}& \\ \cline{2-9}
& $p_k \sim Dir(0.1)$ & \textbf{84.0\%$\pm$1.3\%} & 83.9\%$\pm$1.2\% & 83.7\%$\pm$1.1\% & 80.9\%$\pm$0.9\% & 78.8\%$\pm$0.8\% & 83.2\%$\pm$0.9\%& 82.9\%$\pm$0.3\% & \\ \cline{2-9}
& $p_k \sim Dir(0.5)$ & 87.2\%$\pm$0.2\% & 87.2\%$\pm$0.2\% & 87.2\%$\pm$0.2\% & 87.1\%$\pm$0.1\% & 87.1\%$\pm$0.5\% & \textbf{87.5\%$\pm$0.1\%}& \textbf{87.9\%$\pm$0.3\%}& \\ \hline
\multirow{4}{*}{FMNIST} & $\#C=2$ & 79.0\%$\pm$4.7\% & 81.8\%$\pm$2.8\% & 81.4\%$\pm$1.4\% & 78.3\%$\pm$1.8\% & 77.7\%$\pm$2.9\% & \textbf{84.4\%$\pm$0.6\%}& \textbf{83.0\%$\pm$2.0\%}& \multirow{4}{*}{88.4\%$\pm$0.2\%} \\ \cline{2-9}
& $\#C=3$ & 84.7\%$\pm$1.4\% & 85.7\%$\pm$0.3\% & 84.6\%$\pm$1.9\% & 85.8\%$\pm$0.8\% & 86.0\%$\pm$0.3\% & \textbf{87.1\%$\pm$0.2\%}& \textbf{86.6\%$\pm$0.1\%}& \\ \cline{2-9}
& $p_k \sim Dir(0.1)$ & 85.1\%$\pm$0.5\% & \textbf{85.2\%$\pm$0.8\%} & 85.0\%$\pm$0.1\% & 82.5\%$\pm$0.5\% & 82.1\%$\pm$1.6\% & 84.5\%$\pm$0.1\%& 85.0\%$\pm$0.4\% & \\ \cline{2-9}
& $p_k \sim Dir(0.5)$ & 87.5\%$\pm$0.3\% & 87.4\%$\pm$0.4\% & 87.4\%$\pm$0.5\% & 87.5\%$\pm$0.5\% & 87.5\%$\pm$0.4\% & \textbf{87.7\%$\pm$0.1\%}& \textbf{87.5\%$\pm$0.2\%}& \\ \hline
\end{tabular}
 }
\end{table*}

\subsection{Experiment Setups}
\paragraph{Datasets} 
Our experiments engage CIFAR-10 \citep{krizhevsky2009learning}, FMNIST \citep{xiao2017fashion}, SVHN \citep{netzer2011reading}, CIFAR-100 \citep{krizhevsky2009learning}, and Tiny-ImageNet datasets \citep{wu2017tiny} to evaluate our algorithm. The partition strategy from \citet{li2021federated} generates various non-IID settings, with a focus on label skews, given their significant accuracy degradation \cite{li2021federated}. In experiments, $\#C=k$ represents clients with $k$ unique labels, while $p_k \sim Dir(\beta)$ denotes the Dirichlet distribution sampled proportion of each class samples assigned to each client. By default, we divide whole dataset into 40 clients.

\paragraph{Baselines}
Our method is compared with well-known, open-sourced FL methods including FedAvg \cite{mcmahan2016communication}, FedProx \cite{lifedprox}, MOON \cite{li2021model}, FedRS \cite{li2021fedrs}, and FedLC \cite{zhang2022federated} which cater to label skews. The baseline settings replicate those from \citet{li2021federated}, running 50 rounds with each client training 10 local epochs per round, batch size 64, and learning rate 0.01 using SGD optimizer with weight decay $10^{-5}$.

\paragraph{Models}
To investigate diverse scenarios with different clients' capacities, we experiment with three different neural networks: simple CNN, VGG-9, and ResNet-50. By default, we use simple CNN. Appendix D.1 provides more details on the setups.

\subsection{Effectiveness}
\label{sec:main_exp}
We evaluate the performance of FedConcat and FedConcat-ID against other baselines. By default, our configuration includes a division of the 40 clients into $K=5$ clusters, and 200 rounds allocated for training the classifier. In order to equate the communication cost of FedConcat to that of 50 rounds of FedAvg, we set the encoder training to 31 rounds. For the FMNIST dataset, due to its image size differing from CIFAR-10 and SVHN, we record the test accuracy at classifier round 173 to maintain similar communication costs. 

The results in Table \ref{tbl:main} illustrate that FedConcat consistently outperforms the other five FL algorithms in most scenarios. Specifically, in the challenging CIFAR-10 dataset, both FedConcat and FedConcat-ID offer an average improvement of about 4\%. When considering partition types, notable improvements are evident in the more complex $\#C=2$ and $\#C=3$ partitions. For the Dirichlet-based label distributions of SVHN and FMNIST datasets, since the label skews are slight, the accuracy degradation of baseline algorithms from centralized training (200 epochs) is small. In such scenarios, our methods exhibit comparable accuracy with the baselines.

\subsection{Scalability}
\label{sec:scalab}
In this section, we evaluate the scalability of FedConcat. We keep the number of clusters $K=5$. During each round, a random selection of 50\% of the clients is sampled to participate in FL training. The results, as illustrated in Table \ref{tbl:scala_main}, confirm that both FedConcat and FedConcat-ID continue to outperform baseline algorithms with 100 or 200 clients and partial participation settings. 

\begin{table}[ht]
\centering
\caption{Scalability of FedConcat and FedConcat-ID compared with baselines on CIFAR-10 dataset. }
\label{tbl:scala_main}
\resizebox{\columnwidth}{!}{
\begin{tabular}{|c|c|c|c|c|c|c||c|c|}
\hline
\#Clients & {Partition} & {FedAvg} & {FedProx} & {MOON} & {FedRS} & {FedLC} & {FedConcat} & {FedConcat-ID} \\ \hline
\multirow{4}{*}{100} & $\#C=2$ & 44.4\% & 43.7\% & 44.4\% & \textbf{54.4\%} & 52.6\% & 51.0\% & 53.2\% \\ \cline{2-9}
& $\#C=3$ & 55.4\% & 54.6\% & 55.2\% & 56.5\% & 54.8\% & \textbf{59.0\%}& \textbf{59.3\%} \\ \cline{2-9}
& $p_k \sim Dir(0.1)$ & 45.0\% & 44.8\% & 45.3\% & 50.5\% & 50.0\% & 49.1\%& \textbf{50.7\%} \\ \cline{2-9}
& $p_k \sim Dir(0.5)$ & 58.1\% & 58.3\% & 57.7\% & 59.0\% & 58.1\% & \textbf{63.3\%}& \textbf{62.2\%}  \\ \hline

\multirow{4}{*}{200} & $\#C=2$ & 39.7\% & 40.9\% & 40.3\% & 46.2\% & 48.6\% & \textbf{48.9\%} & 44.6\%  \\ \cline{2-9}
& $\#C=3$ & 47.7\% & 48.2\% & 47.8\% & 51.5\% & 51.0\% & \textbf{53.1\%}& \textbf{51.8\%} \\ \cline{2-9}
& $p_k \sim Dir(0.1)$ & 40.5\% & 41.2\% & 40.9\% & 43.6\% & 43.7\% & \textbf{47.6\%}& \textbf{46.7\%}  \\ \cline{2-9}
& $p_k \sim Dir(0.5)$ & 53.9\% & 53.5\% & 53.9\% & 54.4\% & 53.7\% & \textbf{56.7\%}& \textbf{56.8\%}  \\ \hline

\end{tabular}
 }
\end{table}

\subsection{Experiments on Larger Model}

\begin{table*}[ht]
\centering
\caption{Experimental results on CIFAR-100 and Tiny-ImageNet with ResNet-50. We tune the weight decay among $\{0.00001,0.001,0.002,0.005\}$ for all algorithms. We present the average and standard deviation of the last 10 rounds.}
\label{tbl:resnet}

\resizebox{2\columnwidth}{!}{
\begin{tabular}{|c|c|c|c|c|c|c||c|c||c|}
\hline
Dataset & Partition & FedAvg & FedProx & MOON & FedRS & FedLC & FedConcat & FedConcat-ID & Centralized \\ \hline
\multirow{4}{*}{CIFAR-100} & $\#C=2$ & 8.4\%$\pm$1.2\% & 8.1\%$\pm$1.3\% & 8.0\%$\pm$1.0\% & 7.0\%$\pm$1.5\% & 3.8\%$\pm$0.4\%  & \textbf{18.5\%$\pm$0.7\%} & \textbf{16.4\%$\pm$0.8\%} & \multirow{4}{*}{70.5\%$\pm$0.1\%} \\ \cline{2-9}
& $\#C=3$ & 21.6\%$\pm$2.8\% & 18.7\%$\pm$6.0\% & 19.0\%$\pm$2.4\% & 19.2\%$\pm$1.7\% & 12.6\%$\pm$0.7\% & \textbf{34.4\%$\pm$1.1\%} & \textbf{32.9\%$\pm$0.6\%} &  \\ \cline{2-9}
& $p_k \sim Dir(0.1)$ & 52.4\%$\pm$4.9\% & 51.5\%$\pm$8.1\% & 55.2\%$\pm$3.3\% & 51.8\%$\pm$2.8\% & 50.5\%$\pm$2.9\% & \textbf{61.2\%$\pm$0.4\%} & \textbf{62.1\%$\pm$0.5\%} &  \\ \cline{2-9}
& $p_k \sim Dir(0.5)$ & 62.0\%$\pm$1.2\% & 61.2\%$\pm$1.3\% & 61.9\%$\pm$0.9\% & 61.9\%$\pm$1.2\% & 61.4\%$\pm$1.3\% & \textbf{66.3\%$\pm$0.1\%} & \textbf{65.6\%$\pm$0.1\%} &  \\ \hline

\multirow{4}{*}{Tiny-ImageNet} &  $\#C=2$ & 3.1\%$\pm$0.4\% & 2.7\%$\pm$0.3\% & 3.0\%$\pm$0.4\% & 3.1\%$\pm$0.1\% & 2.0\%$\pm$0.1\%  & \textbf{4.3\%$\pm$0.1\%} & \textbf{4.3\%$\pm$0.2\%} & \multirow{4}{*}{49.9\%$\pm$0.2\%}    \\ \cline{2-9}
& $\#C=3$ & 4.9\%$\pm$2.4\% & 5.1\%$\pm$1.0\% & 6.3\%$\pm$1.7\% & 3.3\%$\pm$0.4\% & 1.7\%$\pm$0.1\%  & \textbf{11.7\%$\pm$0.3\%} & \textbf{9.6\%$\pm$0.2\%} &  \\ \cline{2-9}
& $p_k \sim Dir(0.1)$ & 40.8\%$\pm$0.5\% & 40.8\%$\pm$0.6\% &  40.6\%$\pm$0.8\% & 39.7\%$\pm$0.7\% & 39.9\%$\pm$0.8\%  & \textbf{43.1\%$\pm$0.2\%} & \textbf{42.6\%$\pm$0.2\%} &  \\ \cline{2-9}
& $p_k \sim Dir(0.5)$ & 44.0\%$\pm$0.6\% & 44.2\%$\pm$0.7\% &  44.1\%$\pm$0.5\% & 43.6\%$\pm$0.9\% & 43.9\%$\pm$0.5\%  & \textbf{44.3\%$\pm$0.1\%} & 43.8\%$\pm$0.2\% &   \\ \hline

\end{tabular}
 }
\end{table*}

In this section, we conduct experiments on larger models, more clients and more complicated tasks, i.e. training ResNet-50 on CIFAR-100 and Tiny-ImageNet. There are 200 clients, and in each round a random selection of 20\% of the clients participate in the training. For baseline algorithms, we train 500 communication rounds. For FedConcat and FedConcat-ID, we train 480 encoder rounds and 500 classifier rounds to match the communication cost. Since ResNet-50 has huge memory and computation overhead, we set the number of clusters as 2 to constrain memory and computation costs.

New problems arise when training FedConcat with ResNet-50 on CIFAR-100 and Tiny-ImageNet. Firstly, local cluster models tend to overfit since each local cluster witnesses fewer data compared with training on all clients. Secondly, the cluster sizes become quite unbalanced since the label distribution points become more sparse in the high dimensional (100-D or 200-D) space. Some points may be so far from others that they are allocated into a very small cluster. Thirdly, the training process of the final classifier layer is more difficult to converge since there are many more hidden neurons in ResNet-50 than simple CNN. 

To address these problems, we increase the weight decay factor to tackle overfitting. Client members of the majority cluster are relocated to force each cluster to be balanced. At the beginning of the post-training stage, the global classifier is initialized with parameters of cluster classifiers to speed up convergence. We discuss these adaptations in detail and conduct ablation studies in Appendix D.8. 

As shown in Table \ref{tbl:resnet}, by tackling these issues, our methods achieve higher accuracy than baselines by an average of 8\% on CIFAR-100 and 2\% on Tiny-ImageNet.

\subsection{Effect of Clustering}
\label{sec:large-model}
If we concatenate all models from all clients without clustering, when there are a large number of clients, our global model can become very large. Large final model leads to heavy communication cost when training the global classifier. During inference, it also costs much computation time. Moreover, concatenating all models from all clients can suffer from unstable convergence and low test accuracy, since each client can have limited training samples. 

An experiment on CIFAR-10 is illustrated in Figure \ref{fig:wcluster}, where we show the test accuracy at each classifier round after training 100 encoder rounds. No clustering means each client trains a model for server to concatenate, where each client model is only trained with few samples and prone to overfit. Since those models are not well-trained, concatenating their encoders can hardly extract good features. From Figure \ref{fig:wcluster}, we can observe that clustering not only reduces communication cost, but also effectively improves model quality with more samples. 

\begin{figure}[ht]
    \centering
    \includegraphics[width=0.9\columnwidth]{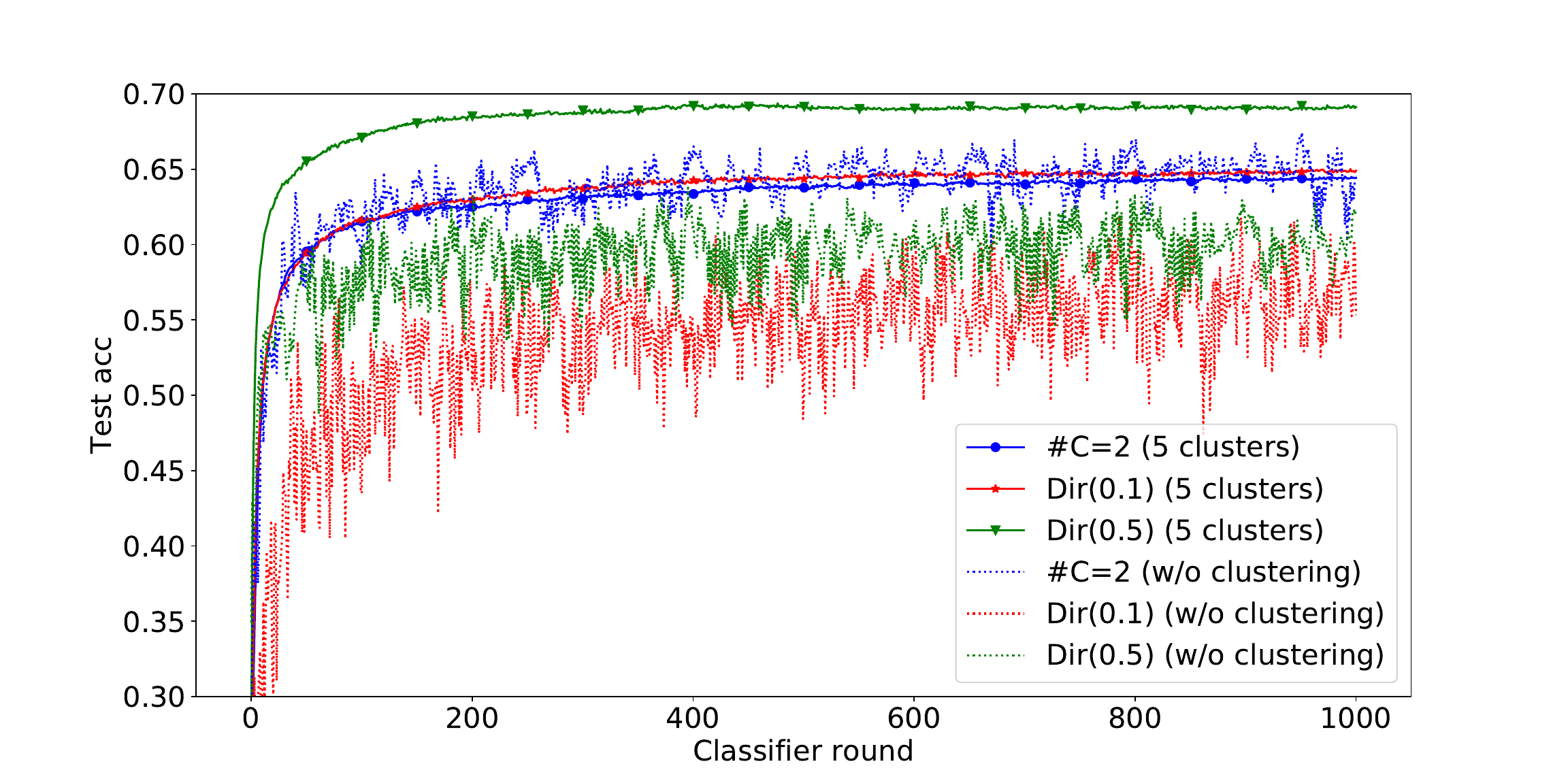}
    \caption{Training curves with clustering versus without clustering on CIFAR-10 (40 clients). }
    \label{fig:wcluster}
\end{figure}

FedConcat concatenates a cluster of models into a large model. If we directly train such a large model with other FL algorithms, does FedConcat still have advantage? 
To answer this question, we train baseline algorithms on the model with equivalent size to the final concatenated model of FedConcat. The training curves on CIFAR-10 dataset are shown in Figure \ref{fig:large}, which illustates that FedConcat still keeps its advantage when compared with the concatenated model.

\begin{figure}[ht]
    \centering
    \includegraphics[width=0.9\columnwidth]{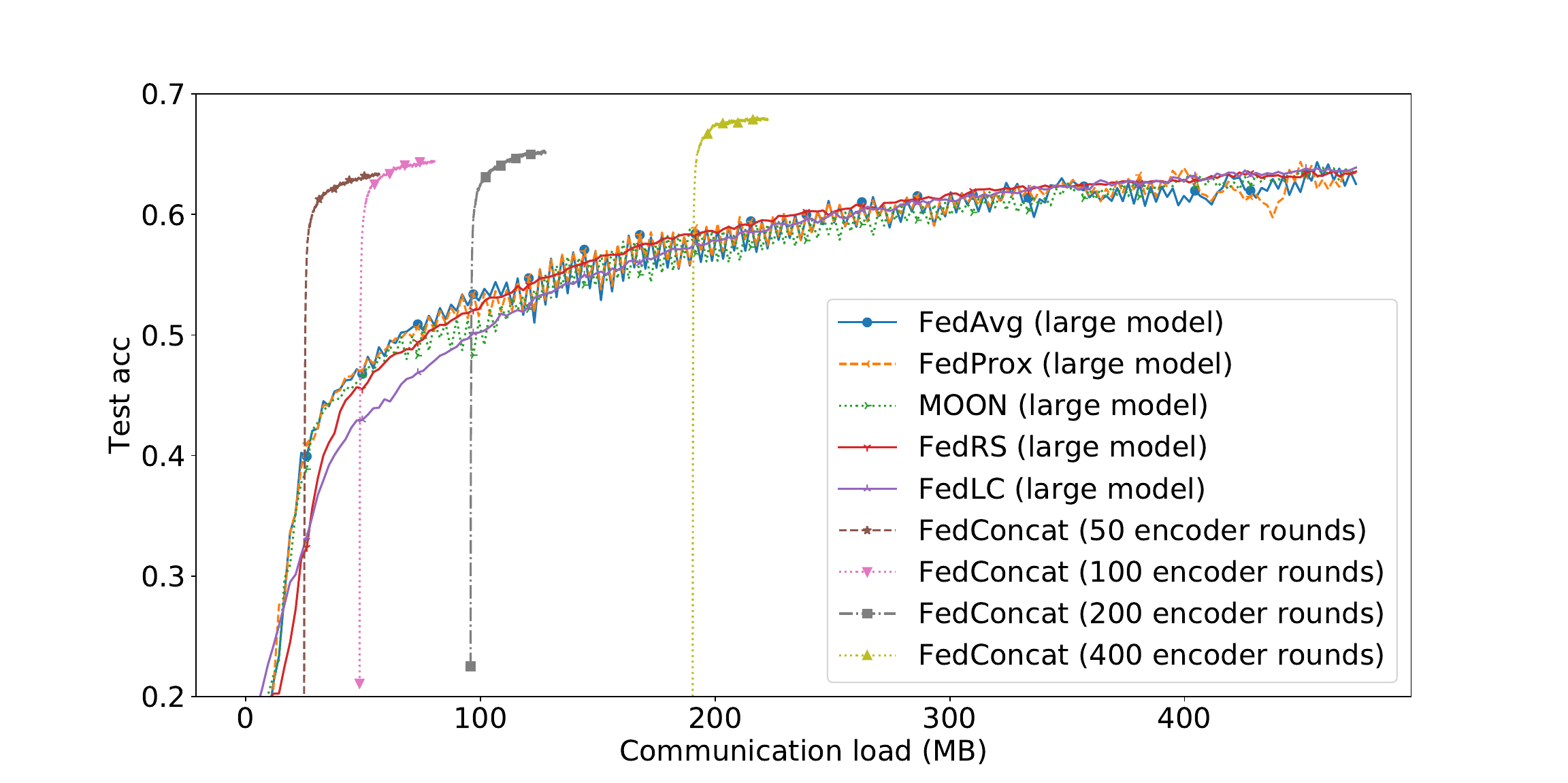}
    \ \ \ \ \ 
    \caption{Comparing with baselines on the final global model of FedConcat on CIFAR-10 (40 clients, $\#C=2$).}
    \label{fig:large}
\end{figure}

\subsection{Comparing with Other Clustered FL}
\label{sec:main_cfl}
{In this section, we employ other clustered FL algorithms during our clustering stage. We conduct experiments with three clustering-based methods including IFCA \cite{Ghosh2020AnEF}, recently proposed FedSoft \cite{ruan2022fedsoft} and FeSEM \cite{long2023multi}. We adhere to the same experimental settings as specified in Section \ref{sec:main_exp}. The results for the CIFAR-10 dataset are presented in Table \ref{tbl:cfl}. It can be observed that both FedConcat and FedConcat-ID outperform other clustering strategies. FeSEM incorporates an additional proximal loss term during local training, which results in an extra computational burden similar to FedProx. Both IFCA and FedSoft entail multiple times of communication cost as all cluster models are transferred to clients in each round. Thus, the clustering strategies of FedConcat and FedConcat-ID prove to be both effective and efficient. }

\begin{table}[htbp]
\centering
\caption{Experimental results of FedConcat and FedConcat-ID compared with other clustering strategies on CIFAR-10.}
\label{tbl:cfl}
\resizebox{\columnwidth}{!}{
\begin{tabular}{|c|c|c|c||c|c|}
\hline
Partition & IFCA & FedSoft & FeSEM & FedConcat & FedConcat-ID \\ \hline
$\#C=2$ & 54.5\%$\pm$0.9\% & 49.5\%$\pm$0.6\% & 54.0\%$\pm$1.5\% & \textbf{56.9\%$\pm$0.2\%}& \textbf{56.5\%$\pm$2.6\%} \\ \hline
$\#C=3$ & 60.0\%$\pm$0.9\%  & 56.1\%$\pm$0.9\% & 60.4\%$\pm$2.4\% & \textbf{62.0\%$\pm$0.6\%}& \textbf{61.8\%$\pm$0.8\%} \\ \hline
$p_k \sim Dir(0.1)$ & 49.5\%$\pm$5.7\%  & 51.8\%$\pm$1.2\% & 56.1\%$\pm$1.7\% & \textbf{57.7\%$\pm$0.4\%}& \textbf{56.9\%$\pm$1.4\%} \\ \hline
$p_k \sim Dir(0.5)$ & 63.3\%$\pm$0.2\% & 59.5\%$\pm$1.5\% & 63.1\%$\pm$0.8\% & \textbf{64.2\%$\pm$0.7\%}& \textbf{63.7\%$\pm$0.8\%} \\ \hline
\end{tabular}
 }
\end{table}

\section{Conclusion}
\label{sec:con}

In this paper, we propose to alleviate the accuracy decay induced by label skews in FL through concatenation. We show that in most cases, our methods can significantly outperform various state-of-the-art FL algorithms with smaller communication costs.
FedConcat can alleviate accuracy decay because it divides the hard problem (training a model among all clients under extreme label skews) into various easy problems (training one model within each cluster under alleviated label skews). Then it collects clues of easy problems (i.e., extracted features) to solve the hard original problem. Our approach brings new insights to the FL community to look for other aggregation approaches instead of averaging.

\section*{Acknowledgments}
This research is supported by the National Research Foundation, Singapore under its AI Singapore Programme (AISG Award No: AISG2-RP-2020-018). Any opinions, findings and conclusions or recommendations expressed in this material are those of the authors and do not reflect the views of National Research Foundation, Singapore. 

\bibliography{aaai24}

\newpage
\appendix
\subfile{del/appendix}

\end{document}

%% file: del/appendix.tex
\section{Discussion}

\subsection{More Related Works on FL}
\label{sec:more_work}
\paragraph{Improving model averaging} FedNova \cite{wang2020tackling} considers the heterogeneity in computational resources and uses weighted average for local steps, so that the global model is not biased towards clients with more local steps. FedFTG \cite{zhang2022fine} generates ``hard samples'' maximizing the discrepancy between local models and global model to finetune the global model. 

\paragraph{Improving data quality} FedRobust \cite{reisizadeh2020robust} utilizes adversarial learning and generates adversarial samples by gradient ascent to improve generalization. One can also generate fake samples by FedGAN \cite{Rasouli2020FedGANFG} to alleviate data heterogeneity and assist training. 

\paragraph{Improving client selection} Active federated learning \cite{Goetz2019ActiveFL} selects those clients with higher loss. Another track of works \cite{Song2019ProfitAF, Wang2020APA} select clients based on their Shapley value. Proper selection strategy can discard gradients uploaded by clients with low-quality data, and finally speed up convergence. FedCor \cite{tang2022fedcor} proposes to utilize the correlation between expected client loss changes to select clients.

\subsection{Clustering Techniques in FL} 
\label{sec:disc_cfl}
{Clustering is also a promising way to alleviate the non-IID data issue. CFL \cite{sattler2020clustered} uses the cosine similarity of uploaded gradient for clustering. FlexCFL \cite{duan2021flexible} clusters with approximated gradient similarity after applying SVD to improve efficiency. FeSEM \cite{long2023multi} applies EM algorithm to adaptively form clusters based on the distance between local models and cluster centers. FedSoft \cite{ruan2022fedsoft} extends to soft clustering where each client belongs to a mixture of clusters. IFCA \cite{Ghosh2020AnEF} proposes that each client trains multiple local models and adaptively adjusts cluster membership, by selecting the lowest loss model. FL+HC \cite{Briggs2020FederatedLW} proposes hierarchical clustering based on each client's local model update. FLT \cite{Rad2021FederatedLW} utilizes a public pre-trained encoder to extract key features of local data for clustering. After clustering, each client trains a specialized model for its own cluster. These methods are for personalized FL, where the goal is to train a good personalized local model for each client and the metric is to evaluate each local model on local test data (seen classes for each client). }

{From the perspective of speeding up convergence, FedCluster \cite{chen2020fedcluster} updates global model for multiple times by different clusters, utilizing a random clustering approach. However, due to the sequential updating of the global model by each cluster, simultaneous training of different clusters is not feasible, resulting in potentially long waiting time. }

{FedConcat pipeline consists of clustering, concatenation and post-training. Different from previous methods, our intuition of applying clustering is to combine the unique features of different clusters with different label distributions. We employ clustering techniques to identify similar clients and constrain the size of the final model.} 

{A comparison of various clustering strategies is detailed in Appendix \ref{sec:exp_cluster_FL}. In this paper, our focus is on typical FL settings, where the objective is to train an effective global model that encapsulates the collective knowledge from clients. The evaluation criterion is the global model's performance on the entire test dataset encompassing all classes. Therefore, we compare the global model accuracy of clustered FL algorithms utilizing our concatenation framework. }

\subsection{Clients with Only One Label}
FedConcat, as well as aforementioned benchmark FL algorithms, does not work for clients with only one label, i.e. $\#C=1$ partition in \citet{li2021federated}. According to experiments, if a neural network is trained on some clients with only one label for a few epochs, its loss becomes zero and such network consistently predicts that label regardless of input. This can be explained by the information bottleneck theory \cite{Tishby2015DeepLA}. Since there is only one label, the network can forget everything related to the feature and directly output that label. Thus many FL algorithms cannot work well for such extreme non-IID case. 

There are some research \cite{felix2020federated, Hosseini2020FederatedLO, diao2023towards} addressing $\#C=1$ scenario. With our framework, one may concatenate some local feature extractors (e.g. anomaly detection models) and collaboratively train a classifier. As long as one can train good feature extractors, our method can be applied to concatenate those extracted features to classify all labels. 



\subsection{Non-IID Data, IID Data and Model Selection}
{Our algorithm can bring more improvement when data is non-IID, while there is little improvement when data is IID. Based on a popular FL non-IID benchmark framework \cite{li2021federated}, we find out that no baseline algorithm can outperforms others in comprehensive label skews, which is also the problem found by the benchmark. Our FedConcat framework can achieve state-of-the-art performance in most label skew settings, suggesting it to be a promising solution for addressing the non-IID problem inherent in FL.}

If label distribution is IID among clients, the models submitted by each cluster only bring similar features, thus concatenating them offers little help. In fact, \citet{li2021federated} show that FedAvg is already good under IID data distribution setting. 

Considering the IID scenario, a future research direction of this algorithm is to select the best number of clusters $K$. 
This is an adaptive model selection problem based on data distribution. 

Further, our method allows model selection flexibility among clusters. Each cluster does not have to train the same model. Clusters with powerful computing devices and more data can train larger models for more rounds. Therefore, computing power and data quantity can also be taken into consideration during clustering.

\section{Differential Privacy (DP)}
\label{sec:dp}
In this section, we list the definitions and techniques for differential privacy \cite{dwork2011differential}.

\paragraph{Definition 1.} ($\epsilon$-DP) For $\epsilon>0$, a randomized function $f$
provides $\epsilon$-differential privacy if, for any datasets $D,D'$ that have only one single record different, for any possible output $O$,
\begin{equation}
    Pr[f(D)\in O] \leq e^{\epsilon} \cdot Pr[f(D')\in O]
\end{equation}

\paragraph{Definition 2.} (Sensitivity) Suppose $f$ is a function and $D,D'$ have only one record different. The sensitivity of $f$ is defined as 
\begin{equation}
    \Delta f= \max_{D,D'} \|f(D)-f(D')\|_1
\end{equation}

Here one record different means a database has one more record than another. For label distribution of a client, adding one more data to a non-empty database, can at most increase 0.5 for the probability of such class, i.e., the distribution changes from (1,0) to (0.5,0.5). Therefore, the sensitivity of label distribution function is 1.   

We utilize the Laplace mechanism \cite{dwork2014algorithmic} to achieve the $\epsilon-$DP.

\paragraph{Laplace Mechanism} For function $f:\mathcal{D}\rightarrow R^d$, function 
\begin{equation}
    F(D)=f(D)+Lap(0,\Delta f/\epsilon)
\end{equation}
provides $\epsilon$-DP, where $Lap(0,\Delta f/\epsilon)$ is sampled from Laplace distribution. 

\paragraph{Protecting label distributions with DP}
One possible concern is that the uploaded label distributions for clustering may leak additional information about training data. A possible solution is to apply differential privacy (DP) \cite{dwork2011differential,dwork2014algorithmic} to protect the uploaded auxiliary label distribution. DP is a rigorous and popular privacy metric, which guarantees that the output does not change with a high probability even though an input data record changes. Specifically, since the sensitivity of the label distribution function is 1, we add Laplacian noises with $\lambda=\frac{1}{\epsilon}$. We set $\epsilon=2.5$ that provides modest privacy guarantees. Results are shown in Table \ref{tbl:dp}. In most cases, DP-FedConcat can still outperform FedAvg, which shows that it is promising to combine our approach with DP to enhance the privacy. Compared with DP-FedConcat, FedConcat-ID adopts a more strict privacy setting without uploading label distributions.  

\begin{table}[!]
\centering
\caption{Experimental results of FedConcat adding DP with $\epsilon=2.5$. The model size of baseline algorithms is the same as the model of one cluster in FedConcat. {Bold means FedConcat adding DP can still outperform FedAvg. We repeat all experiments for three times with random seed 0,1,2.}}
\label{tbl:dp}
\resizebox{1.0\columnwidth}{!}{
\begin{tabular}{|c|c|c|c|}
\hline
Dataset & Partition  & DP-FedConcat & FedAvg  \\ \hline
\multirow{4}{*}{CIFAR-10} & $\#C=2$  & \textbf{54.7\%$\pm$0.6\%} & 53.6\%$\pm$0.8\%  \\ \cline{2-4}
& $\#C=3$  & \textbf{59.4\%$\pm$2.0\%} & 57.6\%$\pm$0.6\%  \\ \cline{2-4}
& $p_k \sim Dir(0.1)$ & \textbf{57.1\%$\pm$0.5\%} & 53.0\%$\pm$1.0\%  \\ \cline{2-4}
& $p_k \sim Dir(0.5)$ & \textbf{62.8\%$\pm$0.9\%} & 59.9\%$\pm$0.5\% \\ \hline
\multirow{4}{*}{SVHN} & $\#C=2$ & \textbf{83.9\%$\pm$1.4\%} & 82.8\%$\pm$0.5\%  \\ \cline{2-4}
& $\#C=3$ & \textbf{86.6\%$\pm$0.3\%} & 85.2\%$\pm$0.4\% \\ \cline{2-4}
& $p_k \sim Dir(0.1)$ & \textbf{84.0\%$\pm$0.2\%} & \textbf{84.0\%$\pm$1.3\%} \\ \cline{2-4}
& $p_k \sim Dir(0.5)$ & \textbf{87.9\%$\pm$0.2\%} & 87.2\%$\pm$0.2\% \\ \hline
\multirow{4}{*}{FMNIST} & $\#C=2$ & \textbf{82.5\%$\pm$2.4\%} & 79.0\%$\pm$4.7\% \\ \cline{2-4}
& $\#C=3$ & \textbf{85.9\%$\pm$0.5\%} & 84.7\%$\pm$1.4\% \\ \cline{2-4}
& $p_k \sim Dir(0.1)$& 82.8\%$\pm$3.3\% & \textbf{85.1\%$\pm$0.5\%} \\ \cline{2-4}
& $p_k \sim Dir(0.5)$& 87.2\%$\pm$0.1\% & \textbf{87.5\%$\pm$0.3\%} \\ \hline
\end{tabular}
}
\end{table}

\section{Further Analysis why Concatenating Encoders Works}
\label{sec:theory_justify}
\paragraph{Preliminary of mutual information} 
The mutual information between random variables $X$ and $Y$ is defined as 
\begin{equation}
    I(X;Y)=\sum_{x,y}p(x,y)\log \frac{p(x,y)}{p(x)p(y)}.
\end{equation}
It equals the information gain about $X$ by knowing $Y$. We mainly use the property
\begin{equation}
    I(X;Y_1,Y_2) \geq \max \{I(X;Y_1), I(X;Y_2)\}.
\end{equation}
By knowing both $Y_1$ and $Y_2$, the information gain about $X$ is no less than only knowing either $Y_1$ or $Y_2$.

\paragraph{Experiment setups} We also conduct some experiments to justify our theoretical analysis. For clear demonstration, we experiment on the first two clients of $\#C=2$ setting in FMNIST (random seed = 1). The first client only has classes 0 and 2, while the second client only has classes 1 and 9.

\paragraph{Exchanging encoders} To investigate how much related information one encoder carries to another client's task, we exchange the trained encoders of two clients, and record the classification loss (after training a linear classifier) and reconstruction loss (after training a decoder). We train 50 local epochs to ensure each local model has converged. The classification loss estimates $I(f_e(X);Y)$ and the reconstruction loss estimates $I(f_e(X);X)$. Smaller loss indicates more mutual information. 

\begin{table}[ht]
\centering
\caption{Comparing mutual information with and without exchanging encoders.}
\label{tbl:exc}
\newcommand{\y}{\ding{51}}
\newcommand{\n}{\ding{55}}
 \resizebox{\columnwidth}{!}{
\begin{tabular}{|c|c|c|c|}
\hline

\multirow{2}{*}{Encoder} & Exchange & Classification & Reconstruction \\
& or not & loss & loss \\ \hline
\multirow{2}{*}{$f_{e1}$} & \n & \textbf{0.0235} & \textbf{0.0646} \\ \cline{2-4}
& \y & 0.2345 & 0.0657 \\ \hline
\multirow{2}{*}{$f_{e2}$} & \n & \textbf{0.0007} & \textbf{0.0796} \\ \cline{2-4}
& \y & 0.0292 & 0.0816 \\ \hline

\end{tabular}
 }
\end{table}

As shown in Table \ref{tbl:exc}, exchanging encoders leads to larger loss, which indicates less mutual information. However, comparing the classification loss and reconstruction loss, we find that exchanging encoders can significantly increase classification loss by 10 times or even 40 times, while the increase in reconstruction loss is relatively small (only 2\%). It indicates that $I(f_e(X);Y)$ is much more sensitive to training task than $I(f_e(X);X)$. In other words, maximizing $I(f_e(X);Y)$ seems more important than minimizing $I(f_e(X);X)$. Concatenating encoders combines the strengths of local models since $I(f_e(X);Y)$ contains no less information than local models. Although the mutual information of representation and raw input $I(f_e(X);X)$ also increases, it does not matter much as $I(f_e(X);X)$ is much less sensitive. 

\paragraph{Testing encoders on combined datasets} Here we test the effectiveness of local encoders and concatenated encoder for global task. Both encoders are trained on client local dataset. Then we fix the encoder (or concatenated encoder) and train a linear classifier as well as a decoder on centralized training dataset, i.e. combining both client's local data. Results are shown in Table \ref{tbl:comb}. From classification loss, the concatenated encoder is much smaller, which verifies the improvement of mutual information $I(f_e(X);Y)$ by combining strengths of local models. For reconstruction loss, we can see that these rows are quite similar, which further justifies that $I(f_e(X);X)$ is much less sensitive and the concatenation of encoders can hardly increase $I(f_e(X);X)$. {Notice that $f_e$ has much smaller classification loss than $f_{avg}$, which verifies the effectiveness of concatenation compared with averaging under extreme label skews.}

\begin{table}[h]
\centering
\caption{Comparing encoders on global task. $f_e$ means the concatenation of $f_{e1}$ and $f_{e2}$, while $f_{avg}$ means the averaging of $f_{e1}$ and $f_{e2}$.}
\label{tbl:comb}
 \resizebox{0.7\columnwidth}{!}{
\begin{tabular}{|c|c|c|}
\hline

\multirow{2}{*}{Encoder} & Classification & Reconstruction \\
& loss & loss \\ \hline
$f_{e1}$ & 0.0810 & \textbf{0.0817} \\ \hline
$f_{e2}$ & 0.2456 & 0.0823 \\ \hline
$f_{avg}$ & 0.1072 & 0.0850 \\ \hline
$f_e$ & \textbf{0.0493} & 0.0819  \\ \hline

\end{tabular}
 }
\end{table}

\paragraph{Visualization of learned representation} We visualize the representation of single encoder and concatenated encoders in Figure \ref{fig:tsne}. As we can see, the concatenated encoder $f_e$ generates better representation than single local encoder $f_{e1}$ and $f_{e2}$, which is consistent with its low classification loss in Table \ref{tbl:comb}. 

We also visualize the representation of encoder using FedAvg in Figure \ref{fig:tsneavg}. Following the default setting, in each round clients train 10 local epochs on local dataset. Compared with the representation of concatenated encoder in Figure \ref{concat}, we find that one-round FedAvg encoder has much worse representation. After 3-5 rounds, FedAvg can get encoder comparable with FedConcat of only one round. This verifies that FedConcat can reduce the communication cost in extreme label skews. The concatenation operation is more efficient than averaging, in terms of keeping important features related to target task.

\begin{figure*}[]
    \centering
    \subfloat[$f_e$\label{concat}]{\includegraphics[width=0.33\textwidth]{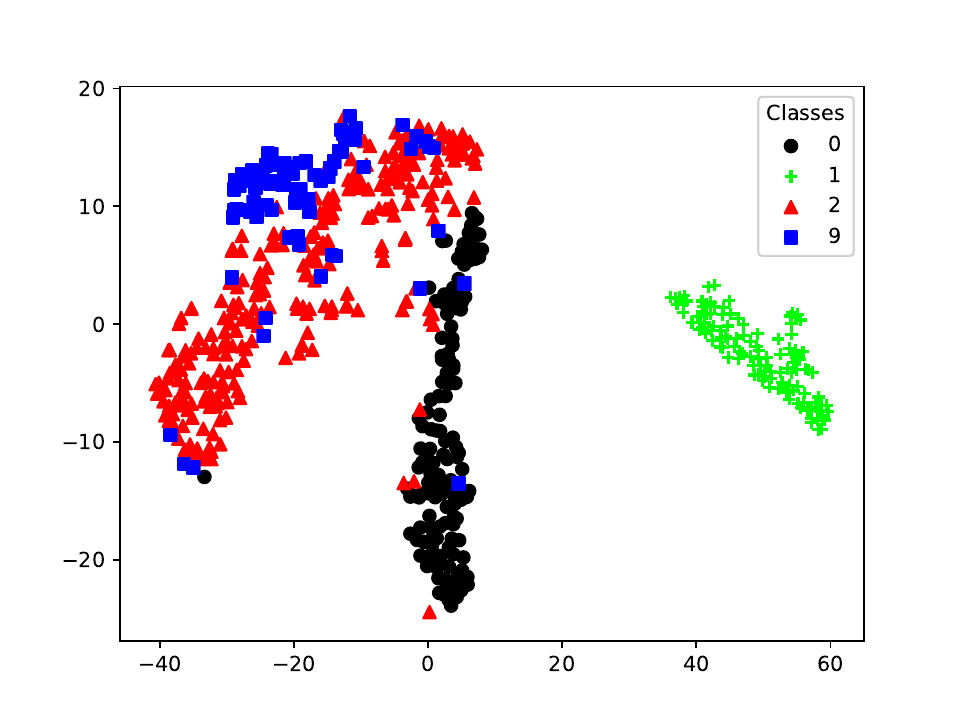}}
    \subfloat[$f_{e1}$]{\includegraphics[width=0.33\textwidth]{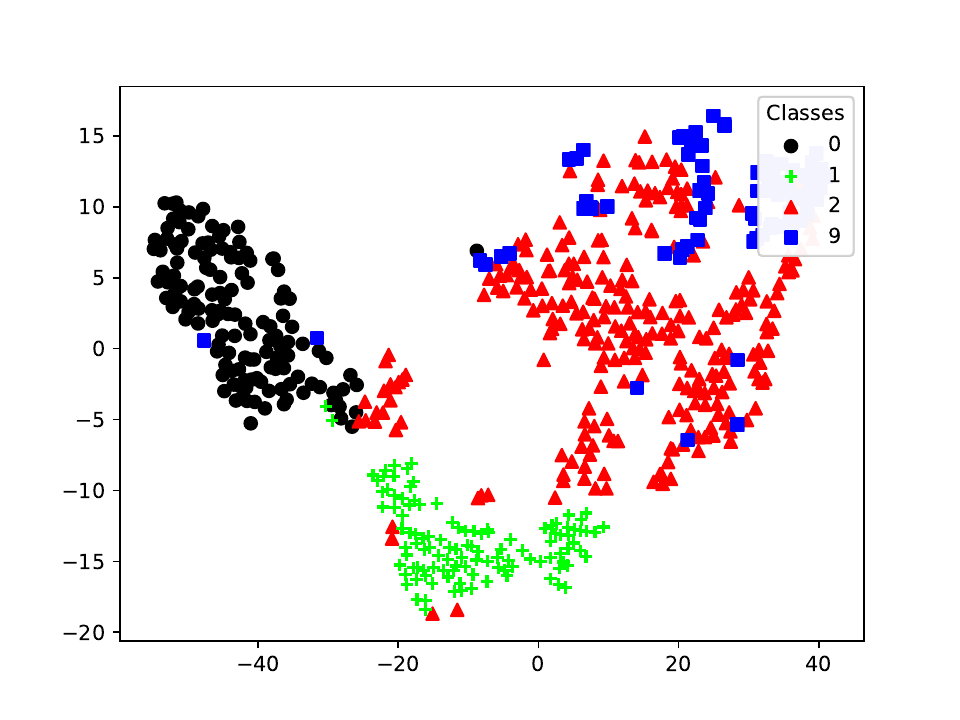}}
    \subfloat[$f_{e2}$]{\includegraphics[width=0.33\textwidth]{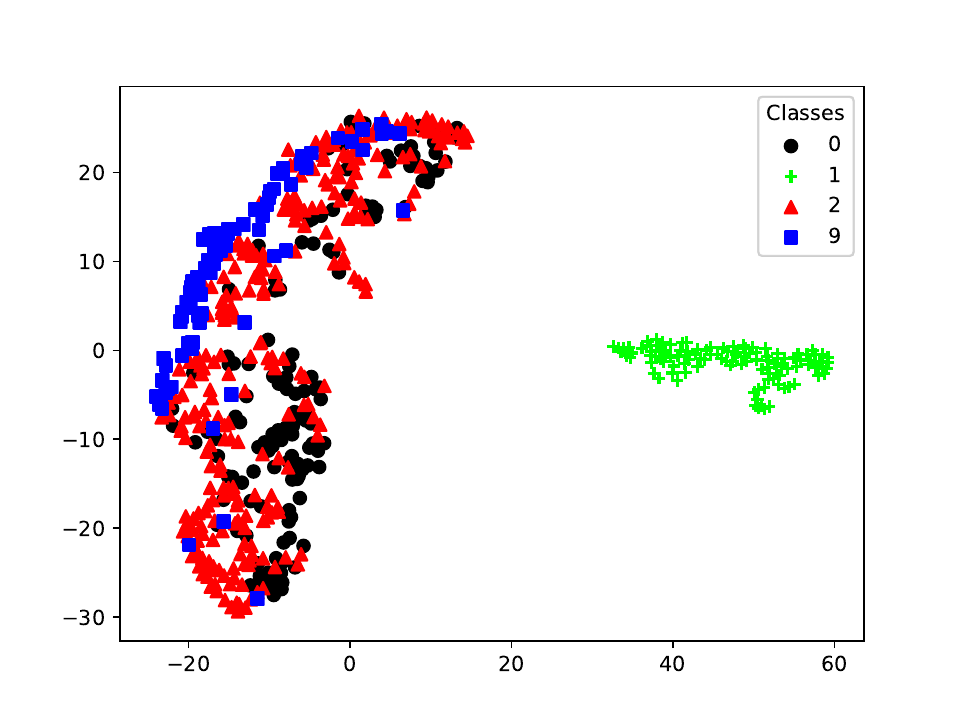}}

    \caption{Representation of different encoders on combined datasets.}
    \label{fig:tsne}
\end{figure*}

\begin{figure*}[]
    \centering
    \subfloat[After 1 round.]{\includegraphics[width=0.33\textwidth]{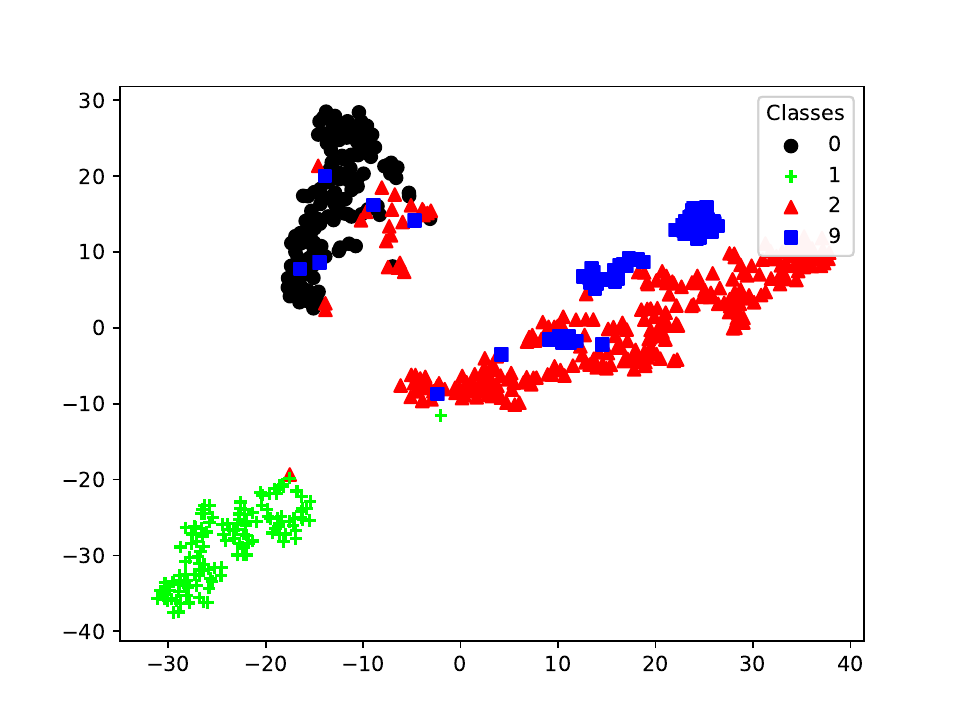}}
    \subfloat[After 3 rounds.]{\includegraphics[width=0.33\textwidth]{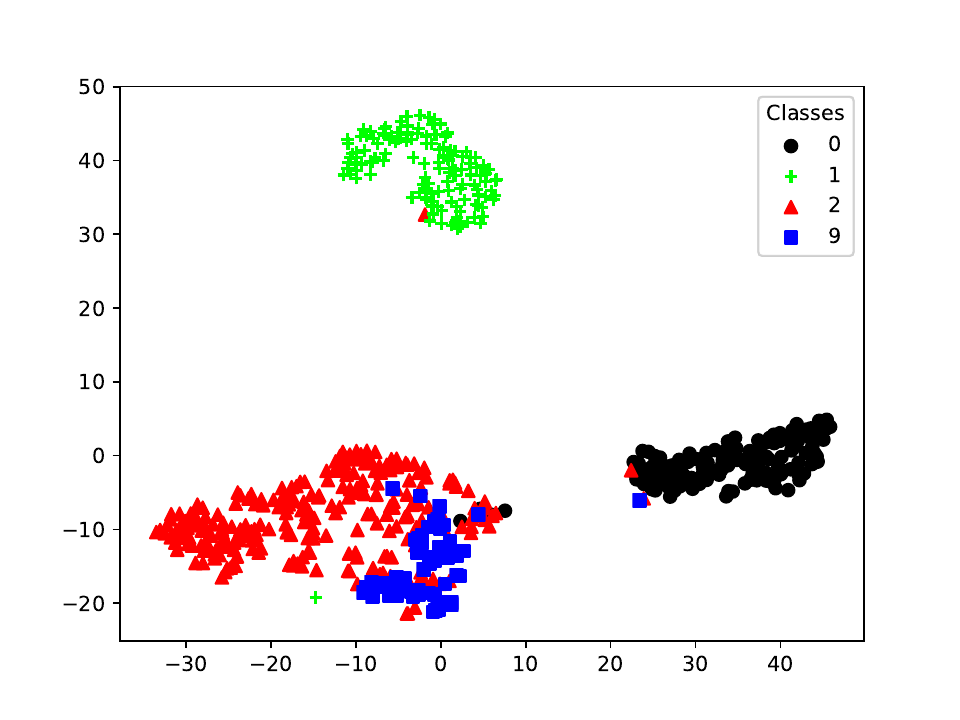}}
    \subfloat[After 5 rounds.]{\includegraphics[width=0.33\textwidth]{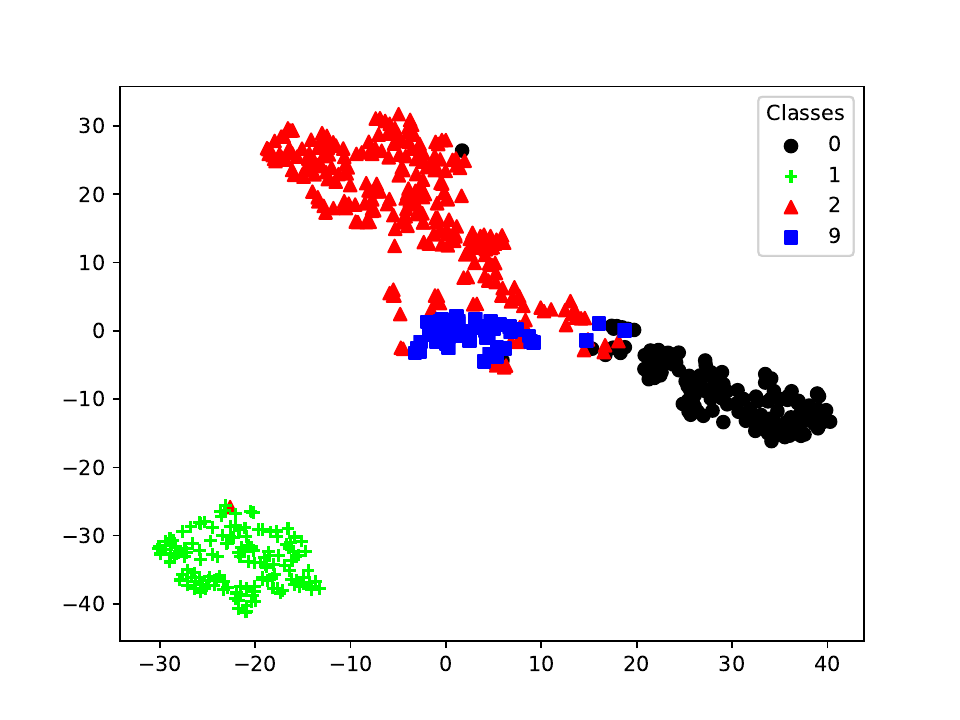}}

    \caption{Representation of FedAvg encoders on combined datasets.}
    \label{fig:tsneavg}
\end{figure*}


\section{Additional Experiments}
In this Appendix, we first introduce experimental setup details in Appendix \ref{sec:more_setup}. Then we analyze the training curves with respect to communication cost in Appendix \ref{sec:more_curve} and computation cost in \ref{sec:anal_computation}. We experiment with baseline algorithms on concatenated models in Appendix \ref{sec:more_concatenated_model}, cluster number in Appendix \ref{sec:exp_num_cluster}, FL clustering strategies in Appendix \ref{sec:exp_cluster_FL}, partial participation in Appendix \ref{sec:exp_partial}, large model in Appendix \ref{sec:exp_resnet} and scalability in Appendix \ref{sec:more_scalab}. Finally we analyze the label inference and training curves of FedConcat-ID in Appendix \ref{sec:more_exp_concat_id}. 

\subsection{More Details on Setups}
\label{sec:more_setup}
The statistics about the used datasets are summarized in Table \ref{tbl:datasets}. 

\begin{table}[ht]
\centering
\caption{Datasets used in our experiments.}
\vspace{-5pt}
\label{tbl:datasets}
\begin{tabular}{|c|c|c|c|}
\hline
Datasets & \#Training & \#Test & \#Classes \\ \hline
CIFAR-10 & 50,000 & 10,000  & 10 \\ \hline
SVHN & 73,257 & 26,032 & 10 \\ \hline
FMNIST & 60,000 & 10,000  & 10 \\ \hline
CIFAR-100 & 50,000 & 10,000 & 100 \\ \hline
Tiny-ImageNet & 100,000 & 10,000  & 200 \\ \hline
\end{tabular}
\end{table}

In our implementation of FedConcat-ID, since we experiment on image datasets and all pixels are pre-processed in the range of $[0,1]$, we generate 10,000 random images with each pixel uniformly randomly chosen in $[0,1]$ to infer label distribution by Eq. \eqref{eq:id}. 

By default, our model is the same 5-layer simple CNN model as in \citet{mcmahan2016communication,li2021model,li2021federated}. By default, the number of clients is set to 40 and the number of clusters is set to 5. We set local learning rate as 0.01, batch size as 64, optimizer as SGD with momentum $0.9$ and weight decay $10^{-5}$. In each communication round of averaging stage, each client trains 10 local epochs. In each communication round of post-training, each client trains 3 steps. 

We tune FedProx with best $\mu \in \{0.001,0.01\}$, MOON with best $\mu \in \{1,5\}$ and FedLC with best $\tau \in \{0.1,1,5\}$. We find that MOON with $\mu \in \{1,5\}$ does not work well for FMNIST and SVHN, so we further tune them with $\mu \in \{0.001,0.01,0.1\}$. For FedRS, we use the default setting $\alpha=0.5$ recommended by the authors. We follow FedSoft paper to set the importance estimation interval $\tau=2$. 
All the experiments are run on a machine with 8 * NVIDIA GeForce RTX 3090 GPUs.

\subsection{Training Curves and Communication Cost}
\label{sec:more_curve}
We plot training curves of some experiments in Figure \ref{fig:comm}. In these experiments, five baseline algorithms are trained 200 communication rounds each. FedConcat is trained with various different encoder rounds together with 1000 classifier rounds. It shows that FedConcat can outperform other methods with smaller communication cost and smoother convergence. Generally, the accuracy of FedConcat increases when we train encoder for more rounds. 

On scalability, as long as we keep the cluster number $K$, the communication cost of FedConcat increases linearly with the number of clients, just like FedAvg. In the averaging stage where each cluster performs FedAvg, the communication cost is exactly the same as FedAvg. At the start of the post-training stage, the size of concatenated encoders equals $K$ times single encoder size. During the post-training rounds, each client only communicate the final classifier layer with the server, which equals $K$ times single classifier size. Therefore, by limiting the cluster number $K$, the communication cost of FedConcat is scalable with respect to the number of clients.

\begin{figure*}[h]
    \centering
    \subfloat[CIFAR-10, \#C=2]{\includegraphics[width=0.33\textwidth]{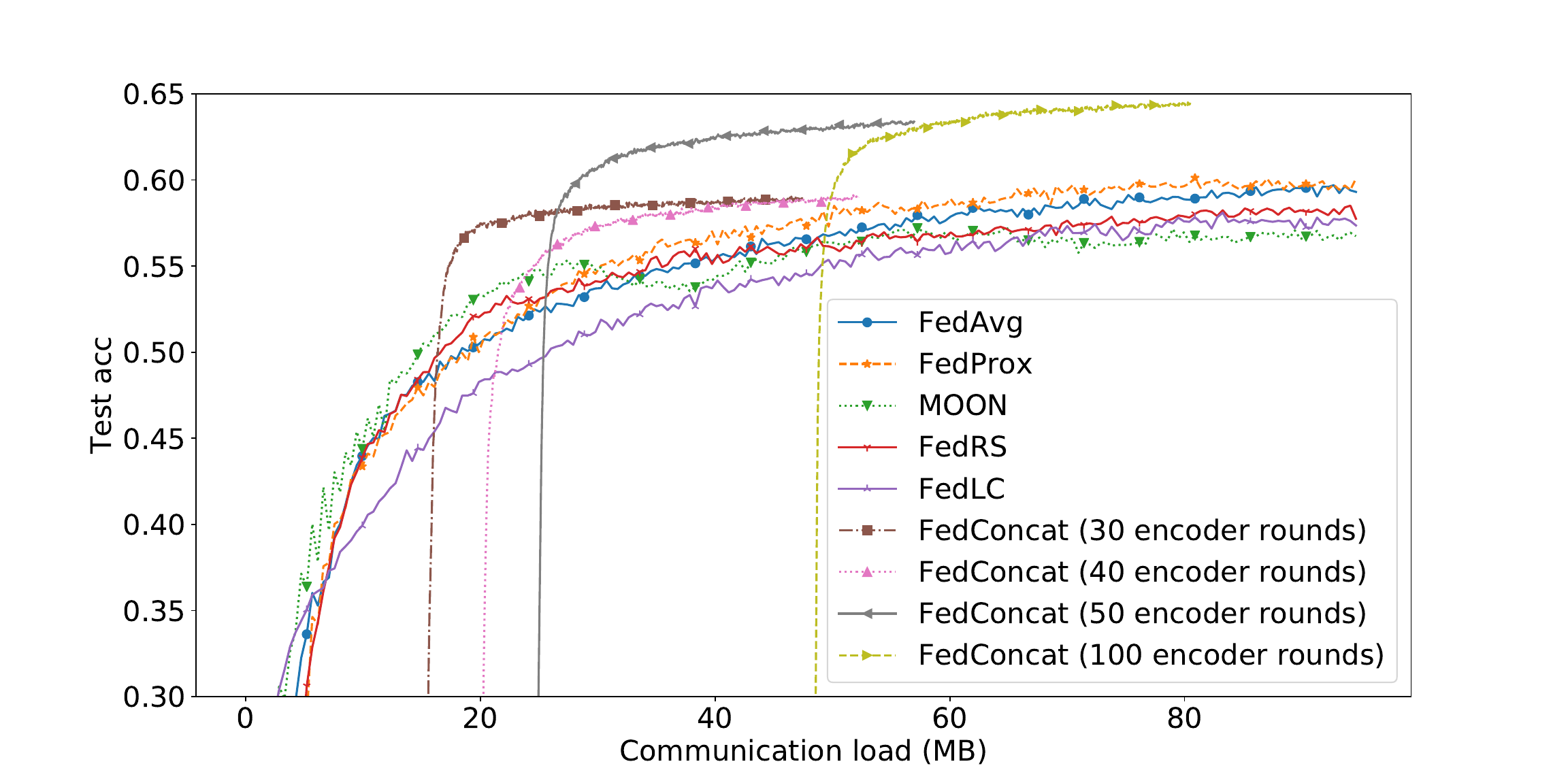}}
    \subfloat[SVHN, \#C=2]{\includegraphics[width=0.33\textwidth]{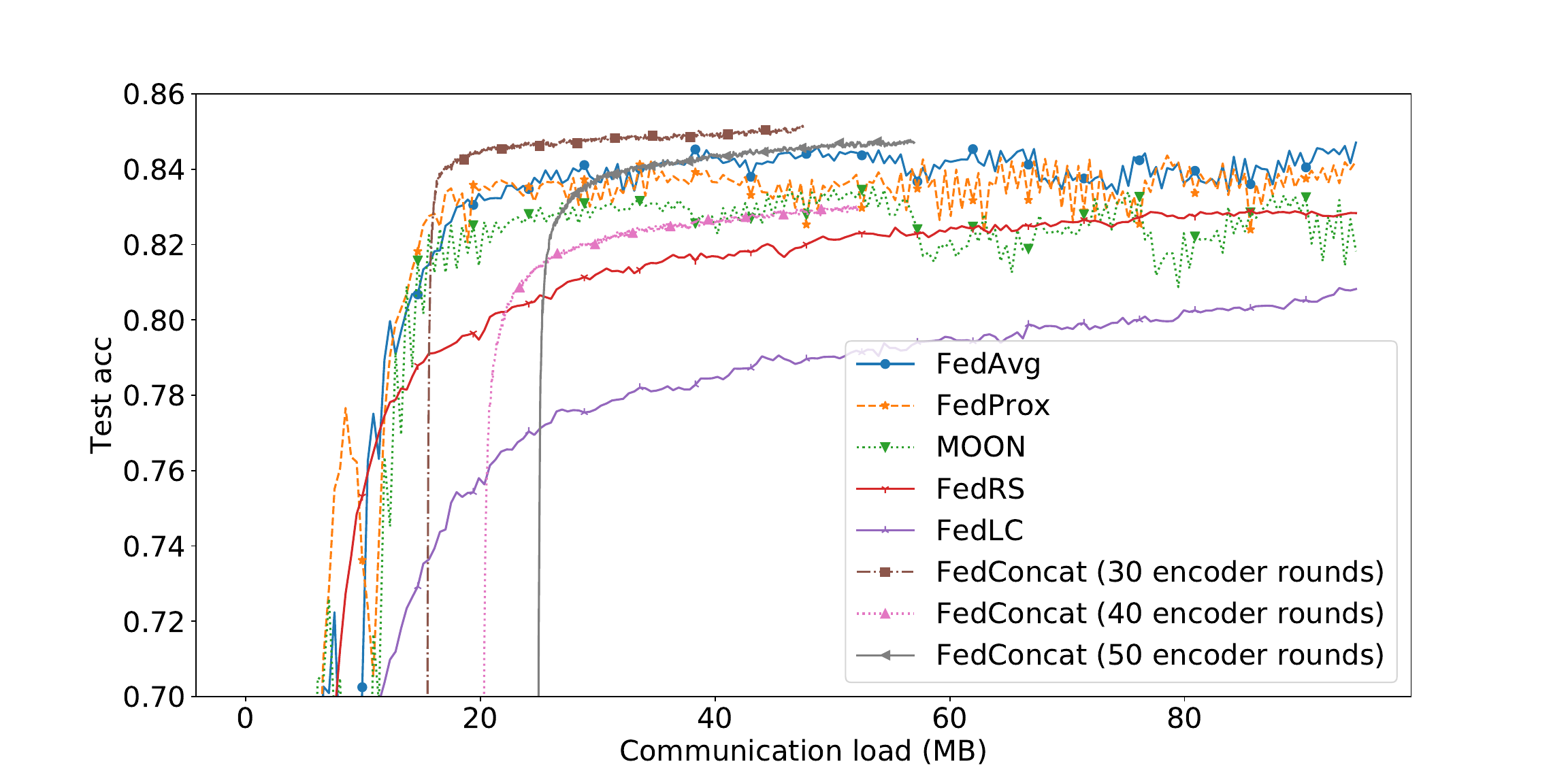}}
    \subfloat[FMNIST, \#C=2]{\includegraphics[width=0.33\textwidth]{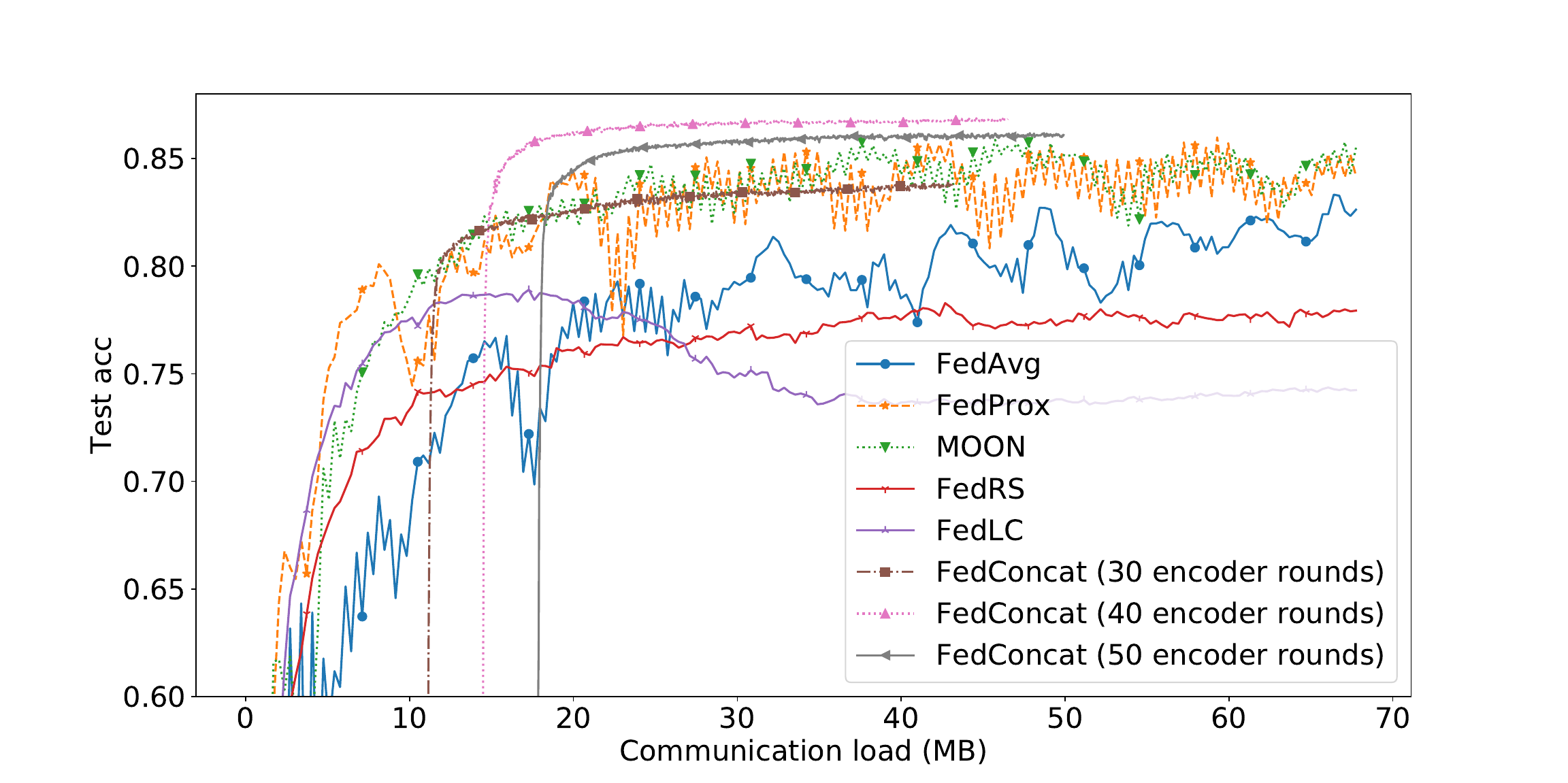}}
    \hfill
    \subfloat[CIFAR-10,$p_k \sim Dir(0.5)$]{\includegraphics[width=0.33\textwidth]{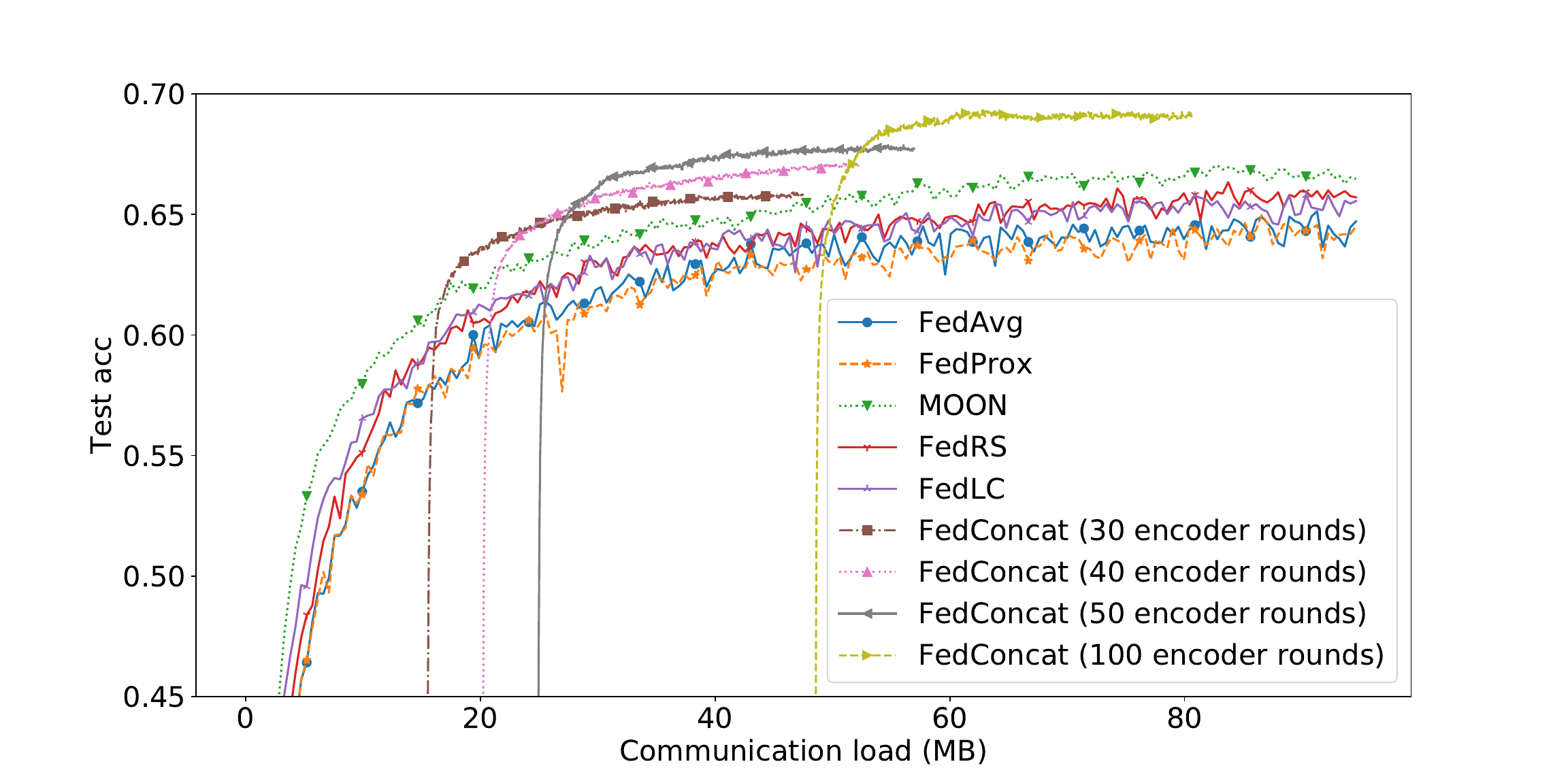}}
    \subfloat[SVHN, $p_k \sim Dir(0.5)$]{\includegraphics[width=0.33\textwidth]{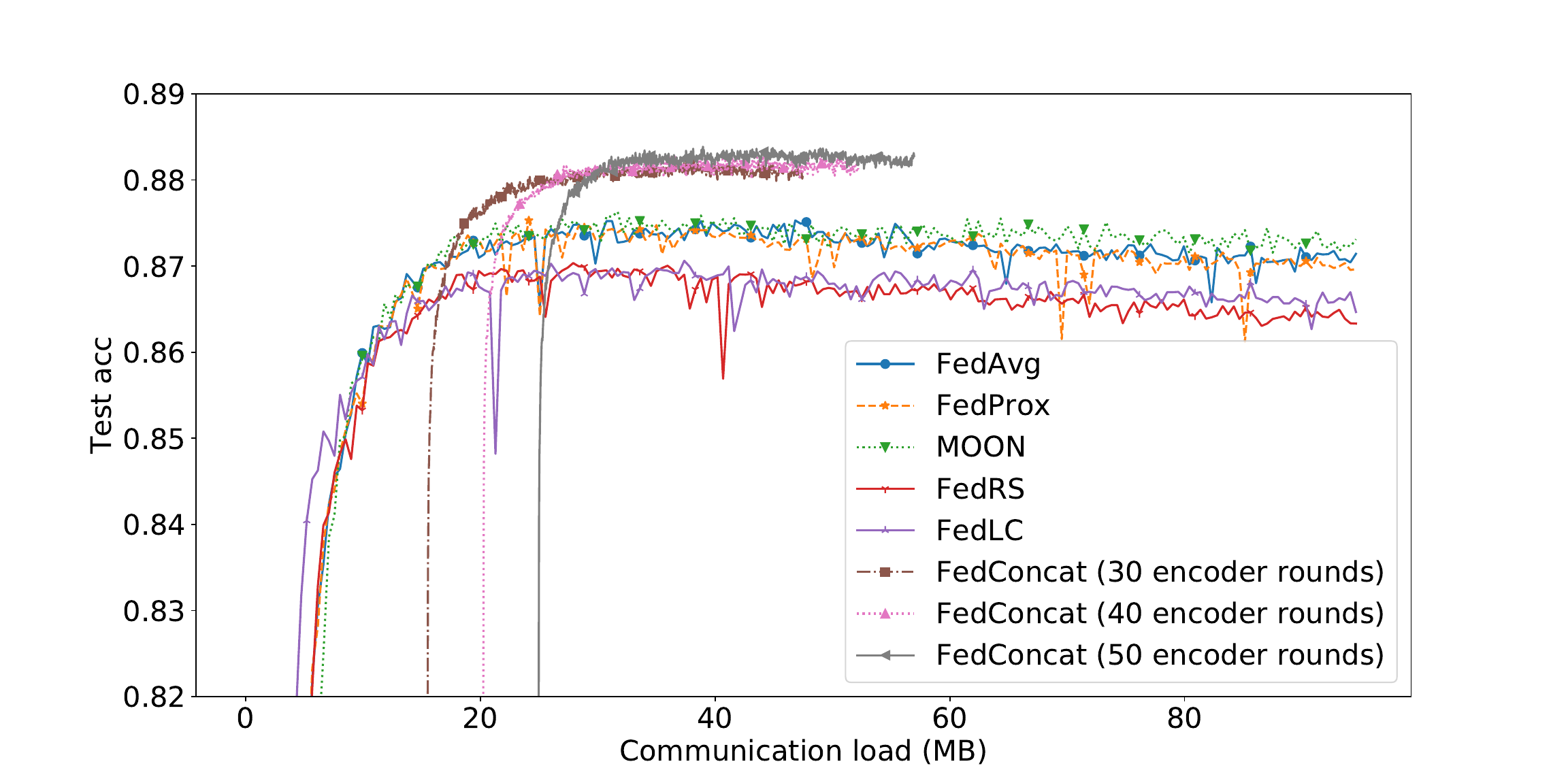}}
    \subfloat[FMNIST, $p_k \sim Dir(0.5)$]{\includegraphics[width=0.33\textwidth]{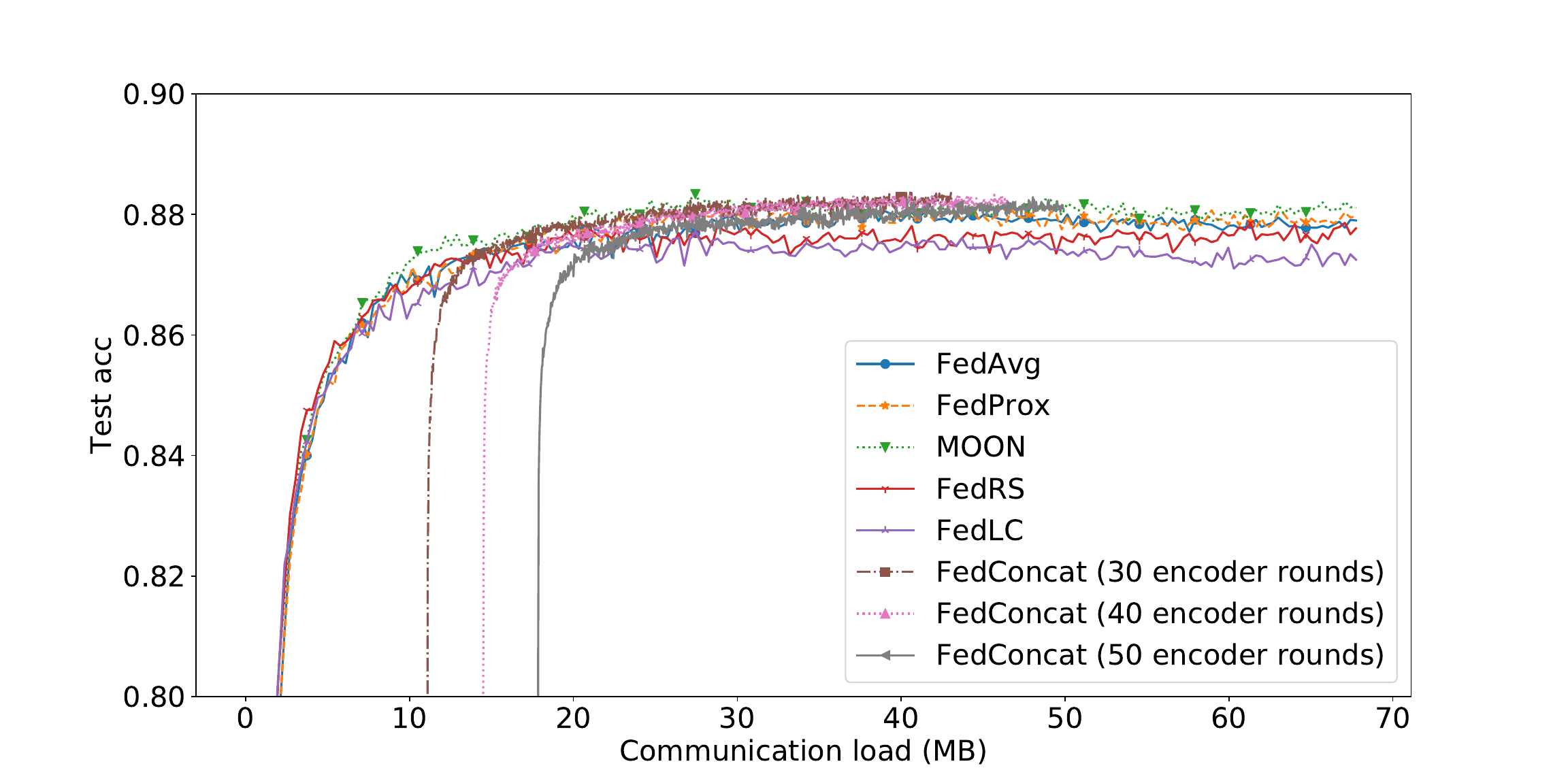}}
    \caption{The training curves under different settings.}
    \label{fig:comm}
\end{figure*}

\subsection{Analyzing Computation Cost of FedConcat and Baselines}
\label{sec:anal_computation}
{Besides the one-time clustering, computation cost of FedConcat differs with FedAvg only in the post-training stage. FedProx and MOON have additional regularization loss term in local training, which leads to additional computation compared with FedAvg. FedRS and FedLC perform slight changes on cross entropy loss to address label skews, so their computation costs are similar to that of FedAvg. We show the training curve with respect to the time in Figure \ref{fig:time}. As we can see, FedConcat can achieve better accuracy in shorter time compared with baselines. }

\begin{figure}[]
    \centering
    \includegraphics[width=\columnwidth]{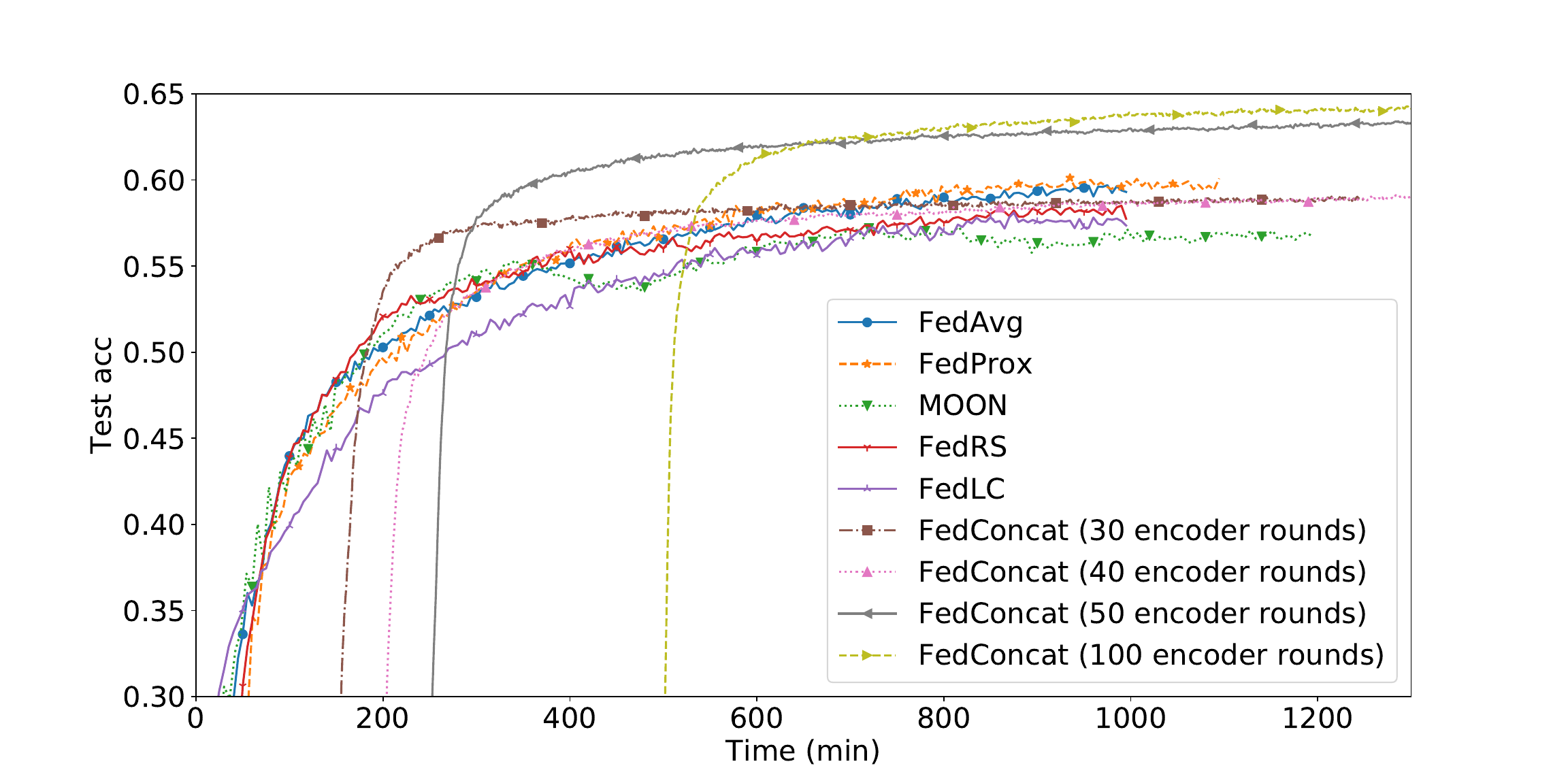}
    \caption{Training curves with respect to time. We divide CIFAR-10 into 40 clients by \#C=2 partition. }
    \vspace{-10pt}
    \label{fig:time}
\end{figure}

\subsection{Results of Baselines Running on Concatenated Models}
\label{sec:more_concatenated_model}
In this section we train the large CNN model (i.e. the model with equivalent size to the final concatenated model of FedConcat) for 50 communication rounds via FedAvg, FedProx, MOON, FedRS and FedLC. In order to control the same communication cost, we train 200 extra encoder rounds for FedConcat. The results are shown in Table \ref{tbl:large}. 
We can observe that FedConcat and FedConcat-ID still outperform other methods in most cases. The improvement of FedConcat is significant, which is higher than 5\% accuracy in many settings.

\begin{table*}[h]
\centering
\caption{Experimental results of FedConcat and FedConcat-ID compared with FedAvg, FedProx, MOON, FedRS and FedLC with same communication cost. The model size of baseline algorithms is the same as the final global model in FedConcat. }

\label{tbl:large}
\resizebox{2.1\columnwidth}{!}{
\begin{tabular}{|c|c|c|c|c|c|c||c|c|}
\hline
Dataset & Partition & FedAvg & FedProx & MOON & FedRS & FedLC & FedConcat & FedConcat-ID \\ \hline
\multirow{4}{*}{CIFAR-10} & $\#C=2$ & 54.2\%$\pm$2.6\% & 52.5\%$\pm$1.7\% & 55.6\%$\pm$4.5\% & 55.4\%$\pm$1.7\% & 52.5\%$\pm$0.9\% & \textbf{64.0\%$\pm$0.8\%} &\textbf{63.5\%$\pm$1.8\%} \\ \cline{2-9}
& $\#C=3$ & 62.3\%$\pm$0.5\% & 62.8\%$\pm$1.3\% & 63.9\%$\pm$3.3\% & 63.9\%$\pm$0.8\% & 61.6\%$\pm$1.0\% & \textbf{67.1\%$\pm$1.0\%}& \textbf{67.2\%$\pm$0.2\%} \\ \cline{2-9}
& $p_k \sim Dir(0.1)$ & 55.7\%$\pm$1.3\% & 55.7\%$\pm$0.9\% & 57.0\%$\pm$3.0\% & 58.0\%$\pm$0.3\% & 49.5\%$\pm$1.6\% & \textbf{63.0\%$\pm$1.7\%}& \textbf{62.2\%$\pm$0.6\%} \\ \cline{2-9}
& $p_k \sim Dir(0.5)$ & 65.3\%$\pm$1.0\% & 65.4\%$\pm$1.0\% & 67.0\%$\pm$3.5\% & \textbf{68.3\%$\pm$0.4\%} & 67.1\%$\pm$0.4\% & 67.9\%$\pm$0.3\% & 68.1\%$\pm$0.5\%  \\ \hline
\multirow{4}{*}{SVHN} & $\#C=2$ & 82.5\%$\pm$0.8\% & 82.0\%$\pm$1.4\% & 82.5\%$\pm$1.1\% & 75.5\%$\pm$2.4\% & 67.1\%$\pm$5.5\% & \textbf{85.7\%$\pm$0.1\%}& \textbf{86.4\%$\pm$0.1\%} \\ \cline{2-9}
& $\#C=3$ & 84.4\%$\pm$0.5\% & 84.6\%$\pm$0.6\% & 84.3\%$\pm$0.5\% & 84.8\%$\pm$0.6\% & 83.8\%$\pm$0.4\% & \textbf{87.4\%$\pm$0.2\%}& \textbf{87.5\%$\pm$0.2\%} \\ \cline{2-9}
& $p_k \sim Dir(0.1)$ & 83.8\%$\pm$0.8\% & 83.8\%$\pm$0.8\% & 82.2\%$\pm$1.9\% & 79.2\%$\pm$0.7\% & 75.2\%$\pm$0.6\% & \textbf{85.8\%$\pm$0.6\%}& \textbf{86.0\%$\pm$0.2\%} \\ \cline{2-9}
& $p_k \sim Dir(0.5)$ & 86.9\%$\pm$0.2\% & 86.9\%$\pm$0.3\% & 87.4\%$\pm$0.2\% & 86.5\%$\pm$0.9\% & 86.3\%$\pm$0.4\% & \textbf{88.4\%$\pm$0.1\%}& \textbf{88.6\%$\pm$0.1\%} \\ \hline
\multirow{4}{*}{FMNIST} & $\#C=2$ & 78.4\%$\pm$5.8\% & 79.4\%$\pm$7.4\% & 79.5\%$\pm$2.5\% & 72.8\%$\pm$9.0\% & 67.8\%$\pm$7.7\% & \textbf{85.1\%$\pm$1.8\%} & \textbf{85.5\%$\pm$1.6\%} \\ \cline{2-9}
& $\#C=3$ & 83.5\%$\pm$1.5\% & 85.2\%$\pm$1.0\% & 84.8\%$\pm$2.2\% & 85.5\%$\pm$0.5\% & 84.3\%$\pm$1.2\% & \textbf{87.9\%$\pm$0.2\%}& \textbf{88.1\%$\pm$0.2\%} \\ \cline{2-9}
& $p_k \sim Dir(0.1)$ & 85.3\%$\pm$0.5\% & 85.1\%$\pm$0.8\% & 84.9\%$\pm$0.3\% & 81.6\%$\pm$0.9\% & 80.1\%$\pm$0.1\% & \textbf{86.8\%$\pm$0.3\%}& \textbf{86.5\%$\pm$0.1\%} \\ \cline{2-9}
& $p_k \sim Dir(0.5)$ & 87.7\%$\pm$0.8\% & 87.7\%$\pm$0.8\% & 87.5\%$\pm$0.3\% &  87.8\%$\pm$0.5\% & 87.6\%$\pm$0.4\% & \textbf{88.3\%$\pm$0.1\%}& \textbf{88.1\%$\pm$0.3\%} \\ \hline
\end{tabular}
}
\end{table*}

\subsection{Experiments on the Number of Clusters} 
\label{sec:exp_num_cluster}
By default, we set cluster number $K=5$. Here we conduct experiments with different $K$ to investigate the effect of number of clusters. By default, we train FedConcat with 100 encoder rounds and 1000 classifier rounds. For the baseline, we train FedAvg with 200 communication rounds. We experiment with $K \in \{1,2,5,10\}$, as well as estimated elbow value of specific label distribution. 

We show results of CIFAR-10 and FMNIST in Figure \ref{fig:cluster}. Small cluster number leads to similar performance of FedAvg, while too large cluster number may cause overfitting problem of each cluster, which leads to unstable training curve and degraded performance. Choosing cluster number as elbow value can generally achieve good accuracy. However, sometimes it is hard to identify elbow value due to smooth total distance decrease, and we may identify a rather large elbow value resulting in unstable training curve. We encounter this problem in one experiment on FMNIST dataset. On such case, one can consider decreasing the cluster number. 

\begin{figure*}[h]
    \centering
    \subfloat[CIFAR-10, \#C=2]{\includegraphics[width=0.33\textwidth]{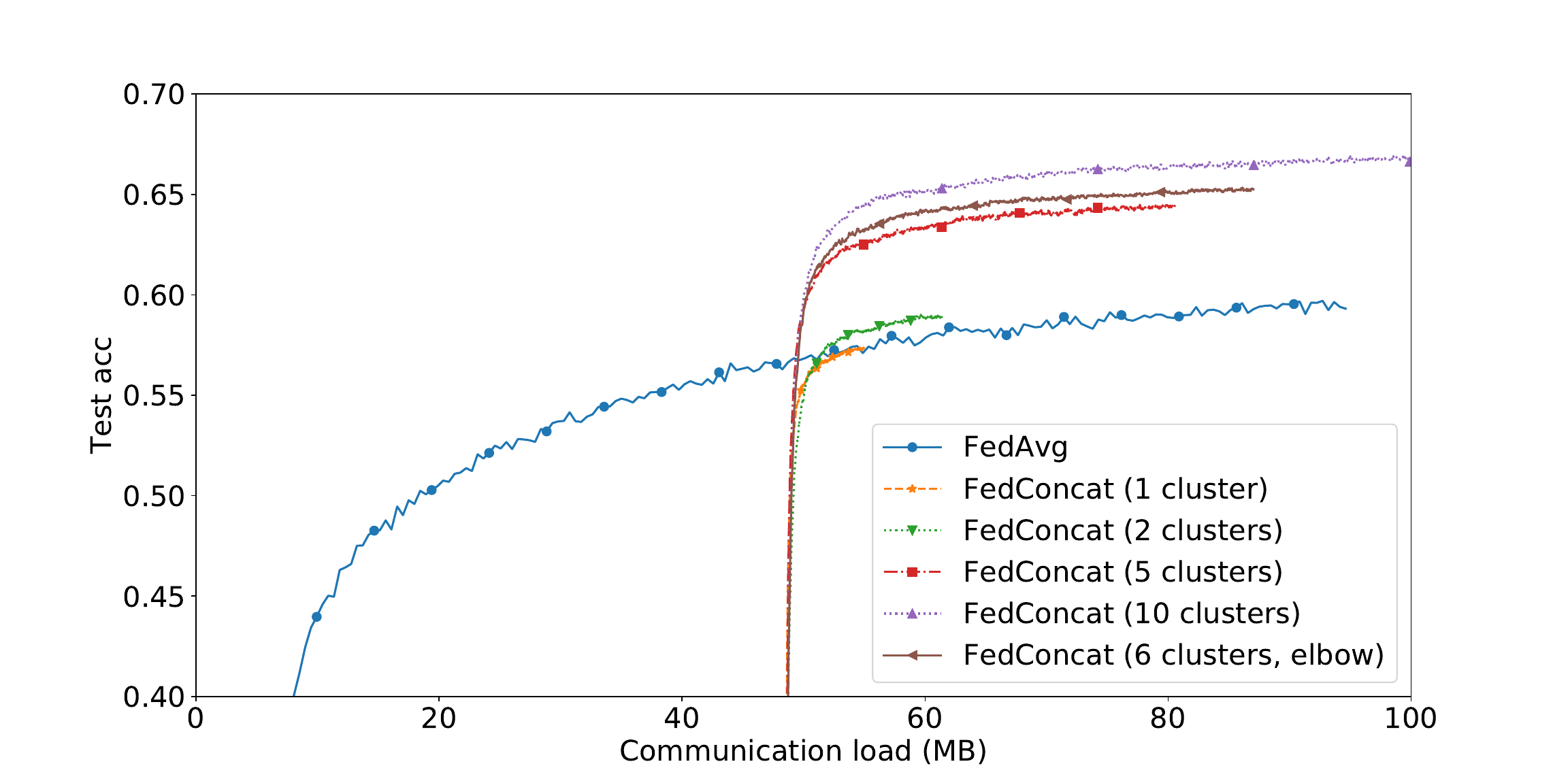}}
    \subfloat[CIFAR-10,$p_k \sim Dir(0.1)$]{\includegraphics[width=0.33\textwidth]{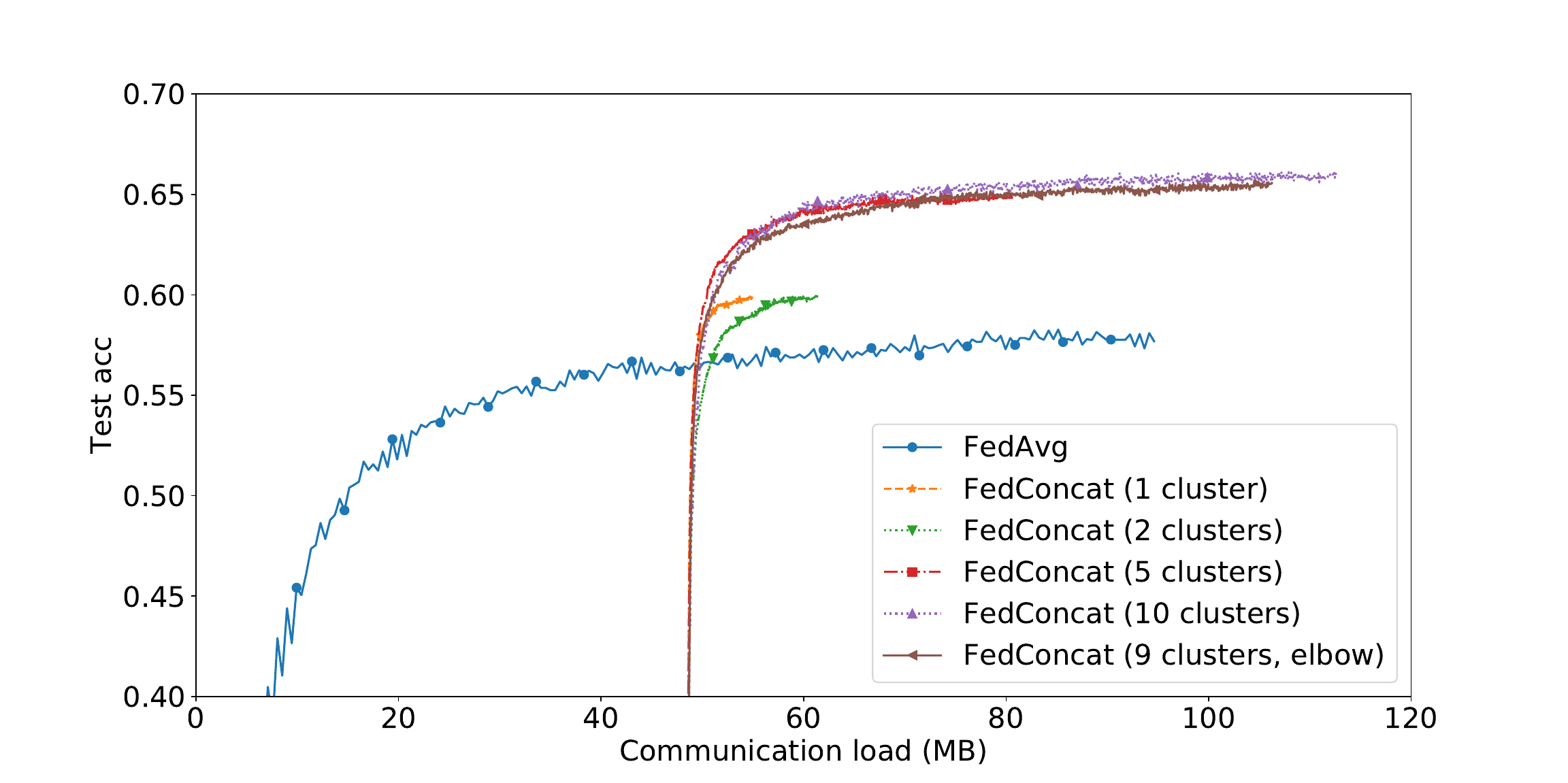}}
    \subfloat[CIFAR-10,$p_k \sim Dir(0.5)$]{\includegraphics[width=0.33\textwidth]{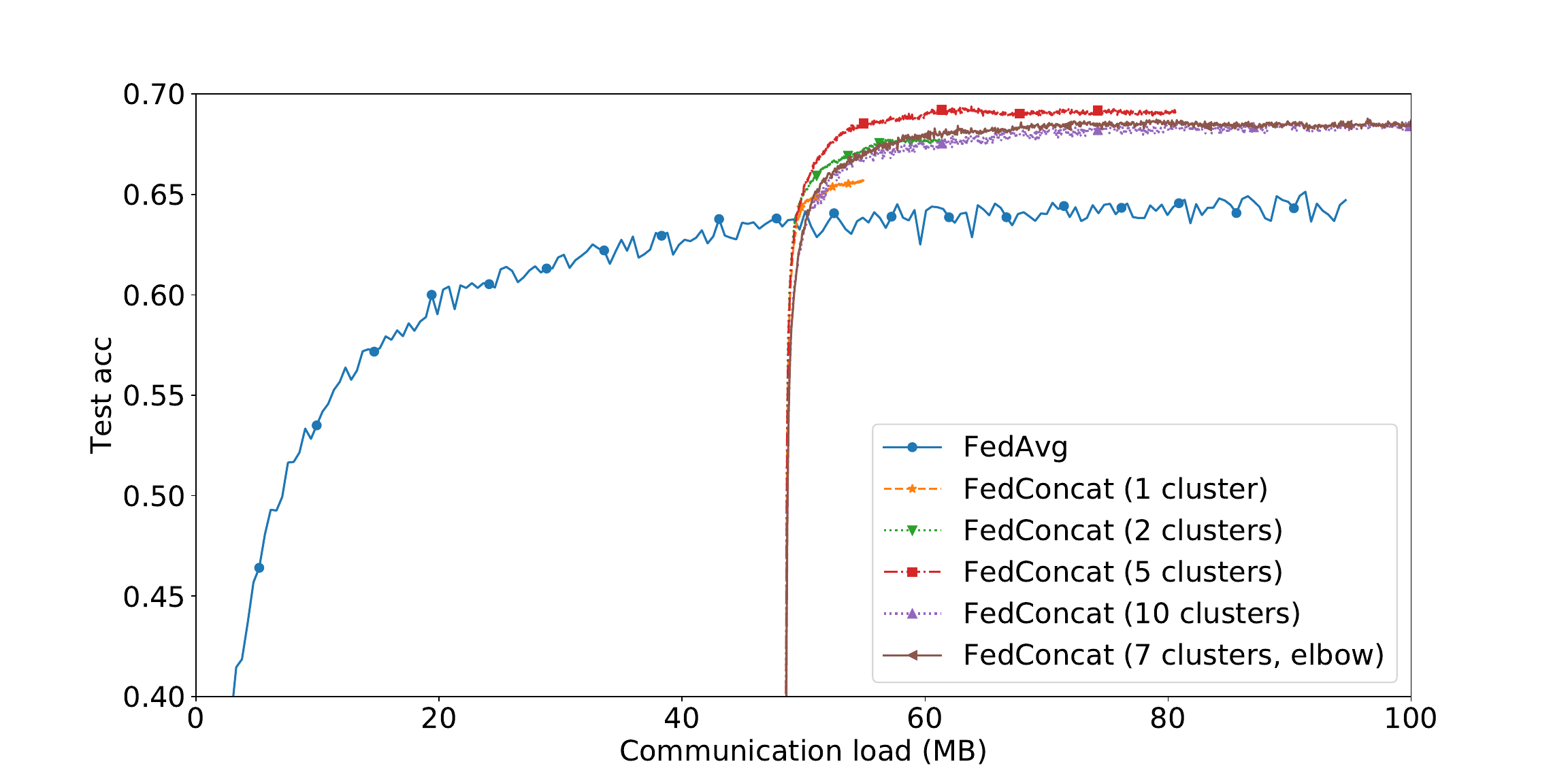}}
    \hfill
    \subfloat[FMNIST, \#C=2]{\includegraphics[width=0.33\textwidth]{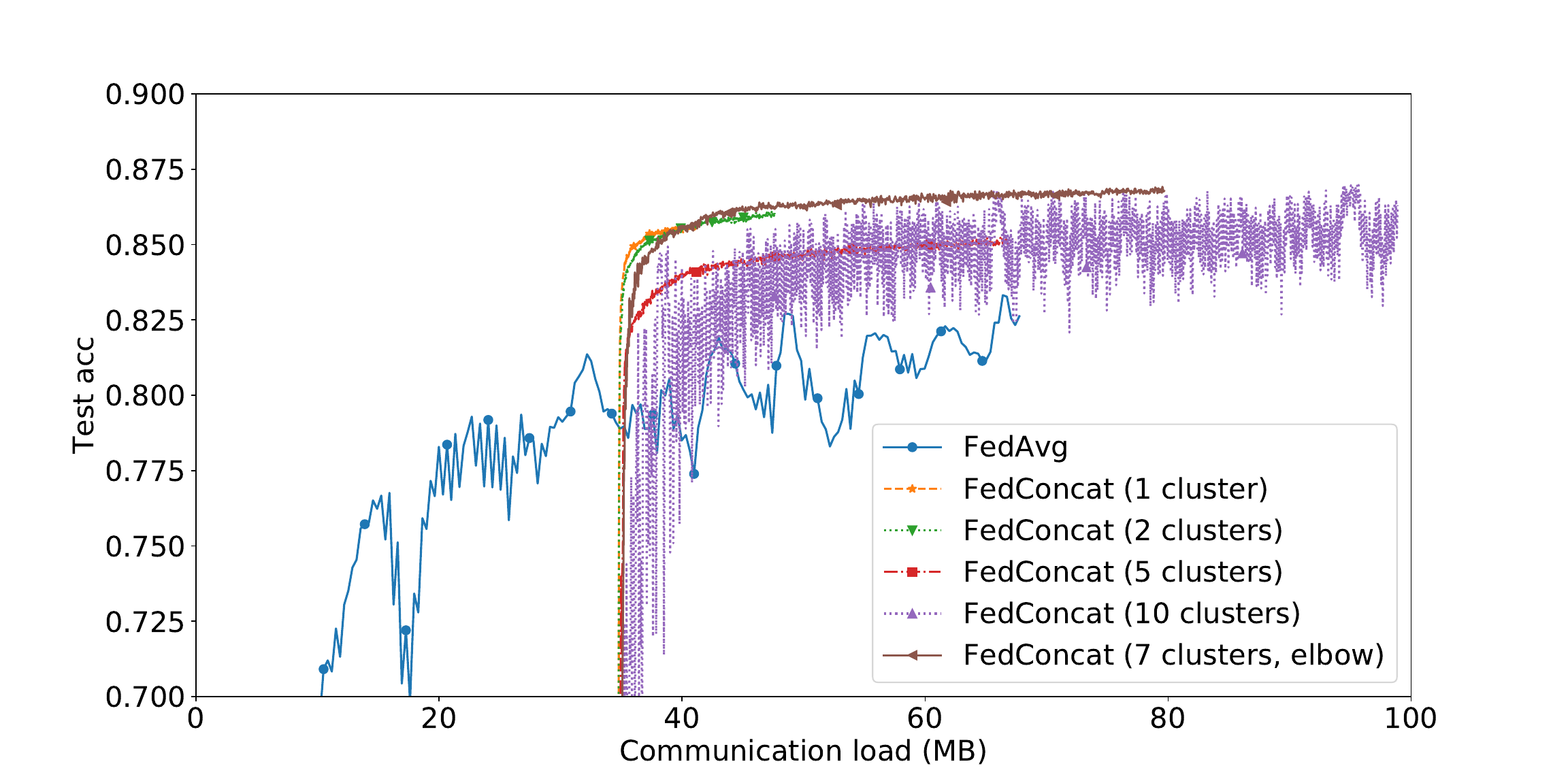}}
    \subfloat[FMNIST, $p_k \sim Dir(0.1)$]{\includegraphics[width=0.33\textwidth]{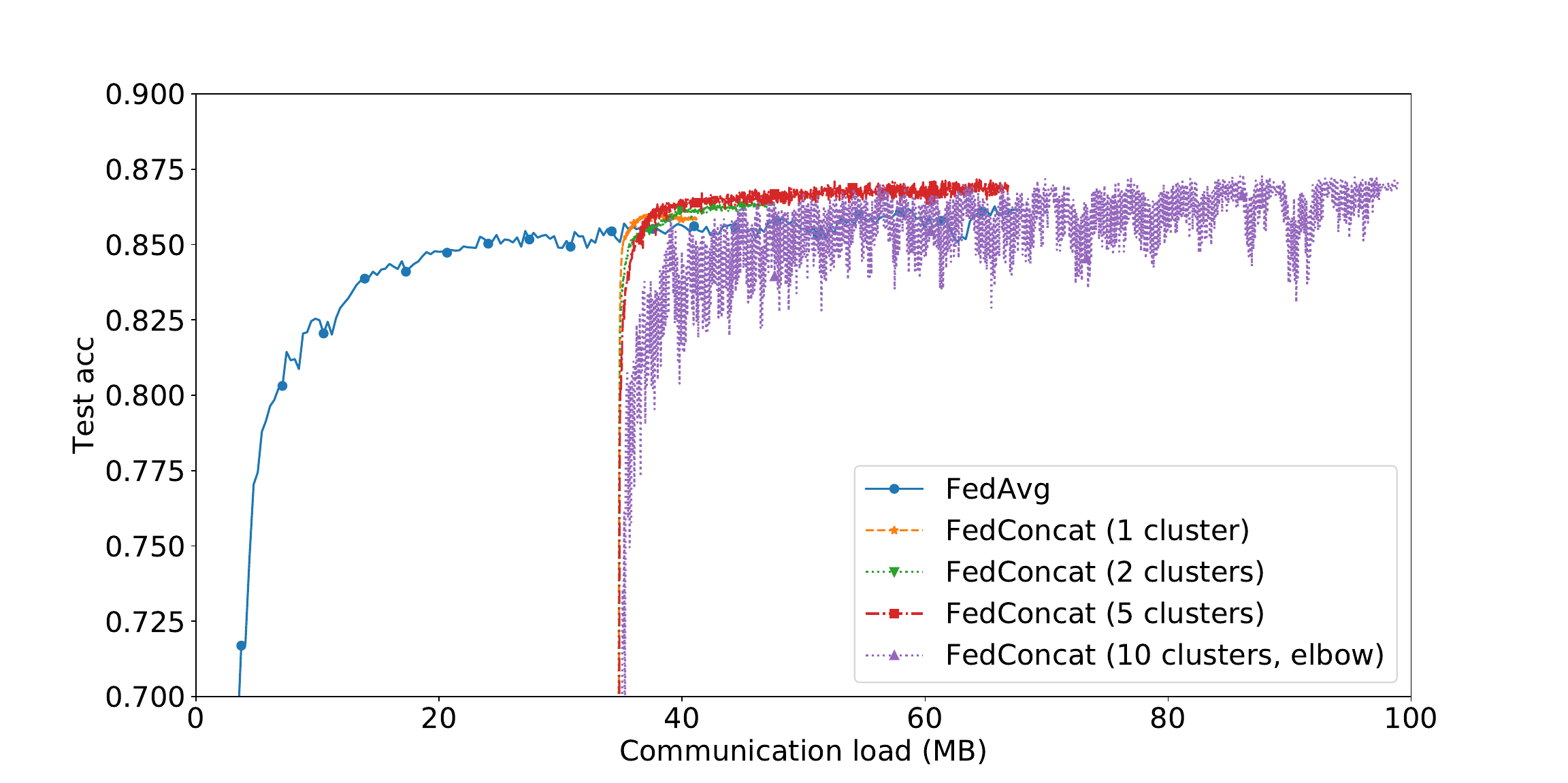}}
    \subfloat[FMNIST, $p_k \sim Dir(0.5)$]{\includegraphics[width=0.33\textwidth]{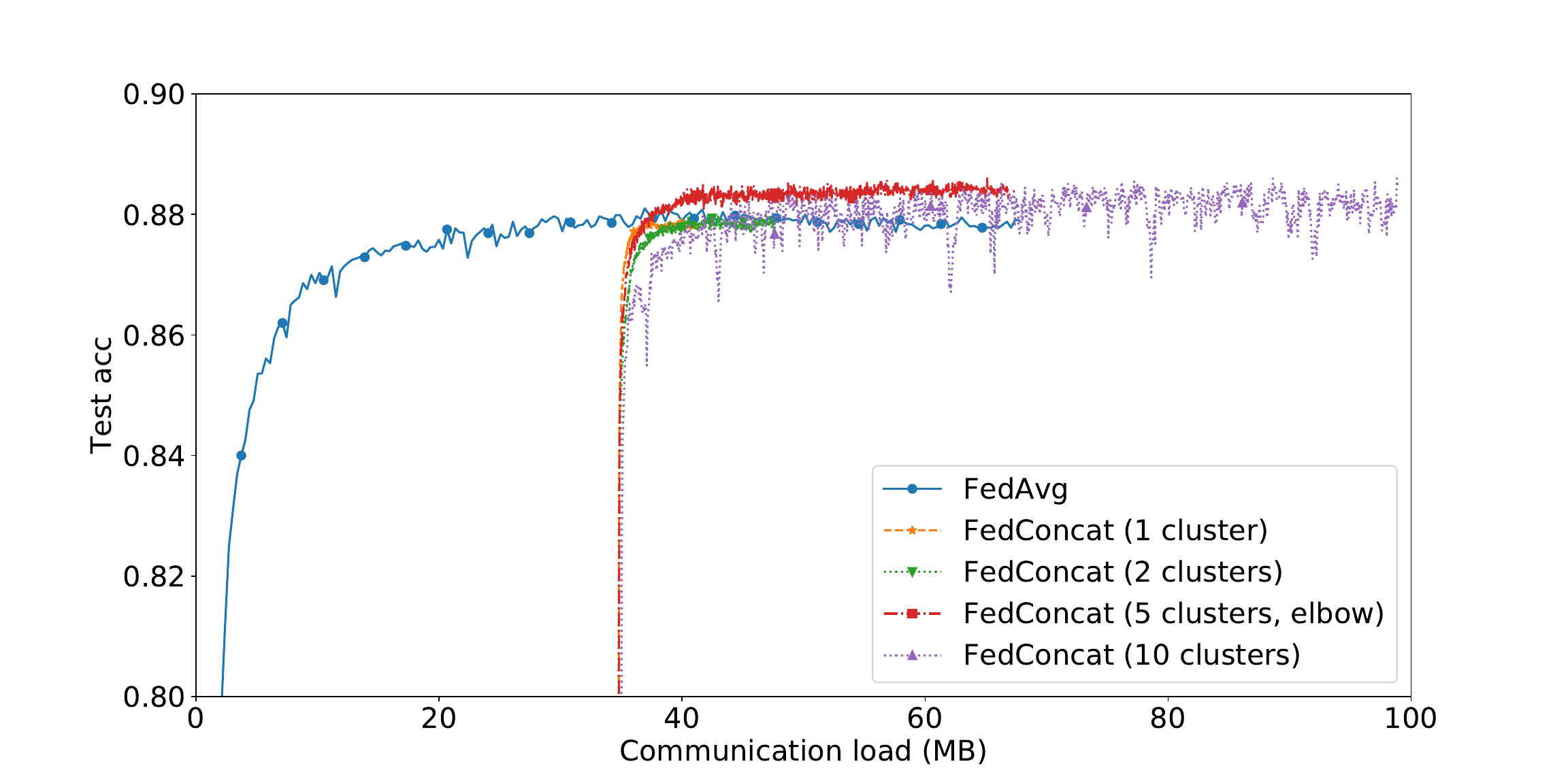}}
    \caption{The training curves of different cluster number on CIFAR-10 and FMNIST. Here elbow means such $K$ is estimated to be about the elbow value. Note that all curves in the figure have actually converged. Some curves, especially those of one cluster, seem not converging because the range of its communication cost is relatively small in the axis. FedConcat with one cluster is different from FedAvg in the post-training stage, where FedConcat only trains the last layer.}
    \label{fig:cluster}
\end{figure*}

\subsection{More Experiments on Clustered FL}
\label{sec:exp_cluster_FL}
{This section supplements Section \ref{sec:main_cfl} in the main paper. We first briefly introduce the compared clustered FL algorithms. FeSEM~\citep{long2023multi} employs the EM algorithm to flexibly form clusters based on the distance between local models and the respective cluster centers. FedSoft \citep{ruan2022fedsoft} proposes soft clustering approach where each client is part of a mixture of clusters rather than being confined to just one. IFCA \citep{Ghosh2020AnEF} enables each client to develop multiple local models and dynamically fine-tune their cluster affiliation, always opting for the model with the least loss. }

{Results on SVHN and FMNIST datasets are shown in Table \ref{tbl:cfl_more}. Combined with Table \ref{tbl:cfl} in main paper, we can see that FedConcat outperforms other clustering strategies in most cases. {FeSEM}~\citep{long2023multi} is a strong baseline and can be applied to our clustering stage. Even in cases where FedConcat cannot beat FedConcat-FeSEM, the difference is no more than 1.5\%. However, considering the extra computation cost of FeSEM and extra communication cost of IFCA \citep{Ghosh2020AnEF} and FedSoft \citep{ruan2022fedsoft}, we can conclude that our proposed clustering strategies are effective and efficient.}

\begin{table*}[ht]
\centering
\caption{Additional results of FedConcat and FedConcat-ID compared with other clustering strategies. These clustering strategies are adapted in FedConcat framework, therefore we denote them as FedConcat-IFCA, FedConcat-FedSoft and FedConcat-FeSEM. }
\label{tbl:cfl_more}
\resizebox{2.1\columnwidth}{!}{
\begin{tabular}{|c|c|c|c|c||c|c|}
\hline
Dataset & Partition & FedConcat-IFCA & FedConcat-FedSoft & FedConcat-FeSEM & FedConcat & FedConcat-ID \\ \hline
\multirow{4}{*}{SVHN} & $\#C=2$ & 83.7\%$\pm$0.6\%  & 82.0\%$\pm$1.4\%  & \textbf{84.6\%$\pm$0.4\%} & 83.4\%$\pm$1.4\% &  83.2\%$\pm$1.9\%  \\ \cline{2-7}
& $\#C=3$ & 85.7\%$\pm$0.6\%  & 84.5\%$\pm$0.2\%  & \textbf{86.4\%$\pm$0.4\%} & 86.0\%$\pm$0.9\% & 86.1\%$\pm$0.5\% \\ \cline{2-7}
& $p_k \sim Dir(0.1)$ & \textbf{83.9\%$\pm$0.6\%}  & 83.0\%$\pm$0.8\%  & \textbf{83.9\%$\pm$1.4\%} & 83.2\%$\pm$0.9\%& 82.9\%$\pm$0.3\%  \\ \cline{2-7}
& $p_k \sim Dir(0.5)$ &  86.7\%$\pm$2.0\% & 86.9\%$\pm$0.7\%  & 87.7\%$\pm$0.3\% & 87.5\%$\pm$0.1\% & \textbf{87.9\%$\pm$0.3\%} \\ \hline
\multirow{4}{*}{FMNIST} & $\#C=2$  & 79.4\%$\pm$8.3\% & 81.9\%$\pm$1.9\% & 84.2\%$\pm$1.6\%& \textbf{84.4\%$\pm$0.6\%} & 83.0\%$\pm$2.0\% \\ \cline{2-7}
& $\#C=3$  & 86.1\%$\pm$0.6\% & 85.8\%$\pm$0.2\% & 86.6\%$\pm$0.1\%& \textbf{87.1\%$\pm$0.2\%}& \textbf{86.6\%$\pm$0.1\%} \\ \cline{2-7}
& $p_k \sim Dir(0.1)$ & 84.6\%$\pm$0.5\% & 85.0\%$\pm$0.6\% & \textbf{85.4\%$\pm$0.5\%}& 84.5\%$\pm$0.1\%& 85.0\%$\pm$0.4\%  \\ \cline{2-7}
& $p_k \sim Dir(0.5)$ & 87.4\%$\pm$0.2\% & 87.2\%$\pm$0.2\% & 87.6\%$\pm$0.2\%& \textbf{87.7\%$\pm$0.1\%} & 87.5\%$\pm$0.2\%  \\ \hline
\end{tabular}
 }
\end{table*}

\subsection{Experiments on Partial Client Participation}
\label{sec:exp_partial}
{In this section, we experiment with client partial participation. We adopt the similar experimental setting as in Section 4.2. In each round, we randomly choose 50\% clients to participate in the training. To maintain the same communication cost, we double the number of rounds for all algorithms. For FedConcat and FedConcat-ID, the partial participation happens after collecting the label distribution (or inferred label distribution) and clustering. The model size of baseline algorithms is the same as the model of one cluster in FedConcat. Results are shown in Table \ref{tbl:sample}. We can conclude that FedConcat and FedConcat-ID are robust to partial client participation and still better than baselines in most cases.}

\begin{table*}[ht]
\centering
\caption{Experimental results of FedConcat and FedConcat-ID compared with FedAvg, FedProx, MOON, FedRS and FedLC with same communication cost under 50\% client participation each round. }
\label{tbl:sample}
\resizebox{2.1\columnwidth}{!}{
\begin{tabular}{|c|c|c|c|c|c|c||c|c|}
\hline
Dataset & Partition & FedAvg & FedProx & MOON & FedRS & FedLC & FedConcat & FedConcat-ID \\ \hline
\multirow{4}{*}{CIFAR-10} & $\#C=2$ & 49.2\%$\pm$3.5\% & 48.1\%$\pm$3.1\% & 48.6\%$\pm$3.9\% & 57.0\%$\pm$1.3\% & 49.1\%$\pm$4.1\% & \textbf{57.0\%$\pm$1.9\%} & \textbf{57.2\%$\pm$2.2\%} \\ \cline{2-9}
& $\#C=3$ & 58.2\%$\pm$1.9\% & 57.9\%$\pm$2.0\% & 58.3\%$\pm$1.6\% & 62.4\%$\pm$1.3\% & 58.8\%$\pm$1.2\% & \textbf{63.0\%$\pm$1.7\%}& \textbf{64.6\%$\pm$0.5\%} \\ \cline{2-9}
& $p_k \sim Dir(0.1)$ & 49.8\%$\pm$4.2\% & 49.4\%$\pm$4.0\% & 50.2\%$\pm$4.1\% & 52.3\%$\pm$1.1\% & 51.1\%$\pm$0.8\% & \textbf{54.4\%$\pm$1.5\%}& \textbf{57.2\%$\pm$1.1\%} \\ \cline{2-9}
& $p_k \sim Dir(0.5)$ & 62.1\%$\pm$0.3\% & 62.3\%$\pm$0.4\% & 62.0\%$\pm$0.7\% & 62.6\%$\pm$1.1\% & 60.2\%$\pm$0.5\% & \textbf{66.6\%$\pm$0.4\%}& \textbf{66.5\%$\pm$0.1\%} \\ \hline
\multirow{4}{*}{SVHN} & $\#C=2$ & \textbf{82.2\%$\pm$1.1\%} & 82.0\%$\pm$1.3\% & 81.9\%$\pm$1.3\% & 80.3\%$\pm$1.5\% & 82.0\%$\pm$1.9\% & 81.7\%$\pm$1.9\% & 80.0\%$\pm$2.4\% \\ \cline{2-9}
& $\#C=3$ & 83.1\%$\pm$2.6\% & 82.9\%$\pm$3.1\% & 83.0\%$\pm$2.5\% & \textbf{85.9\%$\pm$0.5\%} & 82.7\%$\pm$2.6\% & 84.2\%$\pm$1.6\% & 83.7\%$\pm$3.3\% \\ \cline{2-9}
& $p_k \sim Dir(0.1)$ & 81.8\%$\pm$3.2\% & 82.4\%$\pm$2.3\% & 81.4\%$\pm$3.0\% & 75.5\%$\pm$1.8\% & 80.2\%$\pm$3.5\% & \textbf{82.6\%$\pm$0.5\%} & 82.0\%$\pm$2.1\%  \\ \cline{2-9}
& $p_k \sim Dir(0.5)$ & 87.0\%$\pm$0.5\% & 86.9\%$\pm$0.6\% & 87.0\%$\pm$0.3\% & 87.4\%$\pm$0.2\% & 86.9\%$\pm$0.4\% & \textbf{88.1\%$\pm$0.2\%}& \textbf{88.3\%$\pm$0.2\%} \\ \hline
\multirow{4}{*}{FMNIST} & $\#C=2$ & 81.5\%$\pm$1.5\% & 81.7\%$\pm$1.5\% & 81.4\%$\pm$0.6\% & 75.0\%$\pm$6.5\% & 72.3\%$\pm$4.1\% & \textbf{84.4\%$\pm$1.9\%}& \textbf{84.5\%$\pm$1.7\%} \\ \cline{2-9}
& $\#C=3$ & 84.3\%$\pm$1.6\% & 84.2\%$\pm$1.1\% & 84.2\%$\pm$1.5\% & 86.2\%$\pm$0.4\% & 85.6\%$\pm$0.9\% & \textbf{86.8\%$\pm$0.1\%}& \textbf{86.9\%$\pm$0.3\%} \\ \cline{2-9}
& $p_k \sim Dir(0.1)$ & 84.4\%$\pm$1.1\% & 84.1\%$\pm$1.3\% & 84.5\%$\pm$1.4\% & 81.9\%$\pm$0.6\% & 81.1\%$\pm$1.0\% & \textbf{84.8\%$\pm$1.6\%} & \textbf{85.6\%$\pm$0.3\%} \\ \cline{2-9}
& $p_k \sim Dir(0.5)$ & 87.4\%$\pm$0.5\% & 87.2\%$\pm$0.5\% & 87.4\%$\pm$0.5\% & 87.3\%$\pm$0.3\% & 87.6\%$\pm$0.3\% & \textbf{87.9\%$\pm$0.3\%}& \textbf{88.0\%$\pm$0.3\%} \\ \hline
\end{tabular}
 }
\end{table*}

\subsection{More Experiments of Larger Models}
\label{sec:exp_resnet}

\paragraph{Tackling overfitting}



{We observe that the cluster models of FedConcat overfit severely in $p_k \sim Dir(0.1)$ and $p_k \sim Dir(0.5)$, when the weight decay is kept as $10^{-5}$. Since the volume of data each cluster receives in FL settings is small, large models can easily overfit the dataset of each cluster. Comparatively, directly training FedAvg on all clients mitigates the overfitting effect since the quantity of data is higher. The overfitting issue leads to the poor performance of FedConcat.}

{In order to counteract overfitting, we increase the weight decay factor. Our experiments indicate that a weight decay factor among $\{0.001,0.002,0.005\}$ significantly enhances the test accuracy of most algorithms. However, an excessive increase (e.g. $0.01$) leads to a decrease in accuracy. }

\paragraph{Balancing the size of each cluster}
{Apart from overfitting, we notice that the clustering process becomes increasingly imbalanced as the number of classes rises. As the label distribution expands over higher dimensional spaces, each client's label distribution becomes sparser and more isolated, thus posing a challenge to clustering. Clusters with minimal client count tend to be more prone to overfitting. To rectify the issue of unbalanced clusters, we modify the K-means algorithm to manually relocate distant points from majority clusters to minority ones. In particular, if a cluster exceeds $1.2\times$ the average number of members, redistribution takes place until its size falls below this threshold. Results are shown in Table \ref{tbl:bal_cluster}. We can find out that the balanced clustering technique brings significant improvement in extreme non-IID cases (e.g. \#C=2 and \#C=3). In slight non-IID cases like $p_k \sim Dir(0.1)$ or $p_k \sim Dir(0.5)$, normal clustering strategy works better. }

\begin{table*}[ht]
\centering
\caption{Experimental results of applying balanced clustering on CIFAR-100 with ResNet-50. }
\label{tbl:bal_cluster}
\begin{tabular}{|c|c|c||c|c|}
\hline
 Partition & FedConcat (naive) & FedConcat (bal) & FedConcat-ID (naive) & FedConcat-ID (bal) \\ \hline
 $\#C=2$ & 11.6\%$\pm$0.5\% & \textbf{15.9\%$\pm$0.8\%} & 12.4\%$\pm$0.9\% & \textbf{13.5\%$\pm$0.5\%} \\ \hline
$\#C=3$ & 30.0\%$\pm$0.8\% & \textbf{32.2\%$\pm$0.5\%} & 26.7\%$\pm$0.7\% & \textbf{31.0\%$\pm$0.6\%}  \\ \hline

$p_k \sim Dir(0.1)$ & 58.9\%$\pm$0.5\% & \textbf{60.1\%$\pm$0.3\%} & \textbf{62.1\%$\pm$0.4\%} & 57.9\%$\pm$0.4\%  \\ \hline
$p_k \sim Dir(0.5)$ & \textbf{66.4\%$\pm$0.1\%} & 66.3\%$\pm$0.2\% & \textbf{68.1\%$\pm$0.2\%} & 66.2\%$\pm$0.1\%  \\ \hline



\end{tabular}
\end{table*}

\paragraph{Initializing with parameters of cluster classifiers}
{Large models like ResNet-50 take much computation power to train. In the post-training stage of FedConcat framework, we can initialize the large linear classifier with parameters of classifiers from each cluster. The weights are concatenated and the bias are summed up. Such initialization is equal to an ensemble by summing up the logits of these well-trained cluster models, which can be a good start point to train the last classification layer. }

{We show the training curves of FedConcat with balanced clusters on CIFAR-100 in Figure \ref{fig:init-classifier}. As we can see, by initializing with parameters of cluster classifiers in the post-training stage, we can achieve higher accuracy with faster convergence, compared with random initialization.}

\begin{figure*}[h]
    \centering
    \subfloat[\#C=2]{\includegraphics[width=0.25\textwidth]{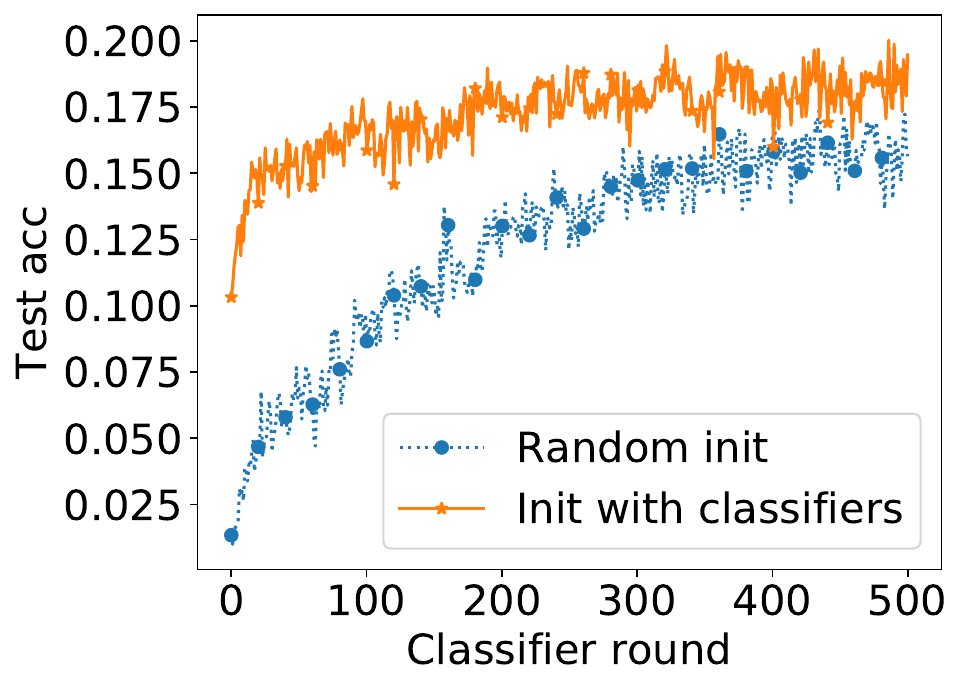}}
    \subfloat[\#C=3]{\includegraphics[width=0.25\textwidth]{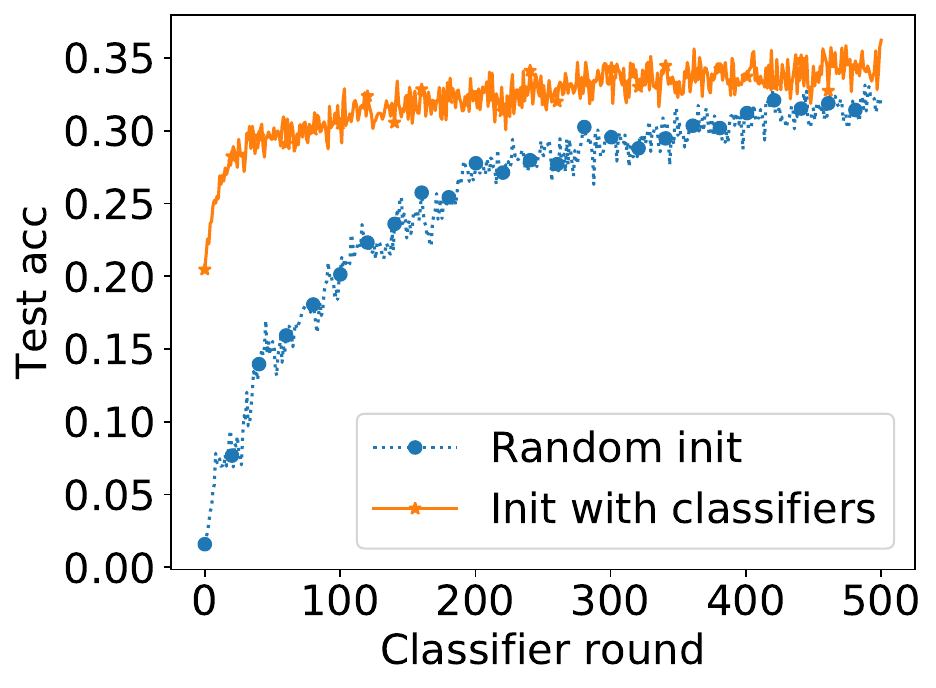}}
    \subfloat[$p_k \sim Dir(0.1)$]{\includegraphics[width=0.25\textwidth]{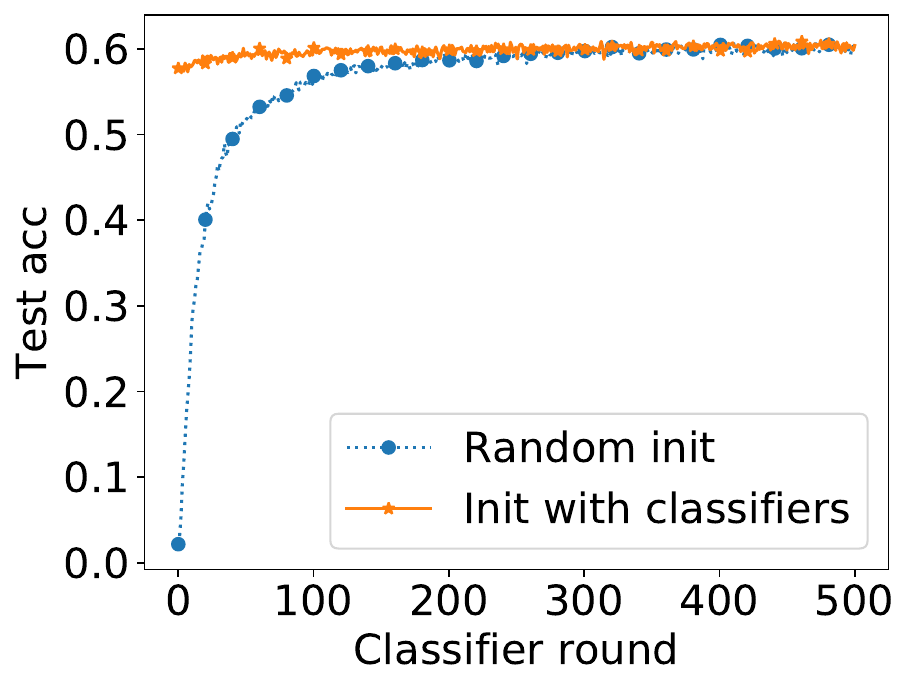}}
    \subfloat[$p_k \sim Dir(0.5)$]{\includegraphics[width=0.25\textwidth]{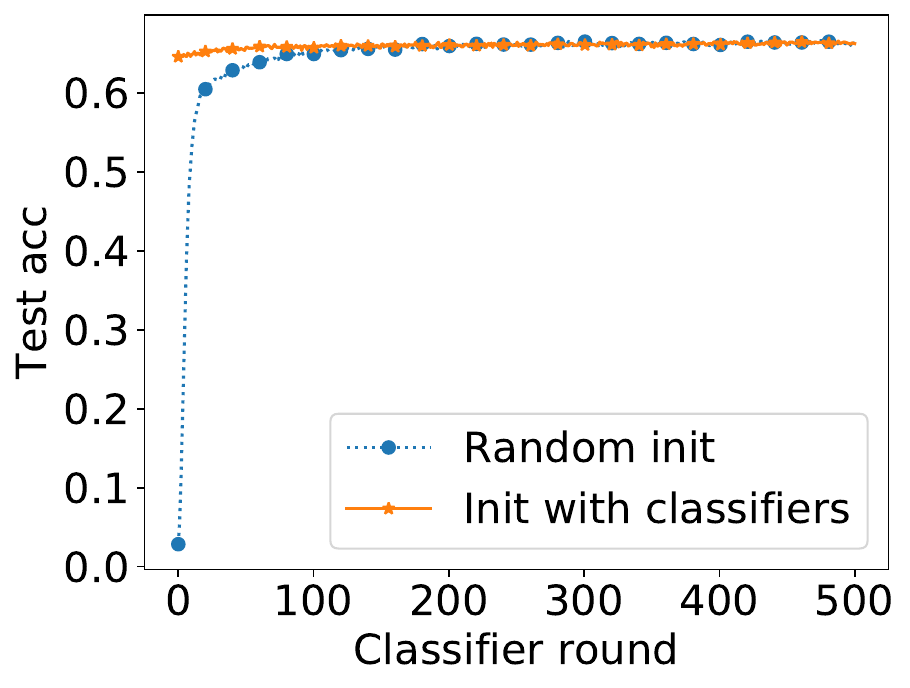}}
    
    \caption{Training curves of random initialization and initializing with cluster classifiers in the post-training stage. We run FedConcat with balanced clusters on CIFAR-100 dataset.}
    \label{fig:init-classifier}
\end{figure*}

\paragraph{VGG-9 model experiments}
We also compare FedConcat and FedConcat-ID with the other baselines with VGG-9 on CIFAR-10. The results are shown in Figure \ref{fig:heavy}.
Similarly, our FedConcat still has significant advantage on VGG-9 model for the extreme non-IID $\#C=2$ case. It can achieve higher accuracy with smooth convergence. FedConcat-ID achieves slightly lower accuracy than FedConcat, but can still outperform five baseline algorithms. 

\begin{figure}[!]
    \centering
    \includegraphics[width=1\columnwidth]{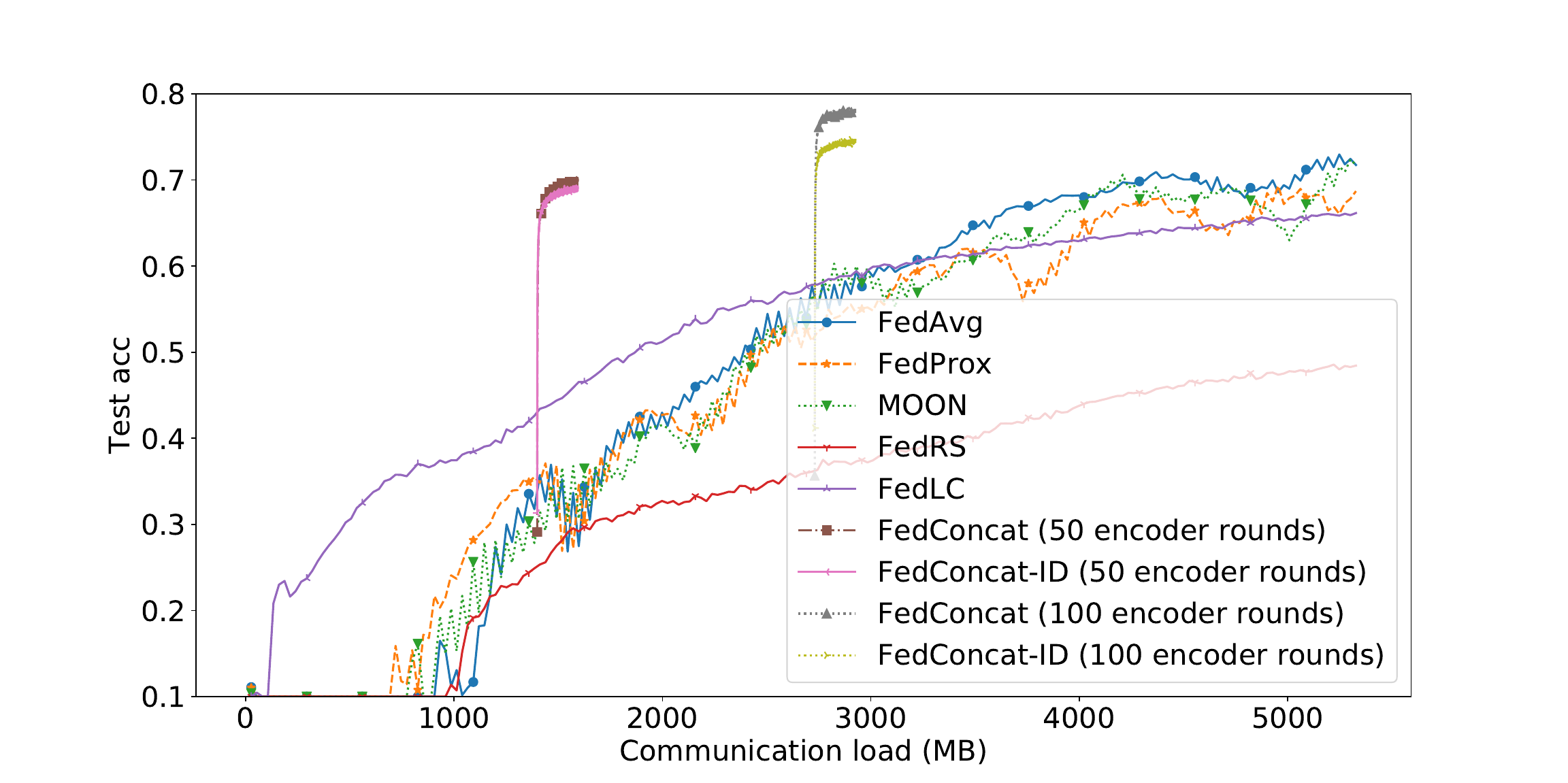}
    \caption{The training curves on CIFAR-10 dataset with VGG-9. The data partition is $\#C=2$.}
    \label{fig:heavy}
\end{figure}

\subsection{More Experiments on Scalability}
\label{sec:more_scalab}
{In this section, we present scalability experiments for the SVHN and FMNIST datasets in Table \ref{tbl:scala}. Using the same settings as those in Section \ref{sec:scalab}, i.e., with a 50\% participation rate in each round, FedConcat consistently outperforms the baselines in both datasets. }

\begin{table*}[h]
\centering
\caption{Scalability of FedConcat and FedConcat-ID compared with baselines on SVHN and FMNIST. }
\label{tbl:scala}
\resizebox{2.1\columnwidth}{!}{
\begin{tabular}{|c|c|c|c|c|c|c|c||c|c|}
\hline
\#Clients & Dataset & Partition & FedAvg & FedProx & MOON & FedRS & FedLC & FedConcat & FedConcat-ID  \\ \hline
\multirow{8}{*}{100}& \multirow{4}{*}{SVHN} & $\#C=2$ & 76.4\% & 77.9\% & 77.2\% & 82.3\% & 81.6\% & \textbf{85.2\%} & \textbf{85.2\%} \\ \cline{3-10}
& & $\#C=3$ & 84.7\% & 84.7\% & 84.4\% & 85.9\% & 84.4\% & \textbf{86.1\%} & \textbf{86.9\%} \\ \cline{3-10}
& & $p_k \sim Dir(0.1)$ & 81.6\% & 82.5\% & 82.4\% & 80.2\% & 79.1\% & \textbf{84.5\%} & \textbf{83.8\%} \\ \cline{3-10}
& & $p_k \sim Dir(0.5)$ & 86.7\% & 86.6\% & 86.6\% & 86.8\% & 86.1\% & \textbf{87.6\%}& \textbf{87.5\%} \\ \cline{2-10}
& \multirow{4}{*}{FMNIST} & $\#C=2$ & 80.5\% & 83.0\% & 82.2\% & 81.5\% & 80.4\% & \textbf{85.2\%}& \textbf{84.1\%} \\ \cline{3-10}
& & $\#C=3$ & 82.8\% & 83.2\% & 83.2\% & 86.0\% & 85.3\% & \textbf{86.4\%}& \textbf{86.7\%} \\ \cline{3-10}
& & $p_k \sim Dir(0.1)$ & 84.2\% & 84.7\% & 84.2\% & 83.6\% & 83.4\% & \textbf{84.8\%} & \textbf{85.6\%} \\ \cline{3-10}
& & $p_k \sim Dir(0.5)$ & 86.8\% & 86.7\% & 86.6\% & 86.5\% & 86.4\% & \textbf{87.2\%}& \textbf{87.3\%} \\ \hline

\multirow{8}{*}{200}& \multirow{4}{*}{SVHN} & $\#C=2$& 80.4\% & 82.7\% & \textbf{82.8\%} & 80.7\% & 77.6\% & 77.9\% & 78.4\%  \\ \cline{3-10}
& & $\#C=3$ & 83.6\% & 83.9\% & 84.0\% & 84.2\% & 82.5\% & \textbf{84.3\%} & \textbf{84.6\%} \\ \cline{3-10}
& & $p_k \sim Dir(0.1)$& 80.9\% & 80.6\% & \textbf{81.0\%} & 80.7\% & 79.3\% & 80.2\% & 79.4\%  \\ \cline{3-10}
& & $p_k \sim Dir(0.5)$ & 85.6\% & 85.8\% & 85.7\% & 85.9\% & 85.9\% & \textbf{86.6\%}& \textbf{86.2\%}  \\ \cline{2-10}
& \multirow{4}{*}{FMNIST} & $\#C=2$ & 79.7\% & 77.8\% & 78.4\% & 78.5\% & 74.7\% & \textbf{81.8\%}& \textbf{81.8\%} \\ \cline{3-10}
& & $\#C=3$ & 82.5\% & 82.7\% & 82.2\% & 83.1\% & 82.4\% & \textbf{85.7\%}& \textbf{84.9\%} \\ \cline{3-10}
& & $p_k \sim Dir(0.1)$ & 80.3\% & 80.2\% & 80.4\% & 80.3\% & 78.6\% & \textbf{83.3\%} & \textbf{83.1\%} \\ \cline{3-10}
& & $p_k \sim Dir(0.5)$ & 85.5\% & 85.6\% & 85.5\% & 84.9\% & 84.7\% & \textbf{86.9\%}& \textbf{86.1\%} \\ \hline
\end{tabular}
 }
\end{table*}

\subsection{Further Experiments on FedConcat-ID}
\label{sec:more_exp_concat_id}

\paragraph{Effectiveness of label inference}
{In addition to our approach, there exist other techniques to infer label distribution from the client model \cite{wang2021addressing, ramakrishna2022inferring}. However, most of them necessitate an auxiliary global dataset to facilitate the inference, which is not applicable in our experimental setting. The initialized bias (IB) method in \citet{ramakrishna2022inferring} deduces label distributions based on the evolution of classifier bias without requiring an auxiliary dataset. Thus, we compare the label inference in FedConcat-ID with the IB method. Although IB-inferred label distributions can be aggregated across various rounds, we utilize its first-round inferred label distributions for clustering, to maintain consistency with our experimental setting.} 

{To validate the effectiveness of label inference in FedConcat-ID, we visualize two examples of the true label distributions, FedConcat-ID inferred label distributions, and IB inferred label distributions in Figure \ref{fig:infer}. We partition the CIFAR-10 dataset into 40 clients by setting \#C=2 and $p_k \sim Dir(0.5)$ respectively, and display the distributions of the first client. Our observation is that the distributions inferred by FedConcat-ID are closely aligned with the true distributions. While IB inferred distributions do indicate the general trend of the true distributions, there exists a significant gap between its inferred values and the ground truth in the first round. The assumptions of small weight products and stable bias in \citet{ramakrishna2022inferring} are less likely to hold true in real-world test scenarios. }

\begin{figure*}[h]
    \centering
    \subfloat[\#C=2]{\includegraphics[width=\columnwidth]{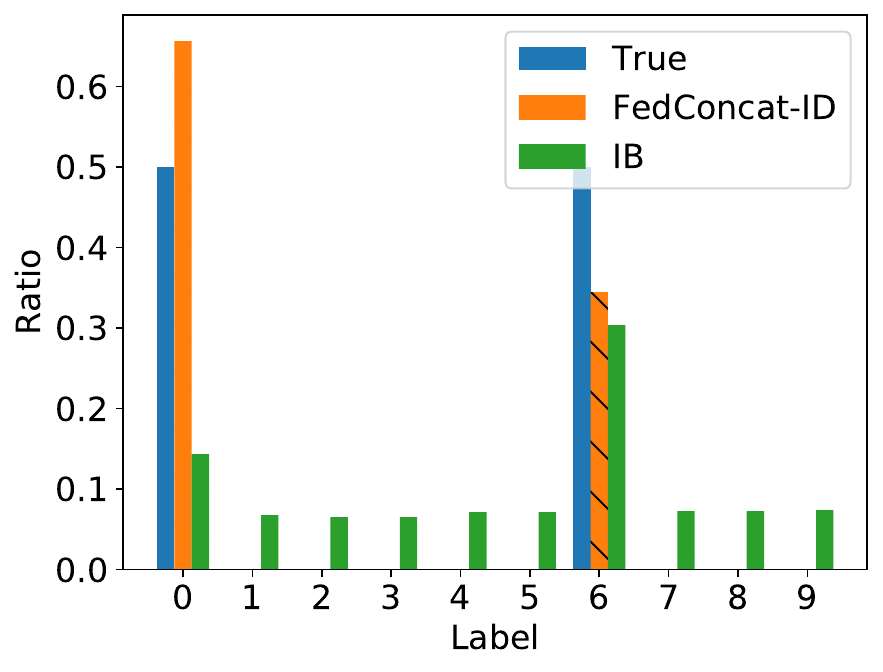}}
    \subfloat[$p_k \sim Dir(0.5)$]{\includegraphics[width=\columnwidth]{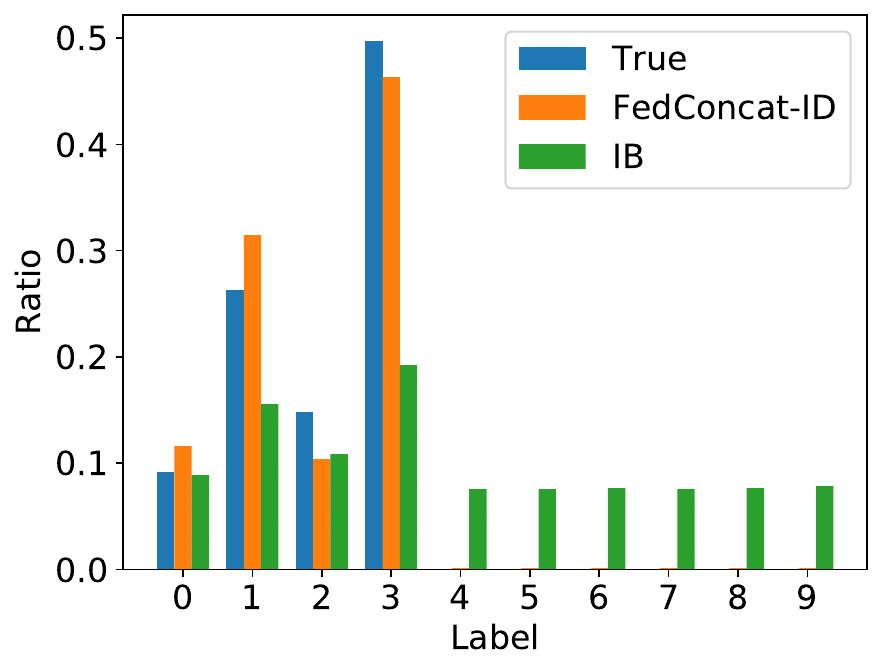}}
    \caption{Comparing FedConcat-ID inferred distributions with true distributions and IB inferred distributions on CIFAR-10.}
    \label{fig:infer}
\end{figure*}

{We also compare with the final model accuracy using clusters based on IB inferred distributions. The results are presented in Table \ref{tbl:ib}. We can see that both label inference methods yield similar final test accuracy. Despite the IB inferred distributions deviating from the true distributions, its predictions can still aid in forming clusters among clients with analogous distributions. However, it's worth noting that the IB inference method is only applicable to neural networks with a bias term in the last layer, while our FedConcat-ID inference method can accommodate neural networks even without a bias term.}

\begin{table}[h]
\centering
\caption{Experimental results of FedConcat-ID compared with IB label inference.}
\label{tbl:ib}
\resizebox{\columnwidth}{!}{
\begin{tabular}{|c|c|c|c|}
\hline
Dataset & Partition & FedConcat-ID & FedConcat-IB \\ \hline
\multirow{4}{*}{CIFAR-10} & $\#C=2$ & \textbf{56.5\%$\pm$2.6\%} & 56.4\%$\pm$1.9\% \\ \cline{2-4}
& $\#C=3$ & \textbf{61.8\%$\pm$0.8\%} & 61.5\%$\pm$1.2\% \\ \cline{2-4}
& $p_k \sim Dir(0.1)$ &  \textbf{56.9\%$\pm$1.4\%} & 55.6\%$\pm$1.0\% \\ \cline{2-4}
& $p_k \sim Dir(0.5)$ &  63.7\%$\pm$0.8\% & \textbf{64.2\%$\pm$0.2\%} \\ \hline
\multirow{4}{*}{SVHN} & $\#C=2$ &  83.2\%$\pm$1.9\% & \textbf{83.8\%$\pm$0.3\%}  \\ \cline{2-4}
& $\#C=3$ & \textbf{86.1\%$\pm$0.5\%} & 86.0\%$\pm$1.1\% \\ \cline{2-4}
& $p_k \sim Dir(0.1)$ & 82.9\%$\pm$0.3\% & \textbf{84.3\%$\pm$0.6\%} \\ \cline{2-4}
& $p_k \sim Dir(0.5)$ & \textbf{87.9\%$\pm$0.3\%} & \textbf{87.9\%$\pm$0.5\%}  \\ \hline
\multirow{4}{*}{FMNIST} & $\#C=2$ & 83.0\%$\pm$2.0\% & \textbf{84.0\%$\pm$1.6\%} \\ \cline{2-4}
& $\#C=3$ & 86.6\%$\pm$0.1\% & \textbf{86.8\%$\pm$0.1\%} \\ \cline{2-4}
& $p_k \sim Dir(0.1)$ & 85.0\%$\pm$0.4\% & \textbf{85.3\%$\pm$0.2\%} \\ \cline{2-4}
& $p_k \sim Dir(0.5)$ & \textbf{87.5\%$\pm$0.2\%} & 87.3\%$\pm$0.3\% \\ \hline
\end{tabular}
 }
\end{table}

\paragraph{Training curves of FedConcat-ID}
We show the training curves of FedConcat-ID compared with five baseline algorithms in Figure \ref{fig:comm_id}. As we can see, FedConcat-ID also outperforms all five baseline algorithms on simple CNN model. When baseline algorithms are trained on model with same size to concatenated final global model (the same setting as Figure \ref{fig:large} in main paper), as shown in Figure \ref{fig:large-id}, FedConcat-ID can also achieve higher accuracy with the same communication cost. 

\begin{figure*}[!]
    \centering
    \subfloat[CIFAR-10, \#C=2]{\includegraphics[width=0.33\textwidth]{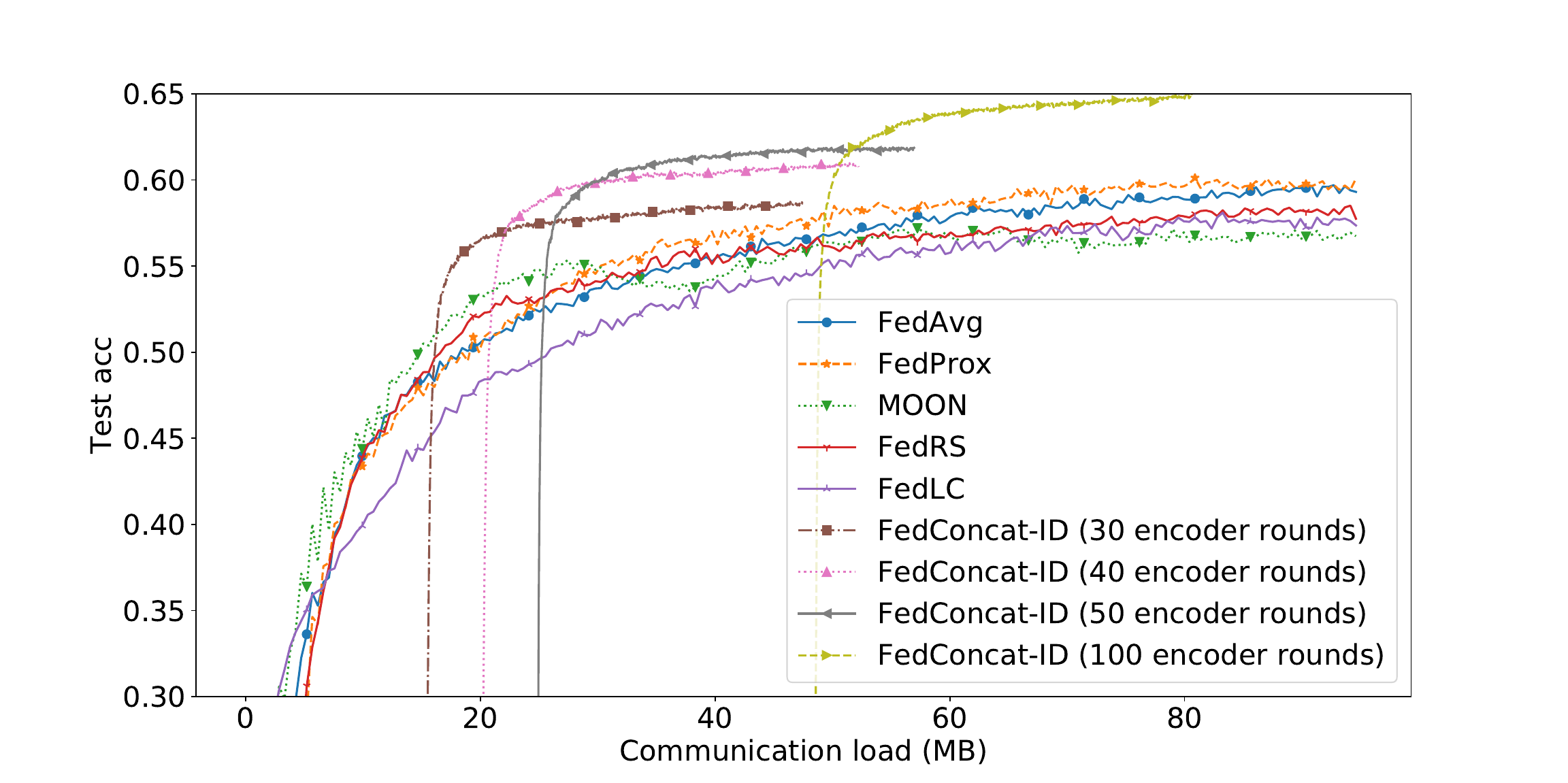}}
    \subfloat[SVHN, \#C=2]{\includegraphics[width=0.33\textwidth]{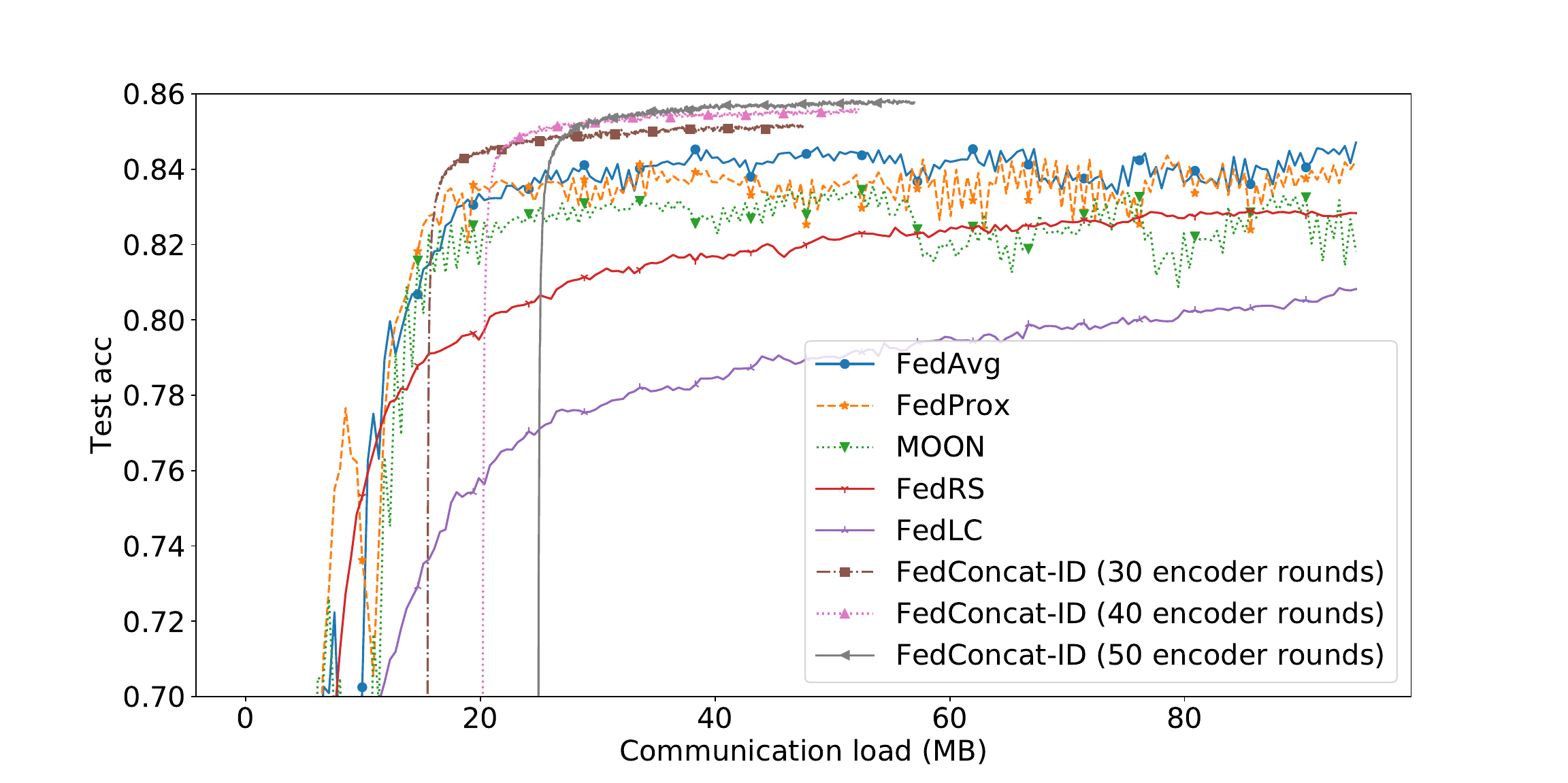}}
    \subfloat[FMNIST, \#C=2]{\includegraphics[width=0.33\textwidth]{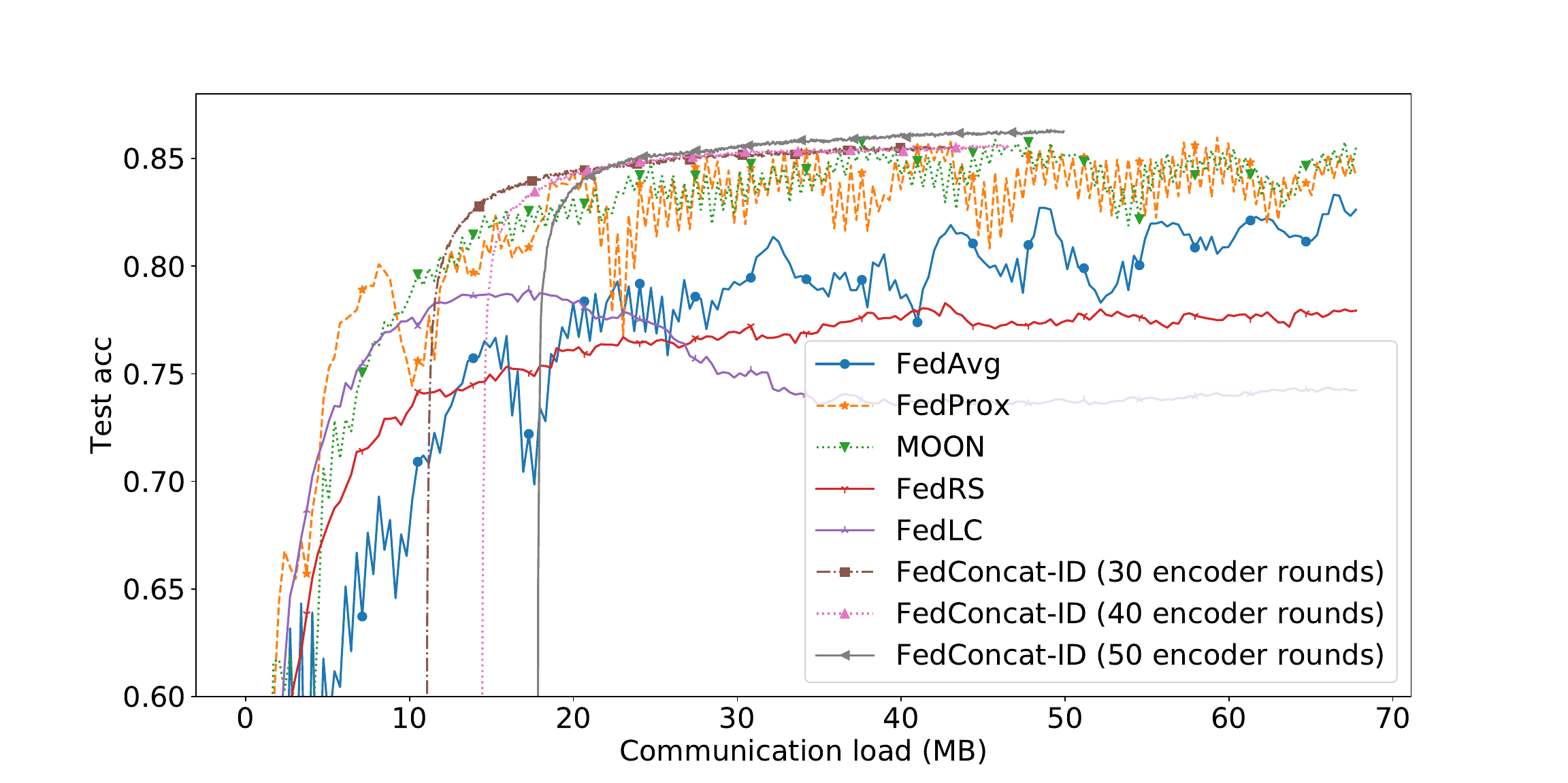}}
    \hfill
    \subfloat[CIFAR-10,$p_k \sim Dir(0.5)$]{\includegraphics[width=0.33\textwidth]{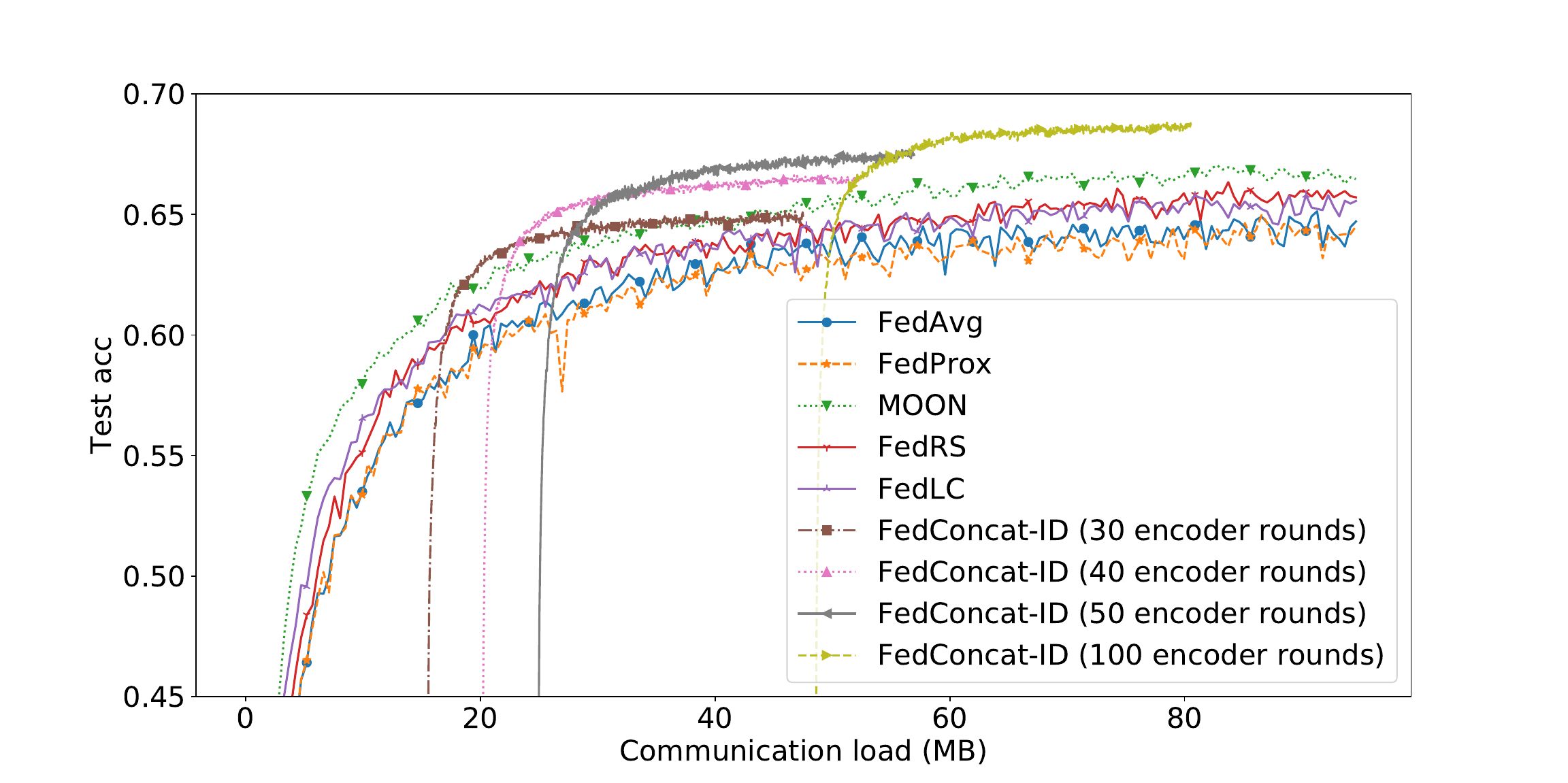}}
    \subfloat[SVHN, $p_k \sim Dir(0.5)$]{\includegraphics[width=0.33\textwidth]{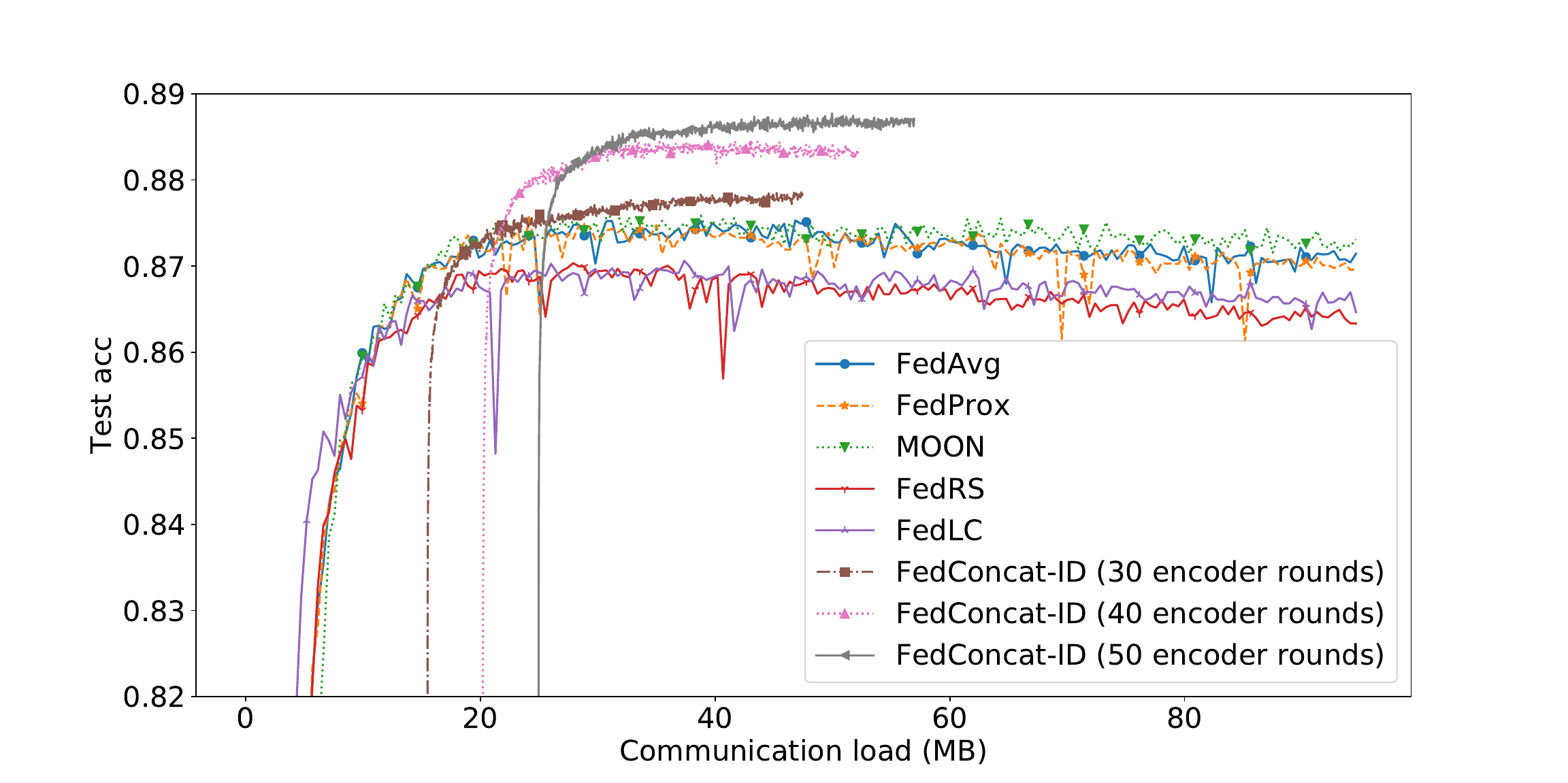}}
    \subfloat[FMNIST, $p_k \sim Dir(0.5)$]{\includegraphics[width=0.33\textwidth]{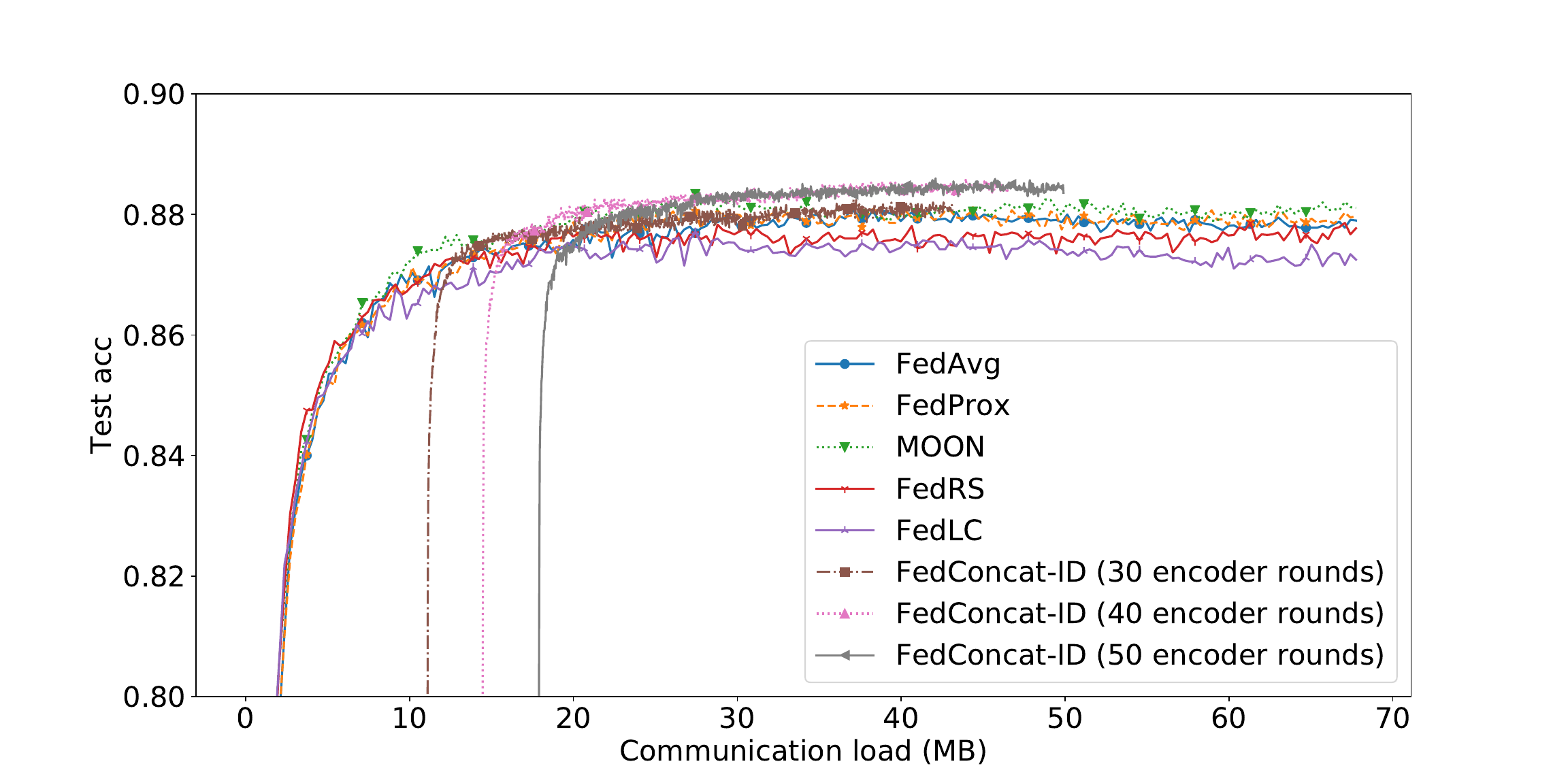}}
    \caption{The training curves of FedConcat-ID. Baseline algorithms use simple CNN model.}
    \label{fig:comm_id}
\end{figure*}

\begin{figure}[]
    \centering
    \includegraphics[width=\columnwidth]{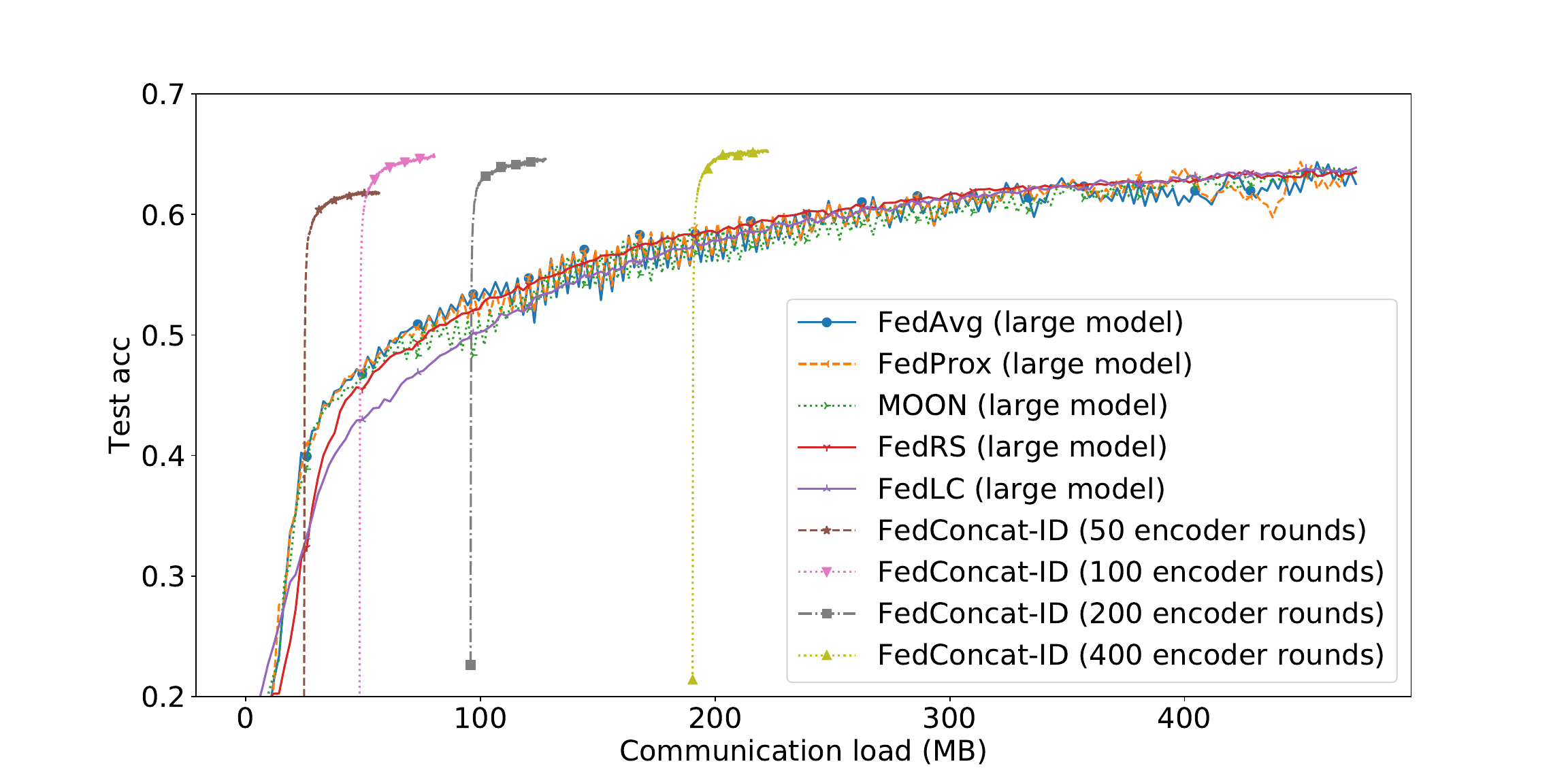}
    \caption{Training curves of FedConcat-ID. Baseline algorithms are trained on model with same size to final global model of FedConcat-ID. Here CIFAR-10 dataset is divided into 40 clients using \#C=2 partition.}
    \vspace{-10pt}
    \label{fig:large-id}
\end{figure}

For the effect of clustering, we compare FedConcat-ID using 5 clusters and without clustering in Figure \ref{fig:cluster-id}. We come to the same conclusion as FedConcat in Figure \ref{fig:wcluster} of main paper. The clustering of FedConcat-ID is effective. Without clustering, the training becomes unstable. The training curves with different number of clusters are shown in Figure \ref{fig:cluster-id2}. Generally by choosing the elbow value as the number of clusters, FedConcat-ID can achieve high accuracy with stable convergence. 

\begin{figure}[]
    \centering
    \includegraphics[width=\columnwidth]{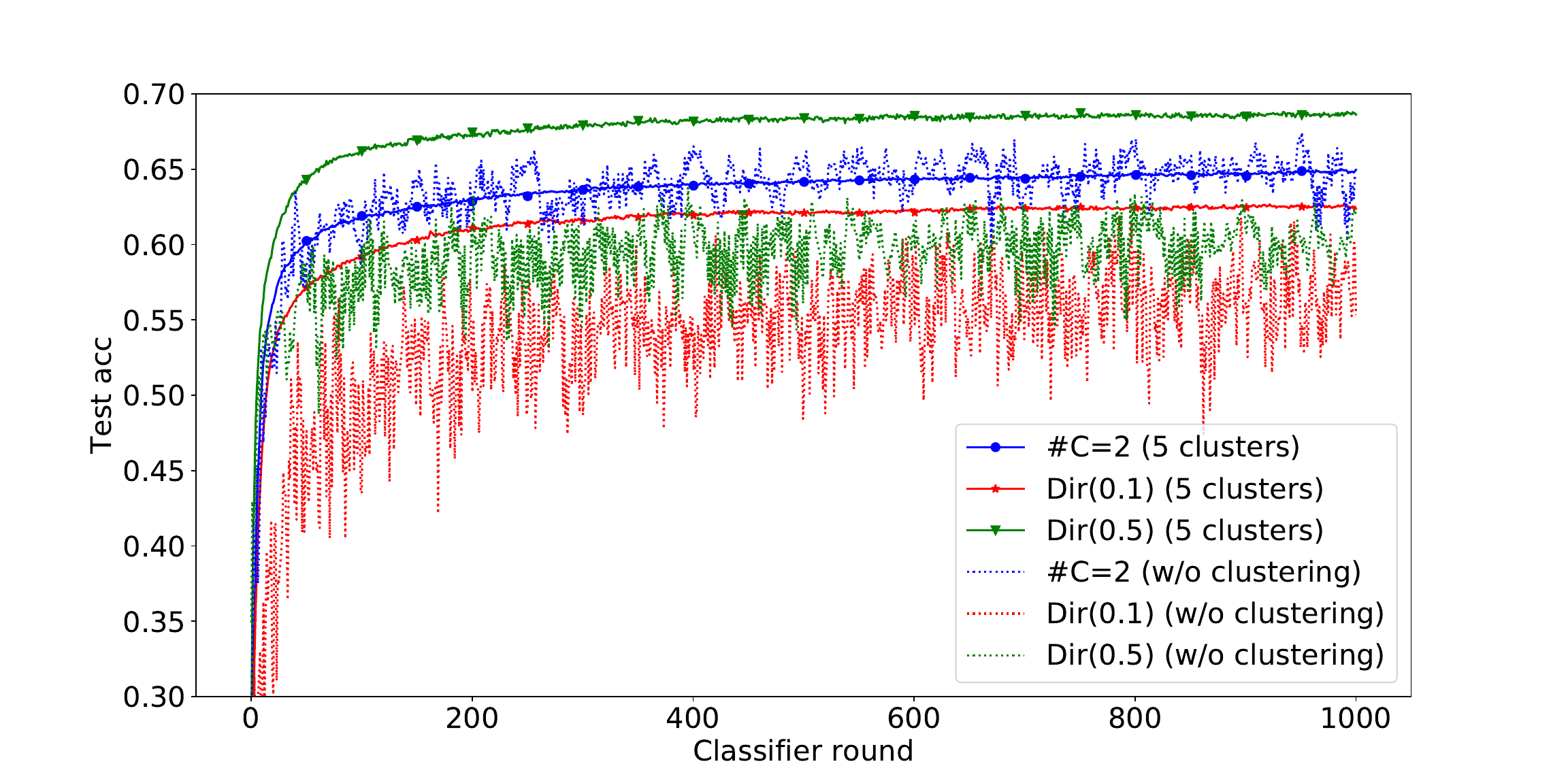}
    \caption{Training curves of FedConcat-ID with clustering versus without clustering on CIFAR-10. Here we divide the whole dataset into 40 clients. }
    \vspace{-10pt}
    \label{fig:cluster-id}
\end{figure}

\begin{figure*}[]
    \centering
    \subfloat[CIFAR-10, \#C=2]{\includegraphics[width=0.33\textwidth]{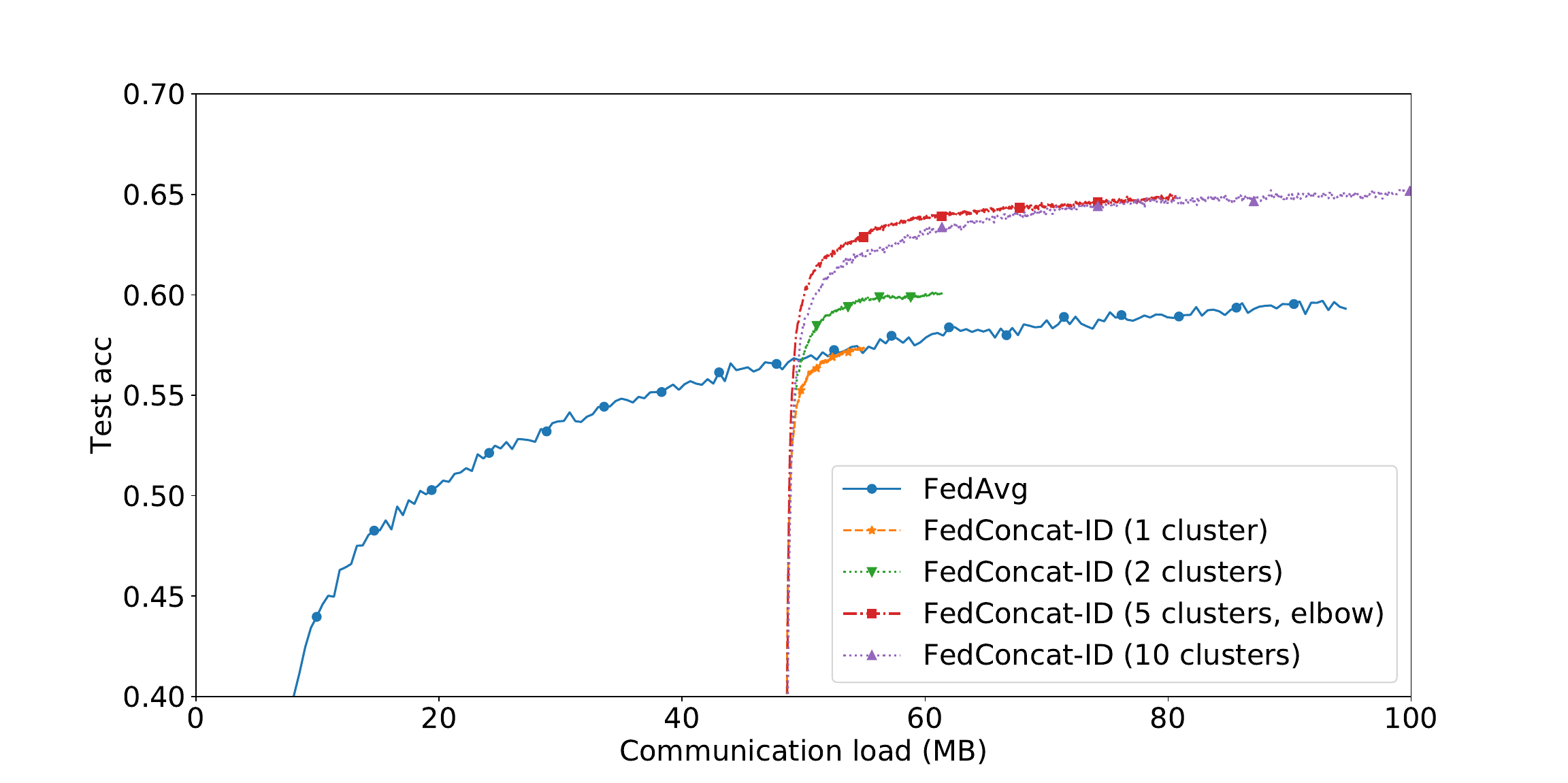}}
    \subfloat[CIFAR-10,$p_k \sim Dir(0.1)$]{\includegraphics[width=0.33\textwidth]{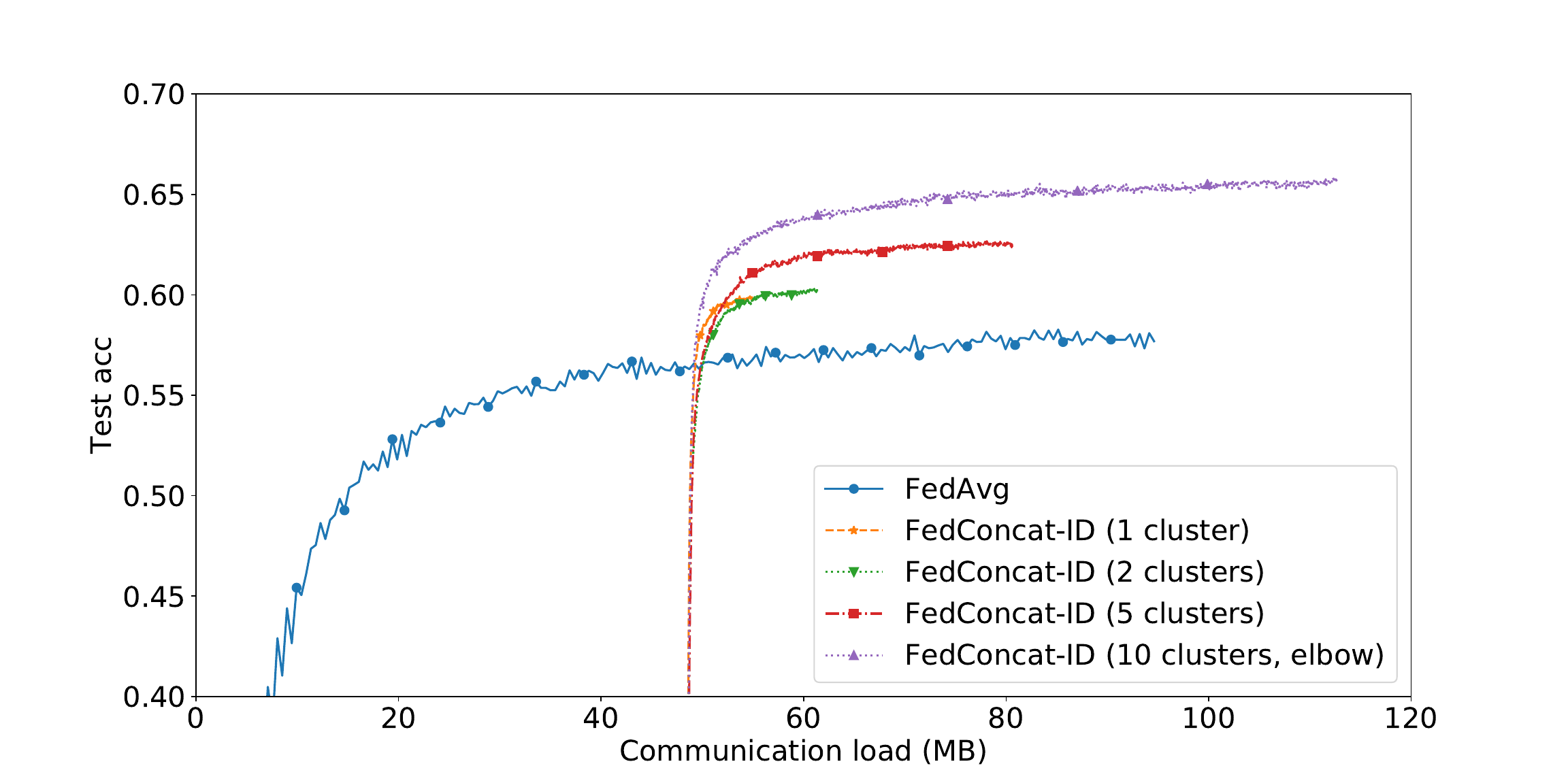}}
    \subfloat[CIFAR-10,$p_k \sim Dir(0.5)$]{\includegraphics[width=0.33\textwidth]{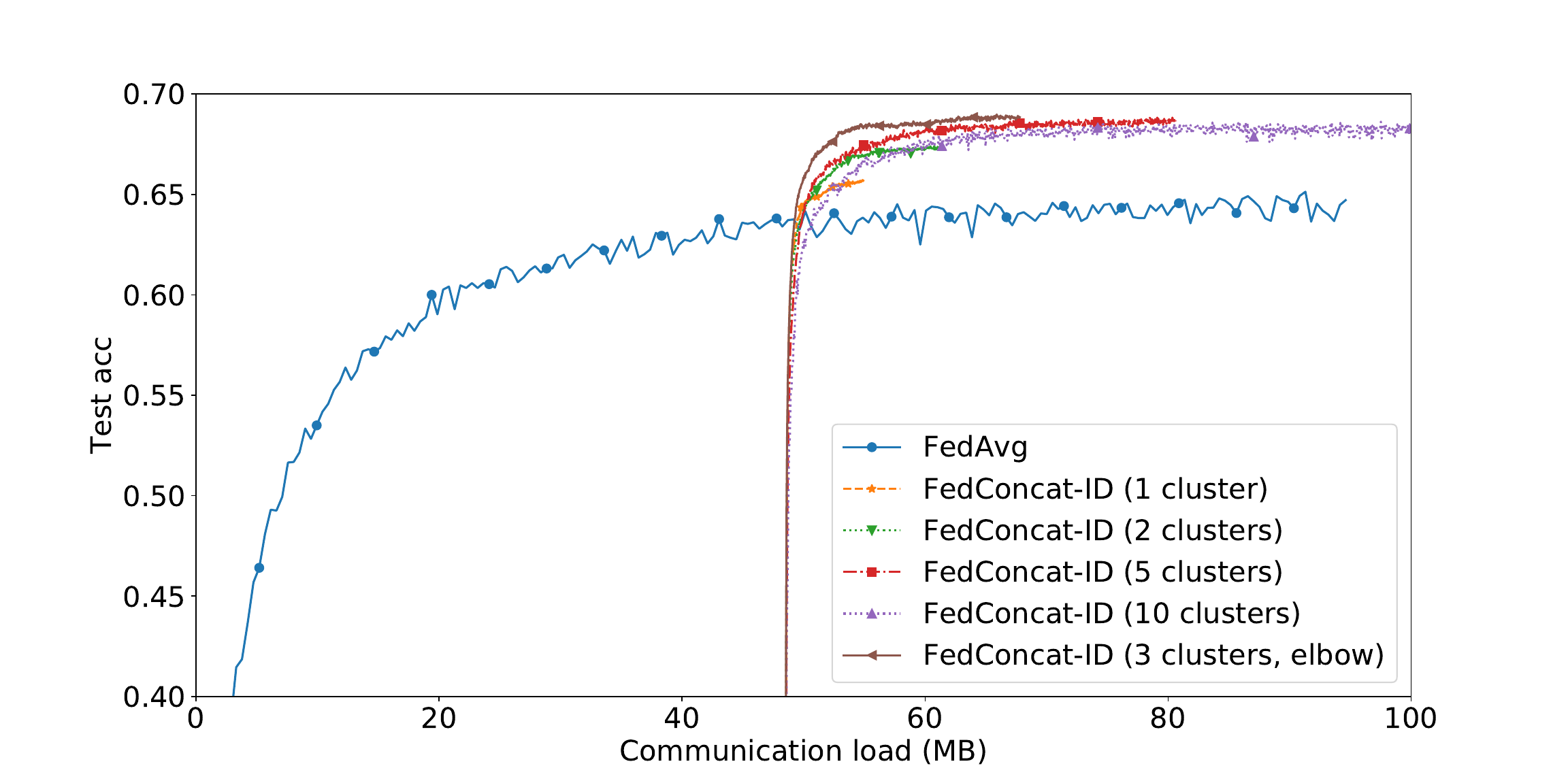}}
    \hfill
    \subfloat[FMNIST, \#C=2]{\includegraphics[width=0.33\textwidth]{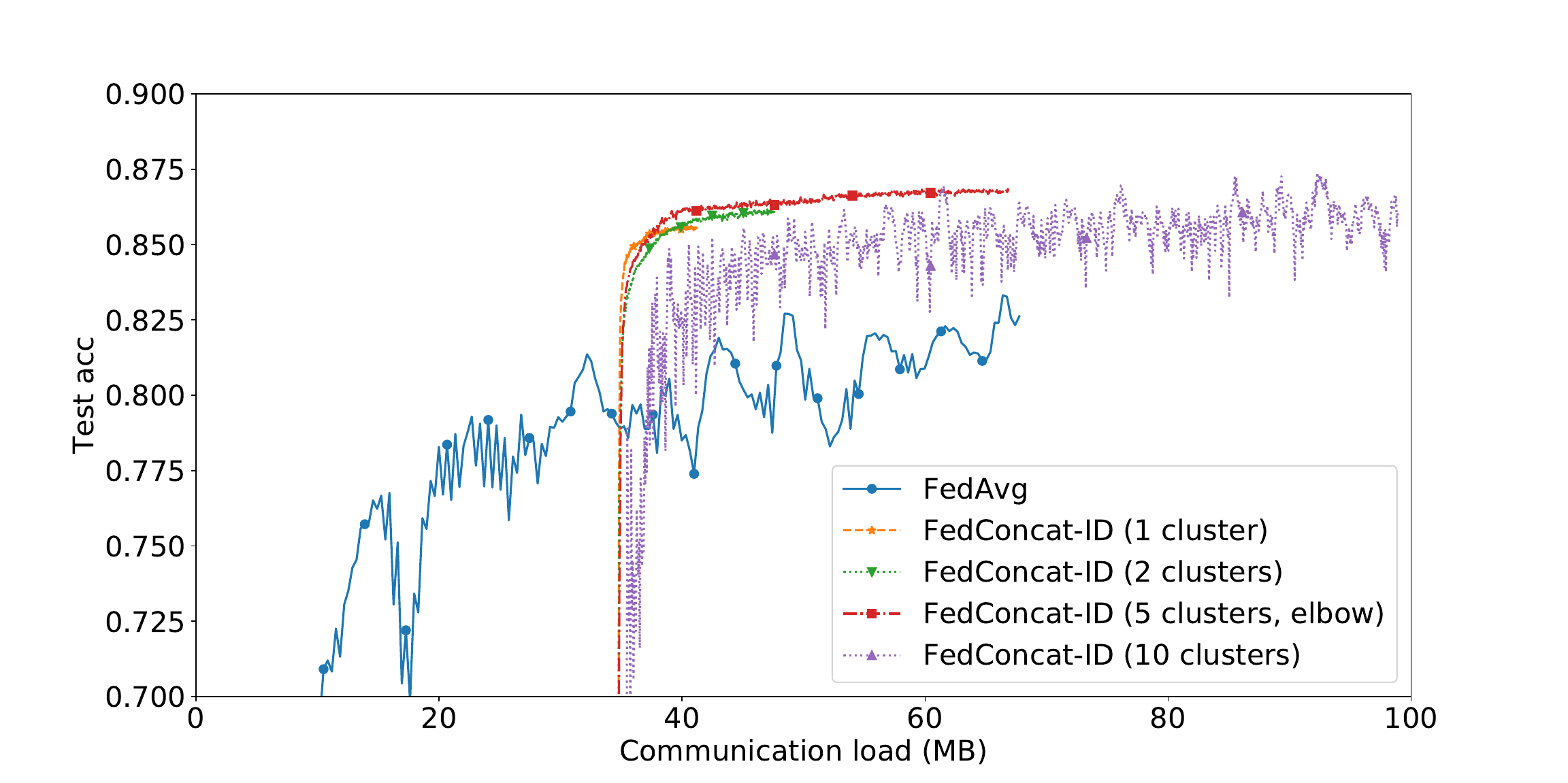}}
    \subfloat[FMNIST, $p_k \sim Dir(0.1)$]{\includegraphics[width=0.33\textwidth]{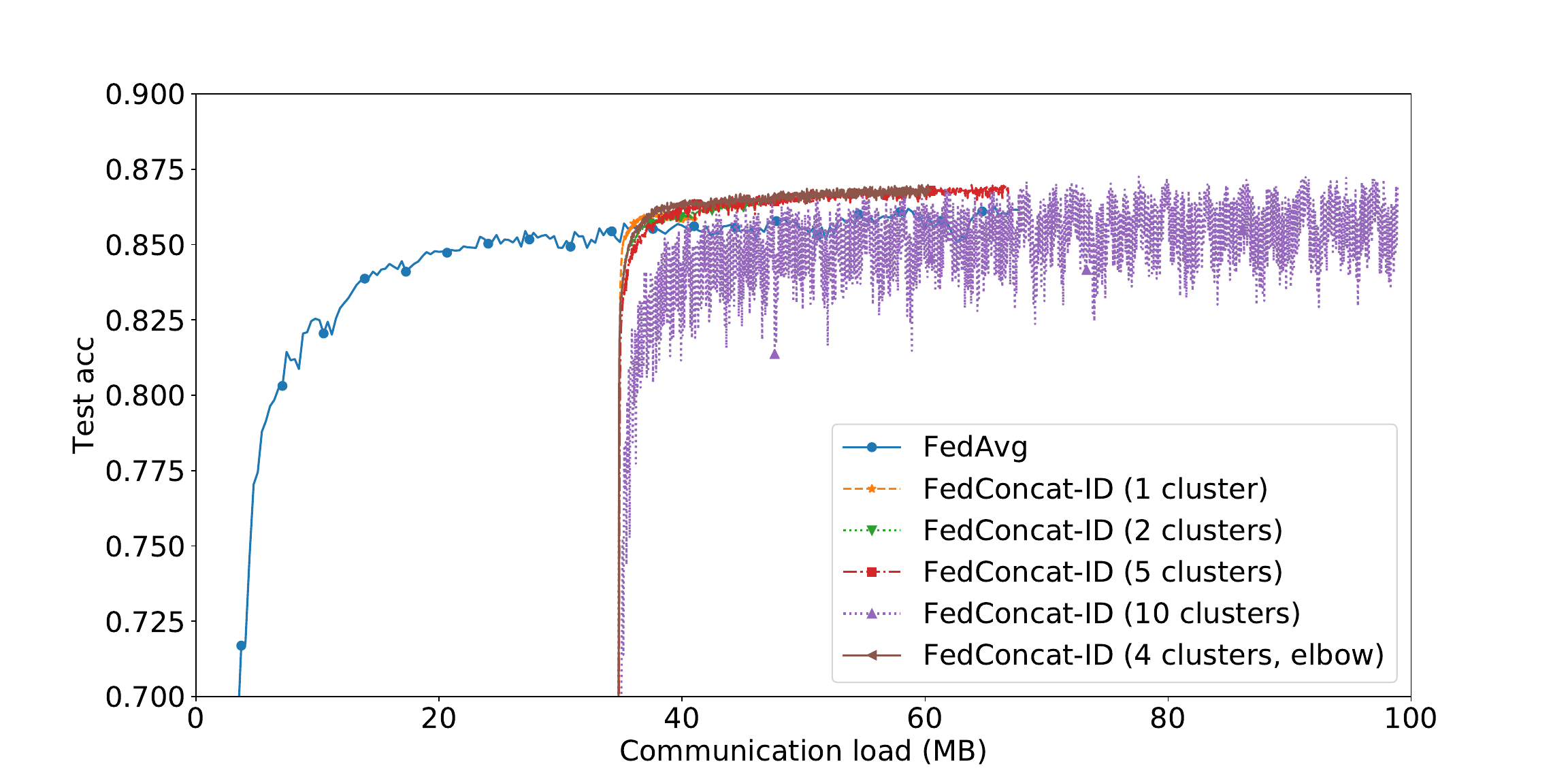}}
    \subfloat[FMNIST, $p_k \sim Dir(0.5)$]{\includegraphics[width=0.33\textwidth]{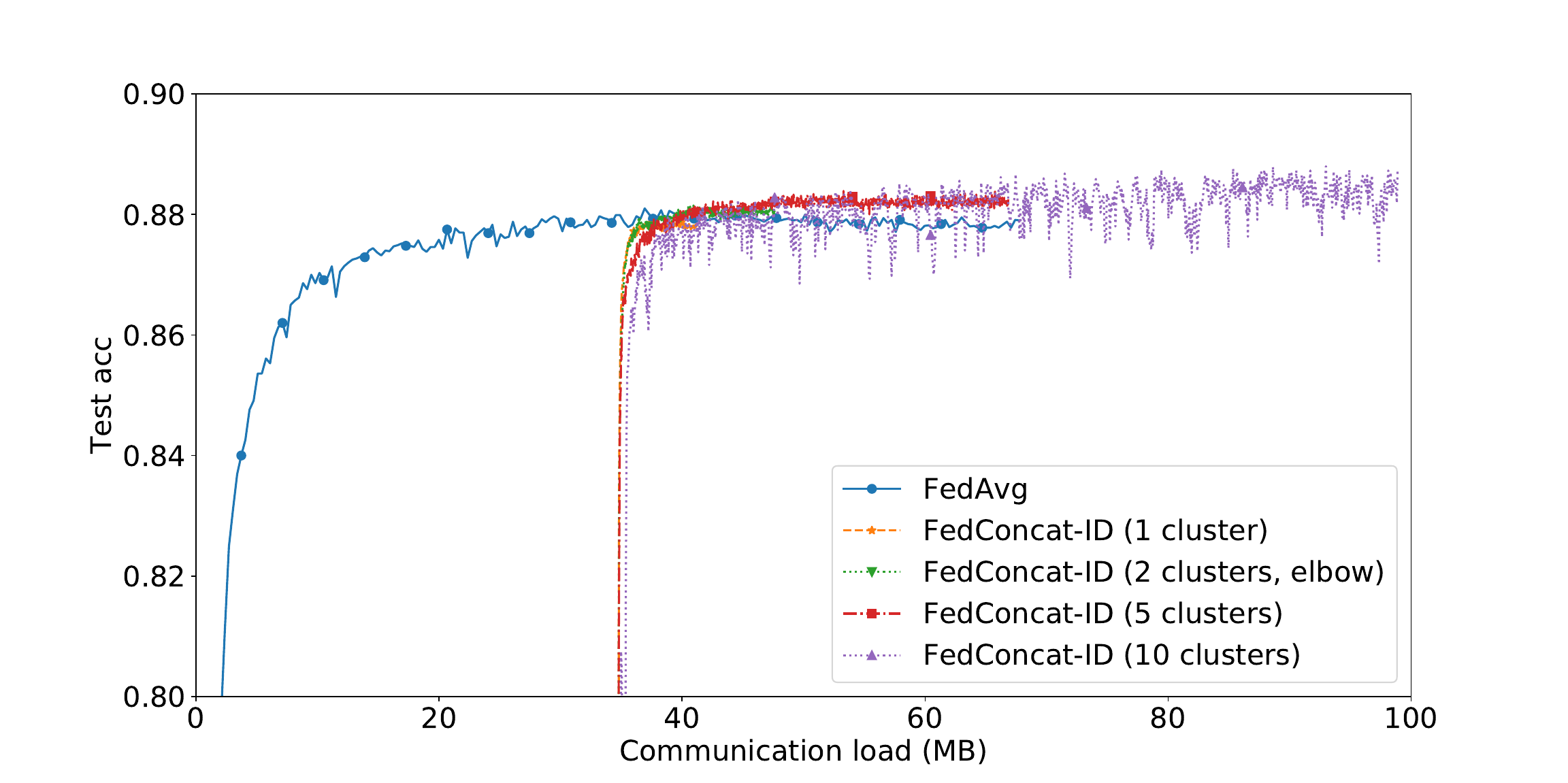}}
    \caption{The training curves of FedConcat-ID with different cluster number. Elbow means such $K$ is estimated to be about the elbow value. }
    \label{fig:cluster-id2}
\end{figure*}